%% file: main.tex
\documentclass[11pt]{article}
\pdfoutput=1
\usepackage[round,compress]{natbib}
\usepackage{ss}
\include{header}
\include{macros}

\begin{document}

\title{Cost-Driven Representation Learning for \\ Linear Quadratic Gaussian Control: Part I}
\author{%
    \name Yi Tian \email{yitian@mit.edu}\\
    \addr{Massachusetts Institute of Technology}\\
    \name Kaiqing Zhang \email{kaiqing@umd.edu}\\
    \addr{University of Maryland, College Park}\\
    \name Russ Tedrake \email{russt@mit.edu}\\
    \addr{Massachusetts Institute of Technology}\\
    \name Suvrit Sra \email{s.sra@tum.de}\\
    \addr{Technical University Munich}
}
\maketitle

\input{lqg}

\end{document}

%% file: header.tex
\usepackage[utf8]{inputenc}
\usepackage[T1]{fontenc}

\usepackage[verbose=true,
letterpaper,
textheight=8.7in,
textwidth=6.4in,
headheight=12pt,
headsep=25pt,
footskip=30pt]{geometry}

\usepackage{mathtools}
\usepackage{booktabs}
\usepackage{graphicx}
\usepackage{amsmath,amssymb,amsfonts,amsthm,amsxtra,bm}
\usepackage{amsthm}
\usepackage{xspace}
\usepackage{enumerate}
\usepackage{hyperref}
\usepackage{url}
\usepackage{colortbl}
\usepackage{makecell,multirow}       
\usepackage{nicefrac}       
\usepackage{thmtools}  
\usepackage{xspace}
\usepackage{footnote, tablefootnote}
\usepackage{float}
\floatstyle{plaintop}
\restylefloat{table}
\usepackage{epstopdf}
\usepackage{latexsym}
\usepackage{algorithm, algorithmicx, algpseudocode}
\algnewcommand{\LineComment}[1]{\Statex \texttt{// #1}}
\usepackage{todonotes}
\usepackage{thmtools, thm-restate}
\usepackage{bbm}

\numberwithin{equation}{section}

%% file: macros.tex
\newtheorem{theorem}{Theorem}
\newtheorem{lemma}{Lemma}

\newtheorem{proposition}{Proposition}

\newtheorem{assumption}{Assumption}
\newtheorem{claim}{Claim}

\def\endproof{\hfill\qedsymbol}

\newcommand{\corel}{\textsc{CoReL}}
\newcommand{\svt}{\textsc{TruncSV}}
\newcommand{\raw}{\textnormal{raw}}
\newcommand{\state}{\textnormal{state}}

\newcommand{\const}{\varrho}

\newcommand{\eps}{\varepsilon}

\newcommand{\E}{\mathbb{E}}
\newcommand{\I}{\mathbb{I}}

\newcommand{\R}{\mathbb{R}}
\newcommand{\s}{\mathbb{S}}

\newcommand{\tr}{\mathrm{tr}}

\newcommand{\Cov}{\mathrm{Cov}}

\DeclareMathOperator*{\argmin}{argmin}

\newcommand\ip[2]{\left\langle #1, #2 \right\rangle}
\newcommand{\ipc}[2]{\bigl\langle #1, #2 \bigr\rangle}

\newcommand{\given}{\;\vert\;}
\newcommand{\vectorize}{\mathrm{vec}}
\newcommand{\svec}{\mathrm{svec}}
\newcommand{\rank}{\mathrm{rank}}
\newcommand{\poly}{\mathrm{poly}}

\newcommand{\nlsum}{\sum\nolimits}

\newcommand{\bigo}{\mathcal{O}}

\newcommand{\gauss}{\mathcal{N}}
\newcommand{\diag}{\textnormal{diag}}

\newcommand{\cD}{\mathcal{D}}

\newcommand{\oSigma}{\overline{\Sigma}}
\newcommand{\ob}{\overline{b}}
\newcommand{\oc}{\overline{c}}
\newcommand{\ooe}{\overline{e}}
\newcommand{\of}{\overline{f}}
\newcommand{\oA}{\overline{A}}
\newcommand{\oB}{\overline{B}}

\newcommand{\oF}{\overline{F}}
\newcommand{\oQ}{\overline{Q}}

\newcommand{\spi}{\pi^{\ast}}
\newcommand{\sSigma}{\Sigma^{\ast}}
\newcommand{\sPhi}{\Phi^{\ast}}
\newcommand{\sA}{A^{\ast}}
\newcommand{\sB}{B^{\ast}}
\newcommand{\sC}{C^{\ast}}

\newcommand{\sF}{F^{\ast}}

\newcommand{\sK}{K^{\ast}}
\newcommand{\sL}{L^{\ast}}
\newcommand{\sM}{M^{\ast}}
\newcommand{\sN}{N^{\ast}}
\newcommand{\sP}{P^{\ast}}
\newcommand{\sQ}{Q^{\ast}}
\newcommand{\sR}{R^{\ast}}
\newcommand{\sS}{S^{\ast}}

\newcommand{\sX}{X^{\ast}}

\newcommand{\soA}{\oA^{\ast}}
\newcommand{\soB}{\oB^{\ast}}
\newcommand{\soF}{\oF^{\ast}}
\newcommand{\soQ}{\oQ^{\ast}}
\newcommand{\starb}{b^{\ast}}

\newcommand{\starf}{f^{\ast}}
\newcommand{\sof}{\of^{\ast}}
\newcommand{\sh}{h^{\ast}}
\newcommand{\su}{u^{\ast}}

\newcommand{\sx}{x^{\ast}}
\newcommand{\sy}{y^{\ast}}
\newcommand{\sz}{z^{\ast}}

\newcommand{\epi}{\hat{\pi}}
\newcommand{\eDelta}{\hat{\Delta}}
\newcommand{\ePhi}{\hat{\Phi}}

\newcommand{\eXi}{\hat{\Xi}}
\newcommand{\eb}{\hat{b}}

\newcommand{\ez}{\hat{z}}
\newcommand{\eA}{\hat{A}}
\newcommand{\eB}{\hat{B}}
\newcommand{\eJ}{\hat{J}}
\newcommand{\eK}{\hat{K}}

\newcommand{\eM}{\hat{M}}
\newcommand{\eN}{\hat{N}}
\newcommand{\eP}{\hat{P}}
\newcommand{\eQ}{\hat{Q}}

\newcommand{\tM}{\widetilde{M}}
\newcommand{\tN}{\widetilde{N}}
\newcommand{\tQ}{\widetilde{Q}}

\newcommand{\spA}{A^{\ast\prime}}
\newcommand{\spB}{B^{\ast\prime}}
\newcommand{\spC}{C^{\ast\prime}}
\newcommand{\spL}{L^{\ast\prime}}
\newcommand{\spM}{M^{\ast\prime}}

\newcommand{\spQ}{Q^{\ast\prime}}

\newcommand{\spz}{z^{\ast\prime}}

\makeatletter
\newcommand{\vast}{\bBigg@{3}}
\newcommand{\Vast}{\bBigg@{3.5}}
\makeatother

%% file: lqg.tex

\begin{abstract}
    We study the task of learning state representations from potentially high-dimensional observations, with the goal of controlling an unknown partially observable system. 
    We pursue a \emph{cost-driven} approach, where a dynamic model in some latent state space is learned by predicting the costs without predicting the observations or actions.   
    In particular, we focus on an intuitive cost-driven state representation learning method for solving Linear Quadratic Gaussian (LQG) control, one of the most fundamental partially observable control problems.
    As our main results, we establish finite-sample guarantees of finding a near-optimal state representation function and a near-optimal controller using the directly learned latent model, for finite-horizon time-varying LQG control problems. To the best of our knowledge, despite various empirical successes, finite-sample guarantees of such a cost-driven approach remain elusive.         
    Our result  underscores the value of predicting multi-step costs, an idea that is key to our theory, and notably also an idea that is known to be empirically valuable for learning state representations.  
    A second part of this work, that is to appear as Part II, addresses the infinite-horizon linear time-invariant setting; it also extends the results to an approach that \emph{implicitly} learns the latent dynamics, inspired by the recent empirical breakthrough of MuZero in model-based reinforcement learning.
\end{abstract}

\section{Introduction}

We consider state representation learning for control in partially observable systems, 
inspired by the recent successes of \emph{control from pixels}~\citep{hafner2019learning, hafner2019dream}. Control from pixels is an everyday task for human beings, but it remains challenging for learning agents. Methods to achieve it generally fall into two main categories: \emph{model-free} and \emph{model-based} ones. Model-free methods directly learn a visuomotor policy, also known as direct reinforcement learning (RL)~\citep{sutton2018reinforcement}.
On the other hand, model-based methods, also known as indirect RL~\citep{sutton2018reinforcement}, attempt to learn a \emph{latent model} that is a compact representation of the system, and to synthesize a policy in the latent model. Compared with model-free methods, model-based ones facilitate generalization across tasks and enable efficient planning~\citep{hafner2020mastering}, and are sometimes more sample efficient~\citep{tu2019gap, sun2019model, zhang2019solar} than the model-free ones.  

In latent-model-based control, the state of the latent model is also referred to as a \emph{state representation} in the deep RL literature, and the mapping from an observed history to a latent state is referred to as the (state) representation function. 
\emph{Reconstructing the observation} often serves as a supervision for representation learning for control in the empirical RL literature~\citep{hafner2019learning, hafner2019dream, hafner2020mastering, fu2021learning, wang2022denoised}. This is in sharp contrast to model-free methods, where the policy improvement step is completely \emph{cost-driven}. 
Reconstructing observations provides a rich supervision signal for learning a task-agnostic world model, but they are high-dimensional and noisy, so the reconstruction requires an expressive reconstruction function; latent states learned by reconstruction contain irrelevant information for control, which can distract RL algorithms~\citep{zhang2020learning, fu2021learning, wang2022denoised}. This is especially the case for practical visuomotor control tasks, e.g., robotic manipulation and self-driving cars, where the visual images contain predominately task-irrelevant objects and backgrounds. 

Various empirical attempts~\citep{schrittwieser2020mastering, zhang2020learning, okada2021dreaming, deng2021dreamerpro, yang2022discrete} have been made to bypass observation reconstruction. Apart from observation, the interaction involves two other variables: actions (control inputs) and costs. Inverse model methods~\citep{lamb2022guaranteed} reconstruct actions; while other methods rely on costs. We argue that since neither the reconstruction function nor the inverse model is used for policy learning, cost-driven state representation learning is the most \emph{direct} one, in that costs are directly relevant for control purposes. 
In this work, we aim to examine the soundness of this methodology in linear quadratic Gaussian (LQG) control, one of the most fundamental partially observable control models.

Parallel to the empirical advances of learning for control from pixels, partially observable linear systems has been extensively studied in the context of learning for dynamic control~\citep{oymak2019non, simchowitz2020improper, lale2020logarithmic, lale2021adaptive, zheng2021sample, minasyan2021online, umenberger2022globally}. 
In this context, the state representation function is more formally referred to as a \emph{filter}, the optimal one being the Kalman filter. Most existing \emph{model-based} learning approaches for LQG control focus on the linear time-invariant (LTI) case, and are based on the idea of \emph{learning Markov parameters}~\citep{ljung1998system}, the mapping from control inputs to observations. 
Hence, they need to {predict observations} by definition.
Motivated by the empirical successes in control from pixels, we take a different, cost-driven route, in hope of avoiding reconstructing observations or control inputs.

We focus on finite-horizon time-varying LQG control and address the following question:
\begin{center}
    \emph{Can cost-driven state representation learning provably solve LQG control?}
\end{center}
This work answers the question in the affirmative, by establishing finite-sample guarantees for a cost-driven state representation learning method.
We address the finite-horizon linear time-varying (LTV) setting in this Part I, and 
will move on to the infinite-horizon linear time-invariant (LTI) setting with additional technical challenges in Part II of the work.

\vspace{3pt}
\noindent \textbf{Challenges \& Our techniques.} 
To establish finite-sample guarantees, a major technical challenge is to deal with the \emph{quadratic regression} problem in cost prediction, arising from the inherent quadratic form of the LQG cost. Directly solving 
for the state representation function involves \emph{quartic} optimization; instead, we propose to solve a quadratic regression problem, followed by low-rank approximate factorization. The quadratic regression problem also appears in identifying the cost matrices, involving the concentration for random variables that are fourth powers of Gaussians. 
Our techniques to address these challenges may be of independent interest.

Moreover, the first $\ell$-step \emph{latent} states may not be adequately \emph{excited} (with full-rank covariance), which 
results in the identification of the latent model only in \emph{partial directions}. This poses a significant challenge for certifying the performance of the learned controller, as the learned latent model from which we synthesize the controller may be neither stable nor controllable. We overcome this challenge by analyzing state covariance mismatch using induction, 
showing that identifying only the \emph{relevant directions} suffices for learning a \mbox{near-optimal controller}.
This fact is reflected in the dependence on $\ell$ in the statement of Theorem~\ref{thm:main}. 

Lastly, the learned latent states and the errors in the learned latent states are \emph{correlated} as they are both functions of the same observed trajectory. This challenge arises both in analyzing latent model identification and in certifying the performance of the learned controller. We tackle this challenge by modeling the errors as general correlated perturbations whose magnitudes are controlled by the errors in the learned state representation function.

\vspace{3pt}
\noindent\textbf{Implications.} 
For practitioners, 
one takeaway 
is the benefit of predicting \emph{multi-step cumulative} costs in cost-driven state representation learning. Whereas the cost at a single time step may not be revealing enough of the latent state, the cumulative cost across multiple steps can be. This is an intuitive idea for the control community, given the multi-step nature in the classical definitions of controllability and observability. Its effectiveness has also been empirically observed in MuZero~\citep{schrittwieser2020mastering} in state representation learning for control, and our work can be viewed as a formal understanding of it in the LQG  setting.

\vspace{3pt}
\noindent\textbf{Notation.}  
We use $0$ (resp. $1$) to denote either the scalar or a matrix consisting of all zeros (resp. all ones); we use $I$ to denote an identity matrix. 
The dimension, when emphasized, is specified in subscripts, e.g., $0_{d_x \times d_x}, 1_{d_x}, I_{d_x}$. 
Let $\I_S$ denote the indicator function for set $S$ and $\I_S(A)$ apply to matrix $A$ elementwise. 
For some positive semidefinite $P$, we define $\|v\|_P := (v^{\top} P v)^{1/2}$. 
Semicolon ``$;$'' denotes stacking vectors or matrices vertically.
For a collection of $d$-dimensional vectors $(v_t)_{t=i}^{j}$, let $v_{i:j} := [v_i; v_{i+1}; \ldots; v_j] \in \R^{d(j-i+1)}$ denote the concatenation along the column.
For random variable $\eta$, let $\|\eta\|_{\psi_{\beta}}$ denote its $\beta$-sub-Weibull norm, a special case of Orlicz norms~\citep{zhang2022sharper}, with $\beta = 1, 2$ corresponding to subexponential and sub-Gaussian norms. 
For matrix $A$, let $\sigma_i(A), \sigma_{\min}(A), \sigma_{\min}^{+}(A), \sigma_{\max}(A)$ denote its $i$th largest, minimum, minimum positive, maximum singular values, respectively.
$\|A\|_{2}, \|A\|_{F}, \|A\|_{\ast}$ denote the operator (induced by vector $2$-norms), Frobenius, nuclear norms of matrix $A$, respectively. $\ip{\cdot}{\cdot}_F$ denotes the Frobenius inner product between matrices.
For square matrix $A$, let $\lambda_{\min}(A)$ be its minimum eigenvalue.
The Kronecker, symmetric Kronecker, and Hadamard products between matrices are denoted by ``$\otimes$'', ``$\otimes_s$'' and ``$\odot$'', respectively. 
$\vectorize(\cdot)$ and $\svec(\cdot)$ denote flattening a matrix and a symmetric matrix by stacking their columns; $\svec(\cdot)$ does not repeat the off-diagonal elements, but scales them by $\sqrt{2}$~\citep{schacke2004kronecker}. 
Let $\diag(\cdot)$ denote the block diagonal matrix formed by the matrices inside the parentheses.
For 
$\Sigma_{i = a}^{b}$, we define the sum to be zero if the lower index $a$ is greater than the upper index $b$.

\vspace*{-8pt}
\section{Problem setup}
\label{sec:prelim}

We study a partially observable linear 
dynamical system 
\begin{align}  \label{eqn:po-ltv}
    x_{t+1} = \sA_t x_t + \sB_t u_t + w_t, \quad 
    y_t = \sC_t x_t + v_t,
\end{align}
for $t = 0, 1, \ldots, T-1$ and $y_T = \sC_T x_T + v_T$.
For all $t \ge 0$, we have the notation of state $x_t \in \R^{d_x}$, observation $y_t \in \R^{d_y}$, and control input $u_t \in \R^{d_u}$. Process noises $(w_t)_{t=0}^{T-1}$ and observation noises $(v_t)_{t=0}^T$ are i.i.d. sampled from $\gauss(0, \Sigma_{w_t})$ and $\gauss(0, \Sigma_{v_t})$, respectively. Let initial state $x_0$ be sampled from $\gauss(0, \Sigma_0)$.

Let $\Phi_{t, t_0} = \sA_{t-1} \sA_{t-2} \cdots \sA_{t_0}$ for $t > t_0$ and $\Phi_{t, t } = I$. Then $x_t = \Phi_{t, t_0} x_{t_0} + \sum_{\tau = t_0}^{t-1} \Phi_{t, \tau + 1} w_{\tau}$ under zero control input. To ensure the state and the cumulative noise do not grow with time, we make the following uniform exponential stability assumption.
\begin{assumption}[Uniform exponential stability]  \label{asp:stability}
    The system is uniformly exponentially stable, i.e., there exists $\alpha > 0, \rho \in (0, 1)$ such that for any $0\le t_0 < t \le T$, $\|\Phi_{t, t_0}\|_2 \le \alpha \rho^{t - t_0}$.
\end{assumption} 

Assumption~\ref{asp:stability} is standard in controlling LTV systems~\citep{zhou2017asymptotic,hu2017exponential,minasyan2021online}, satisfied by a stable LTI system.
It essentially says that zero control is a stabilizing controller, and can be potentially relaxed to the assumption of \emph{being given a stabilizing controller} as in~\citep{lale2020logarithmic, simchowitz2020improper}, where one can excite the system using the stabilizing controller plus Gaussian random noises.

Define the $\ell$-step controllability matrix
\begin{align*}
    \Phi_{t, \ell}^{c} := [\sB_{t}, \sA_{t} \sB_{t-1}, \ldots, \sA_{t} \sA_{t-1} \cdots \sA_{t-\ell+2} \sB_{t-\ell+1}] \in \R^{d_x \times \ell d_u}
\end{align*}
for $\ell-1 \le t \le T-1$, which reduces to the standard controllability matrix $[B, \ldots, A^{\ell-1} B]$ in the LTI setting. 
We make the following controllability assumption.

\begin{assumption}[Controllability]  \label{asp:ctrl}
    For all $\ell-1\le t\le T-1$, $\rank(\Phi_{t, \ell}^c) = d_x$, $\sigma_{\min}(\Phi_{t, \ell}^c) \ge \nu>0$.
\end{assumption}

Under zero noise, we have 
\begin{align*}
    x_{t+\ell} = \Phi_{t+\ell, t} x_t + \Phi_{t+\ell-1, \ell}^c [u_{t+\ell-1}; \ldots; u_t],
\end{align*}
so Assumption~\ref{asp:ctrl} ensures that from any state $x$, there exist control inputs that drive the state to $0$ in $\ell$ steps, and $\nu$ ensures that the equation leading to them is well conditioned. We do not assume controllability for $0\le t < \ell-1$, since we do not want to impose the constraint that $d_u > d_x$. This turns out to present a significant challenge for analyzing state representation function learning, latent model identification, and the performance of the overall policy, resulting in the separation at step $\ell$ in the sample complexity guarantees (see Theorem \ref{thm:main}).

The quadratic cost functions are given by 
\begin{equation}  \label{eqn:q-cost}
    c_t(x, u) = \|x\|_{\sQ_t}^2 + \|u\|_{\sR_t}^2, \quad 0\le t\le T - 1, \quad
    c_T(x) = \|x\|_{\sQ_T}^2,    
\end{equation}
where $(\sQ_t)_{t=0}^{T}$ are positive semidefinite matrices and $(\sR_t)_{t=0}^{T-1}$ are positive definite matrices. Sometimes the cost is defined as a function of the  observation $y$. Since the quadratic form $y^{\top} \sQ_t y = x^{\top} (\sC_t)^{\top} \sQ_t \sC_t x$, our analysis still applies if the assumptions on $(\sQ_t)_{t=0}^{T}$ hold for $((\sC_t)^{\top} \sQ_t \sC_t)_{t=0}^{T}$ instead. 

The observability assumptions on 
$(A, C)$ and $(A, Q^{1/2})$ are standard  in controlling LTI systems. To differentiate from the former, we call the latter \emph{cost observability}, since it implies the states are observable through costs. Whereas Markov-parameter-based approaches need to assume $(A, C)$ observability to identify the system, our cost-driven  approach does not. 
Here we deal with the more difficult problem of having only the scalar cost as the supervision signal (instead of the concatenation of all observations, as in Markov-parameter-based ones). Nevertheless, the notion of cost observability is still important for our approach, formally defined as follows.
\begin{assumption}[Cost observability]  \label{asp:cost-obs}
    For all $0\le t \le \ell - 1$, $\sQ_t \succcurlyeq \mu^2 I$.
    For all $\ell \le t\le T$, there exists $m > 0$ such that the cost observability Gram matrix~\citep{kailath1980linear}
    \begin{align*}
        &\nlsum_{\tau = t}^{t+k-1} \Phi_{\tau, t}^{\top} \sQ_{\tau} \Phi_{\tau, t}
        \succcurlyeq \mu^2 I, 
    \end{align*}
    where $k = m \wedge (T - t + 1)$.
\end{assumption}

This assumption ensures that without noises, if we start with a nonzero state, the cumulative cost becomes positive in $m$ steps. The special requirement for $0\le t\le \ell-1$ results from the difficulty in lacking controllability in these time steps.
The following is a regularity assumption on system parameters.

\begin{assumption}  \label{asp:bounded}
    $(\lambda_{\min}(\Sigma_{v_t}))_{t=0}^{T}$ and $(\lambda_{\min}(\sR_t))_{t=0}^{T-1}$ are uniformly lower bounded; that is, they are all $\Omega(1)$. 
    The operator norms of all matrices in the problem definition and $(\sM_t)_{t=0}^{T}$ to be defined in~\S\ref{sec:latent-model} are uniformly upper bounded, including $(\sA_t, \sB_t, \sR_t, \Sigma_{w_t})_{t=0}^{T-1}$, $(\sC_t, \sQ_t, \Sigma_{v_t})_{t=0}^T$; that is, they are all $\bigo(1)$. 
\end{assumption}

Along with Assumption~\ref{asp:stability}, this assumption ensures that under control $u_t \sim \gauss(0, \sigma_u^2 I)$ for all $t \ge 0$ with $\sigma_u = \bigo(1)$, the covariance matrices $\Cov(x_t)$ and $\Cov(y_t)$ have operator norms bounded by $\bigo(1)$ for all $t \ge 0$. Moreover, $(\sS_t, \sL_t, \sP_t)_{t=0}^{T}$ and $(\sK_t)_{t=0}^{T-1}$, to be defined shortly in~\eqref{eqn:L-riccati} to~\eqref{eqn:p-riccati} in the optimal solution, have $\bigo(1)$ operator norms.

Let $h_t := [y_{0:t}; u_{0:(t-1)}] \in \R^{(t+1)d_y + t d_u}$ denote the available history before deciding control $u_t$ for $t\ge 1$ and define $h_0 := y_0$.
A policy $\pi = (\pi_t: h_t \mapsto u_t)_{t=0}^{T-1}$ determines at time $t$ a control input $u_t$ based on history $h_t$.
With a slight abuse of notation, let $c_t := c_t(x_t, u_t)$ for $0\le t\le T-1$ and $c_T := c_T(x_T)$ denote the cost at each time step. Then, $J(\pi) := \E^{\pi}[\sum_{t=0}^{T} c_t]$ is the expected cumulative cost under policy $\pi$, where the expectation is taken over the randomness in the process noises, observation noises, and controls (if the policy $\pi$ is stochastic). The objective of LQG control is to find a policy $\pi$ such that $J(\pi)$ is minimized.

If the system parameters $((\sA_t, \sB_t, \sR_t)_{t=0}^{T-1}, (\sC_t, \sQ_t)_{t=0}^{T})$ are known, optimal control is obtained by combining the Kalman filter $\sz_0 = \sL_0 y_0$, 
\begin{align*}
    \sz_{t+1} = \sA_t \sz_t + \sB_t u_t + \sL_{t+1} ( y_{t+1} - \sC_{t+1} (\sA_{t} \sz_t + \sB_t u_t) )
\end{align*}
for $0\le t \le T-1$, with the optimal feedback control gains of the linear quadratic regulator (LQR) $(\sK_t)_{t=0}^{T-1}$ such that $\su_t = \sK_t\sz_t$, where $(\sL_t)_{t=0}^{T}$ are the Kalman gains;  
this is known as the \emph{separation principle}~\citep{aastrom2012introduction}. The Kalman gains and optimal feedback control gains are given by 
\begin{gather}
    \sL_t = \sS_t (\sC_t)^{\top} (\sC_t \sS_t (\sC_t)^{\top} + \Sigma_{v_t})^{-1}, \label{eqn:L-riccati} \\
    \sK_t = -((\sB_t)^{\top} \sP_{t+1} \sB_t + \sR_t)^{-1} (\sB_t)^{\top} \sP_{t+1} \sA_t,  \label{eqn:K-riccati}
\end{gather}
where $\sS_t$ and $\sP_t$ are determined by their corresponding Riccati difference equations (RDEs):
\begin{gather}
    \begin{aligned}
        \sS_{t+1} =\;& \sA_t (\sS_t - \sS_t (\sC_t)^{\top} (\sC_t \sS_t (\sC_t)^{\top} + \Sigma_{v_t})^{-1} \sC_t \sS_t) (\sA_t)^{\top} + \Sigma_{w_t},
    \end{aligned}
    \label{eqn:s-riccati} \\
    \begin{aligned}
        \sP_t =\;& (\sA_t)^{\top} (\sP_{t+1} - \sP_{t+1} \sB_t ((\sB_t)^{\top} \sP_{t+1} \sB_t + \sR_t)^{-1} (\sB_t)^{\top} \sP_{t+1} ) \sA_t + \sQ_t, 
    \end{aligned}
    \label{eqn:p-riccati}
\end{gather}
with $\sS_0 = \Sigma_0$ and $\sP_{T} = \sQ_T$.

We consider data-driven control in a partially observable LTV system~\eqref{eqn:po-ltv} with unknown cost matrices $(\sQ_t)_{t=0}^T$. For simplicity, we assume $(\sR_t)_{t=0}^{T}$ is known, though our approaches can be readily generalized to the case where they are unknown; one can identify them in the quadratic regression~\eqref{eqn:warm-up}.

\vspace*{-10pt}
\subsection{Latent model of finite-horizon time-varying LQG}
\label{sec:latent-model}

Under the Kalman filter, the observation prediction error $i_{t+1} := y_{t+1} - \sC_{t+1}(\sA_t \sz_t + \sB_t u_t)$ is called an \emph{innovation}.
It is known that $i_t$ is independent of history $h_t$ and $(i_t)_{t=1}^{T}$ are independent~\citep{bertsekas2012dynamic}.
Now we are ready to present the following proposition that represents the system in terms of the state estimates by the Kalman filter, which we shall refer to as the \emph{latent model}.

\begin{proposition}  \label{prp:state-est-lds}
    Let $(\sz_t)_{t=0}^{T}$ be state estimates given by the Kalman filter. Then, 
    \begin{align*}
        \sz_{t+1} = \sA_t \sz_t + \sB_t u_t + \sL_{t+1} i_{t+1},
    \end{align*}
    where $\sL_{t+1} i_{t+1}$ is independent of $\sz_t$ and $u_t$, i.e.,  the state estimates follow the same linear dynamics as the underlying state, with noises $\sL_{t+1} i_{t+1}$.
    The cost at step $t$ can then be reformulated as functions of the state estimates by 
    \begin{align*}
        c_t = \|\sz_t\|_{\sQ_t}^2 + \|u_t\|_{\sR_t}^2 + b_t + \gamma_t + \eta_t,
    \end{align*}
    where $b_t > 0$ is a problem-dependent constant, and $\gamma_t = \|x_t - \sz_t\|_{\sQ_t}^2 - b_t$, $\eta_t = 2\ipc{\sz_t}{x_t - \sz_t}_{\sQ_t}$ are both zero-mean subexponential random variables. Under Assumptions~\ref{asp:stability} and~\ref{asp:bounded}, $b_t = \bigo(1)$ and $\|\gamma_t\|_{\psi_1} = \bigo(d_x^{1/2})$; moreover, if control $u_t \sim \gauss(0, \sigma_u^2 I)$ for $0\le t\le T$, then $\|\eta_t\|_{\psi_1} = \bigo(d_x^{1/2})$.
\end{proposition}

\begin{proof} 
    By the property of the Kalman filter, $\sz_t = \E[x_t \given y_{0:t}, u_{0:(t-1)}]$ is a function of the past history $(y_{0:t}, u_{0:(t-1)})$. Action $u_t$ is a function of the past history $(y_{0:t}, u_{0:(t-1)})$ and may contain noise independent of all other random variables.
    Innovation $i_{t+1} = \sC_{t+1} (\sA_{t} (x_t - \sz_t) + w_t) + v_{t+1}$ is independent of the past history $(y_{0:t}, u_{0:(t-1)})$; hence, it is independent of $\sz_t$ and $u_t$. For the cost function,
    \vspace*{-3pt}
    \begin{align*}
        c_t = \|\sz_t\|_{\sQ_t}^2 + \|u_t\|_{\sR_t}^2 + \|x_t - \sz_t\|_{\sQ_t}^2 + 2 \ipc{\sz_t}{x_t - \sz_t}_{\sQ_t}. 
    \end{align*}
    Let $b_t = \E[\|x_t - \sz_t\|_{\sQ_t}^2]$ be a constant that depends on system parameters $(\sA_t, \sB_t, \Sigma_{w_t})_{t=0}^{T-1}$, $(\sC_t, \Sigma_{v_t})_{t=0}^{T}$ and $\Sigma_0$. 
    Then, random variable $\gamma_t := \|x_t - \sz_t\|_{\sQ_t}^2 - b_t$ has zero mean. Since $(x_t - \sz_t)$ is Gaussian, its squared norm is subexponential. Since $\sz_t$ and $(x_t - \sz_t)$ are independent zero-mean Gaussian random vectors~\citep{bertsekas2012dynamic}, their inner product, and hence $\eta_t$, are zero-mean subexponential random variables.

    If the system is uniformly exponentially stable (Assumption~\ref{asp:stability}) and the system parameters are regular~(Assumption \ref{asp:bounded}), then $(\sS_{t})_{t=0}^{T}$ given by RDE~\eqref{eqn:s-riccati} has a bounded operator norm determined by system parameters $(\sA_t, \sB_t, \sC_t, \Sigma_{w_t})_{t=0}^{T-1}$, $(\Sigma_{v_t})_{t=0}^{T}$ and $\Sigma_0$~\citep{zhang2021boundedness}.
    Since $\sS_t = \Cov(x_t - \sz_t)$, $\|\gamma_t\|_{\psi_1} = \bigo(d_x^{1/2})$ by Lemma~\ref{lem:subexp}. 
    By Assumption~\ref{asp:stability}, if we apply zero control to the system, then $\|\Cov(\sz_t)\|_2 = \bigo(1)$. By Lemma~\ref{lem:subexp}, $\eta_t = 2\ip{\sz_t}{x_t - \sz_t}_{\sQ_t}$ satisfies $\|\eta_t\|_{\psi_1} = \bigo(d_x^{1/2})$.
\end{proof}

Proposition~\ref{prp:state-est-lds} states that: 1) the dynamics of the state estimates produced by the Kalman filter remains the same as the original system up to noises, determined by $(\sA_t, \sB_t)_{t=0}^{T-1}$; 2) the costs (of the latent model) are still determined by $(\sQ_t)_{t=0}^{T}$ and $(\sR_t)_{t=0}^{T-1}$, up to constants and noises. Hence, a latent model can be parameterized by $((A_t, B_t)_{t=0}^{T-1}, (Q_t)_{t=0}^{T})$ (recall that we assume $(\sR_t)_{t=0}^T$ is known for convenience). Note that observation matrices $(\sC_t)_{t=0}^{T}$ are \emph{not} 
involved.

Now let us take a closer look at the state representation function. 
The Kalman filter can be written as $\sz_{t+1} = \soA_t \sz_t + \soB_t u_t + \sL_{t+1} y_{t+1}$, where $\soA_t = (I - \sL_{t+1} \sC_{t+1}) \sA_t$ and $\soB_t = (I - \sL_{t+1} \sC_{t+1}) \sB_t$.
For $0\le t\le T$, unrolling the \mbox{recursion yields}
\begin{align*}
    \sz_t &= \soA_{t-1} \sz_{t-1} + \soB_{t-1} u_{t-1} + \sL_t y_t \\
    &= [\soA_{t-1} \soA_{t-2} \cdots \soA_0 \sL_0, \ldots, \sL_{t}] [y_{0}; \ldots; y_{t}] 
    + [\soA_{t-1} \soA_{t-2} \cdots \soA_1 \soB_0, \ldots, \soB_{t-1}] [u_{0}; \ldots; u_{t-1}] \\
    &=: \sM_{t} [y_{0:t}; u_{0:(t-1)}],
\end{align*}
where $\sM_t \in \R^{d_x \times ((t+1) d_y + t d_u)}$.
This means the optimal state representation function is \emph{linear} in the history of observations and controls. A state representation function can then be parameterized by matrices $(M_t)_{t=0}^{T}$, and the latent state at step $t$ is given by $z_t = M_t h_t$.

Overall, a policy $\pi$ is a combination of state representation function $(M_t)_{t=0}^{T-1}$ ($M_T$ is not needed) and feedback gain $(K_t)_{t=0}^{T-1}$ in the latent model, so we write $\pi = (M_t, K_t)_{t=0}^{T-1}$ as the composition of the two, and let $\spi = (\sM_t, \sK_t)_{t=0}^{T-1}$ denote the optimal policy. 

\vspace*{-5pt}
\section{Methodology: Cost-driven state representation learning} 
\label{sec:method}

State representation learning involves history data that contains samples of three variables: observation, control input, and cost. Each of them can potentially be used as a \emph{supervision} signal, and be used to define a type of state representation learning algorithms. We summarize our categorization of the methods in the literature as follows. 
\begin{itemize}
    \setlength\itemsep{0em}
    \item \emph{Predicting observations} defines the class of \emph{observation-reconstruction-based} methods, including methods based on Markov parameters (mapping from control actions to observations) in linear systems~\citep{lale2021adaptive, zheng2021sample} and methods that learn a mapping from states to observations in more complex systems~\citep{ha2018world, hafner2019learning, hafner2019dream}. 
    This type of method tends to recover all state components. 
    \item \emph{Predicting actions} defines the class of \emph{inverse model} methods, where the control is predicted from states across different time steps~\citep{mhammedi2020learning, frandsen2022extracting,lamb2022guaranteed}. This type of method tends to recover the control-relevant \mbox{state components}. 
    \item \emph{Predicting (cumulative) costs} defines the class of \emph{cost-driven state representation learning} methods~\citep{zhang2020learning, schrittwieser2020mastering, yang2022discrete}. This type of methods tend to recover the state components relevant to the cost.
\end{itemize}

Our method falls into the cost-driven category.
Compared with Markov parameter-based approaches for linear systems, our approach directly parameterizes the state representation function, without exploiting the structure of the Kalman filter, making our approach closer to empirical practice that was designed for general RL settings. 

\citet{subramanian2020approximate} propose to optimize a simple combination of cost and transition prediction errors to learn what they call the \emph{approximate information state}. That is, we parameterize a state representation function by matrices $(M_t)_{t=0}^{T}$ and a latent model by matrices $((A_t, B_t)_{t=0}^{T-1}, (Q_t)_{t=0}^{T})$ and then solve 
\begin{align}  \label{eqn:emp-loss}
    \min\nolimits_{(M_t, Q_t, b_t)_{t=0}^{T}, (A_t, B_t)_{t=0}^{T-1}} \nlsum_{t=0}^{T} \nlsum_{i=1}^{n} l_t^{(i)},
\end{align}
where $(b_t)_{t=0}^{T}$ are additional scalar parameters to account for noises, and the loss at step $t$ for trajectory $i$ is defined by 
\begin{equation}  \label{eqn:point-loss}
    l_t^{(i)} = \big(\|M_t h_t^{(i)}\|_{Q_t}^2 + \|u_t^{(i)}\|_{\sR_t}^2 + b_t - c_{t}^{(i)} \big)^2 + \big\| M_{t+1} h_{t+1}^{(i)} - A_t M_t h_t^{(i)} - B_t u_t^{(i)} \big\|^2,
\end{equation}
for $0\le t\le T-1$ and $l_T^{(i)} = \big(\|M_T h_T^{(i)}\|_{Q_T}^2 + b_T - c_{T}^{(i)} \big)^2$.
The optimization problem~\eqref{eqn:emp-loss} is nonconvex; even if we can  find a global minimizer, it is unclear how to establish finite-sample guarantees for it. A main finding of this work is that for LQG, we can solve the cost and transition loss optimization problems \emph{sequentially}, with the caveat of using \emph{cumulative} costs. 

\begin{algorithm}[!t]
    \caption{\corel{}: \mbox{Cost-driven state representation learning}}
    \label{alg:corel}
    \begin{algorithmic}[1]
        \State {\bfseries Input:} sample size $n$, input noise magnitude $\sigma_u = \Theta(1)$, singular value threshold $\theta = \Theta(n^{-1/4})$ (hiding dependence on other problem parameters)
        \State Collect $n$ trajectories using $u_t \sim \gauss(0, \sigma_u^2 I)$, for $0\le t\le T-1$, to obtain data in the form of
        \begin{align*}
            \cD_{\raw} = (y_0^{(i)}, u_0^{(i)}, c_0^{(i)}, \ldots, y_{T-1}^{(i)}, u_{T-1}^{(i)}, c_{T-1}^{(i)}, y_T^{(i)}, c_T^{(i)})_{i=1}^{n}
        \end{align*}
        \State Run Algorithm~\ref{alg:repl} with $\cD_{\raw}$ and $\theta$ to obtain state representation function estimate $(\eM_t)_{t=0}^{T}$ and latent state estimates $(\ez_t^{(i)})_{t=0, i=1}^{T, n}$, so that the data are converted to 
        \begin{align*}
            \cD_{\state} = (\ez_0^{(i)}, u_0^{(i)}, c_0^{(i)}, \ldots, \ez_{T-1}^{(i)}, u_{T-1}^{(i)}, c_{T-1}^{(i)}, \ez_T^{(i)}, c_T^{(i)})_{i=1}^{n}
        \end{align*}
        \State Run Algorithm~\ref{alg:sys-id} with $\cD_{\state}$ to obtain system parameter estimates $((\eA_t, \eB_t)_{t=0}^{T-1}, (\eQ_t)_{t=0}^{T})$
        \State Find feedback gains $(\eK_t)_{t=0}^{T-1}$ from $((\eA_t, \eB_t, \sR_t)_{t=0}^{T-1}, (\eQ_t)_{t=0}^{T})$ by RDE~\eqref{eqn:p-riccati}
        \State {\bfseries Return:} policy $\epi = (\eM_t, \eK_t)_{t=0}^{T-1}$
    \end{algorithmic}
\end{algorithm}

Our method is summarized in \corel{} (Algorithm~\ref{alg:corel}). It has three steps: cost-driven state representation function learning (Algorithm~\ref{alg:repl}), latent system identification (Algorithm~\ref{alg:sys-id}), and planning by RDE~\eqref{eqn:p-riccati}.

This three-step approach is very similar to the World Model approach~\citep{ha2018world} used in empirical RL, except that in the first step, instead of using an autoencoder to learn the state representation function, we use cost values to supervise the representation learning.
Most empirical state representation learning methods~\citep{hafner2019learning, hafner2019dream, schrittwieser2020mastering} use cost supervision as one loss term; the special structure of LQG allows us to use it alone and have theoretical guarantees. 

Algorithm~\ref{alg:repl} is the core of our algorithm. Once the state representation function $(\eM_t)_{t=0}^{T}$ is obtained, Algorithm~\ref{alg:sys-id} identifies the latent model using linear and quadratic regressions, followed by planning using RDE~\eqref{eqn:p-riccati} to obtain the controller $(\eK_t)_{t=0}^{T-1}$ from $((\eA_t, \eB_t, \sR_t)_{t=0}^{T-1}, (\eQ_t)_{t=0}^T)$. 
Algorithm~\ref{alg:sys-id} consists of the standard regression procedures.
We explain Algorithm~\ref{alg:repl} below.

\vspace*{-8pt}
\subsection{Learning the state representation function}
\label{sec:learn-repr}

\begin{algorithm}[!t]
    \caption{Cost-driven state representation function learning}
    \label{alg:repl}
    \begin{algorithmic}[1]
        \State {\bfseries Input:} raw data $\cD_{\raw}$, singular value threshold $\theta$
        \State Estimate the state representation function and cost constants by solving $(\eN_t, \eb_t)_{t=0}^{T} \in$
        \begin{equation}  \label{eqn:warm-up}
            \argmin_{(N_t = N_t^{\top}, b_t)_{t=0}^{T}} \sum_{t=0}^{T} \sum_{i=1}^{n}\Big( \bigl\|[y_{0:t}^{(i)}; u_{0:(t-1)}^{(i)}]\bigr\|_{N_t}^2 + \sum_{\tau=t}^{t+k-1} \|u_\tau^{(i)}\|_{\sR_\tau}^2 + b_t - \oc_t^{(i)} \Big)^2,
        \end{equation}
        where $k = 1$ for $0\le t\le l-1$ and $k = m \wedge (T-t+1)$ for $\ell\le t\le T$
        \State Find $\tM_t \in \argmin_{M \in \R^{d_x \times ((t+1)d_y + t d_u)}} \| M^{\top} M - \eN_t\|_{F}$
        \State For all $0\le t \le \ell-1$, set $\eM_t = \svt{}(\tM_t, \theta)$; for all $\ell\le t \le T$, set $\eM_t = \tM_t$
        \State Compute $\ez_t^{(i)} = \eM_t [y_{0:t}^{(i)}; u_{0:(t-1)}^{(i)}]$ for all $t = 0, \ldots, T$ and $i = 1, \ldots, n$
        \State {\bfseries Return:} state representation function estimate $(\eM_t)_{t=0}^{T}$ and latent state estimates $(\ez_t^{(i)})_{t=0, i=1}^{T, n}$
    \end{algorithmic}
\end{algorithm}

\begin{algorithm}[!t] 
    \caption{Latent model identification}
    \label{alg:sys-id}
    \begin{algorithmic}[1]
        \State {\bfseries Input:} data in the form of $(\ez_0^{(i)}, u_0^{(i)}, c_0^{(i)}, \ldots, \ez_{T-1}^{(i)}, u_{T-1}^{(i)}, c_{T-1}^{(i)}, \ez_T^{(i)}, c_T^{(i)})_{i=1}^{n}$
        \State Estimate the system dynamics by $(\eA_t, \eB_t)_{t=0}^{T-1} \in$
        \begin{align}  \label{eqn:sys-id}
             \argmin_{(A_t, B_t)_{t=0}^{T-1}} \sum_{t=0}^{T-1} \sum_{i=1}^{n} \| A_t \ez_t^{(i)} + B_t u_t^{(i)} - \ez_{t+1}^{(i)} \|^2,
        \end{align}
        {picking the minimal-Frobenius-norm solution by pseudoinverse, as in \eqref{eqn:min_norm_equation}}
        \State For all $0\le t\le \ell - 1$ and $t = T$, set $\eQ_t = I_{d_x}$
        \State For all $\ell \le t \le T-1$, obtain $\tQ_t$ by $\tQ_t, \eb_t \in$
        \begin{align}  \label{eqn:q-id}
            \argmin_{Q_t = Q_t^{\top}, b_t} \sum_{i=1}^{n} ( \|\ez_t^{(i)}\|_{Q_t}^2 + \|u_t^{(i)}\|_{\sR_t}^2 + b_t - c_t^{(i)} )^2,
        \end{align}
        and set $\eQ_t = U \max(\Lambda, 0) U^{\top}$, where $\tQ_t = U \Lambda U^{\top}$ is its eigenvalue decomposition
        \State {\bfseries Return:} system parameters $((\eA_t, \eB_t)_{t=0}^{T-1}, (\eQ_t)_{t=0}^{T})$ 
    \end{algorithmic} 
\end{algorithm}

The state representation function is learned via Algorithm~\ref{alg:repl}. 
Given the raw data consisting of $n$ trajectories, Algorithm~\ref{alg:repl} first solves the regression problem~\eqref{eqn:warm-up} to recover the symmetric matrix $\eN_t$. The target $\oc_t$ of regression~\eqref{eqn:warm-up} is defined by 
\begin{align*}
    \oc_t := c_t + c_{t+1} + \ldots + c_{t+k-1},
\end{align*}
where $k = 1$ for $0\le t\le \ell-1$ and $k = m \wedge (T-t+1)$ for $\ell\le t\le T$. The superscript in $\oc_t^{(i)}$ denotes the observed $\oc_t$ in the $i$th trajectory.
The quadratic regression has a closed-form solution, by converting it to linear regression using $\|v\|_P^2 = \ip{vv^{\top}}{P}_F = \ip{\svec(vv^{\top})}{\svec(P)}$.

\vspace{3pt}
\noindent\textbf{Why cumulative cost?}
The state representation function is parameterized by $(M_t)_{t=0}^{T}$ and the latent state at step $t$ is given by $z_t = M_t h_t$. The single-step cost prediction (neglecting control cost $\|u_t\|_{\sR_t}^2$ and constant $b_t$) is given by $\|z_t\|_{Q_t}^2 = h_t^{\top} M_t^{\top} Q_t M_t h_t$. The regression recovers $(\sM_t)^{\top} \sQ_t \sM_t$ as a whole, from which we can recover $(\sQ_t)^{1/2} \sM_t$ up to an orthogonal transformation. If $\sQ_t$ is positive definite and known, then we can further recover $\sM_t$ from it. However, if $\sQ_t$ does not have full rank, information about $\sM_t$ is partially lost, and there is no way to fully recover $\sM_t$ even if $\sQ_t$ is known. 
To see why multi-step cumulative cost helps, define $\soQ_t := \nlsum_{\tau = t}^{t+k-1} \Phi_{\tau, t}^{\top} \sQ_{\tau} \Phi_{\tau, t}$ for the same $k$ above. Under zero control and zero noise, starting from $x_t$ at step $t$, the $k$-step cumulative cost is precisely $\|x_t\|_{\soQ_t}^2$. Under the cost observability assumption (Assumption~\ref{asp:cost-obs}), $(\soQ_t)_{t=0}^{T}$ are \mbox{positive definite}.

\vspace{3pt}
\noindent\textbf{The normalized parameterization.}
Still, since $\soQ_t$ is unknown, even if we recover $(\sM_t)^{\top} \soQ_t \sM_t$ as a whole, it is not viable to extract $\sM_t$ and $\soQ_t$. Such ambiguity is unavoidable; in fact, for every $\soQ_t$ we choose, there is an equivalent parameterization of the system such that the system response is exactly the same. In partially observable LTI systems, it is well-known that the system parameters can only be recovered up to a similarity transform~\citep{oymak2019non}. 
Since every parameterization is correct, we simply choose $\soQ_t = I$, which we refer to as the \emph{normalized parameterization}.
Concretely, let us define $x'_t = (\soQ_t)^{1/2} x_t$. Then, the new parameterization is given by 
\begin{gather*}
    x'_{t+1} = \spA_t x'_t + \spB_t u_t + w'_t, \quad 
    y_t = \spC_t x'_t + v_t, \quad 
    c'_t(x', u) = \|x'\|_{\spQ_t}^2 + \|u\|_{\sR_t}^2,
\end{gather*}
and $c'_T(x') = \|x'\|_{(\sQ_T)'}^2$, where for all $t \ge 0$, 
\begin{gather*}
    \spA_t = (\soQ_{t+1})^{1/2} \sA_t (\soQ_t)^{-1/2}, \quad 
    \spB_t = (\soQ_{t+1})^{1/2} \sB_t, \quad 
    \spC_t = \sC_t (\soQ_t)^{-1/2}, \\
    w'_t = (\soQ_{t+1})^{1/2} w_t, \quad 
    (\sQ_t)' = (\soQ_t)^{-1/2} \sQ_t (\soQ_t)^{-1/2}.
\end{gather*}
One can verify that under the normalized parameterization, the system satisfies Assumptions~\ref{asp:stability},~\ref{asp:ctrl},~\ref{asp:cost-obs}, and~\ref{asp:bounded}, up to a change of some constants in the bounds. Without loss of generality, we assume system~\eqref{eqn:po-ltv} is in the normalized parameterization.

\vspace{3pt}
\noindent\textbf{Low-rank approximate factorization.}
Regression~\eqref{eqn:warm-up} has a closed-form solution  $(\eN_{t}, \eb_t)_{t=0}^{T}$. 
Constants $(\eb_t)_{t=0}^{T}$ account for the variance of the state estimation error, and are not part of the state representation function; 
$d_h \times d_h$ symmetric matrices $(\eN_t)_{t=0}^T$ are estimates of $(\sM_t)^{\top} \sM_t$ under the normalized parameterization, where $d_h = (t+1) d_y + t d_u$. $\sM_t$ can only be recovered up to an orthogonal transformation, as  for any orthogonal $S \in \R^{d_x \times d_x}$, $(S \sM_t)^{\top} S \sM_t = (\sM_t)^{\top} \sM_t$. 

We want to recover $\tM_t$ from $\eN_t$ such that $\eN_t = \tM_t^{\top} \tM_t$. Let $U\Lambda U^{\top} = \eN_t$ be its eigenvalue decomposition. Let $\Sigma := \max(\Lambda, 0)$ be the positive semidefinite diagonal matrix with nonnegative eigenvalues, where ``$\max$'' applies elementwise.
If $d_h \le d_x$, we can construct $\tM_t = [\Sigma^{1/2} U^{\top}; 0_{(d_x - d_h) \times d_h}]$ by padding zeros. If $d_h > d_x$, however, $\rank(\eN_t)$ may exceed $d_x$. Without loss of generality, assume that the diagonal elements of $\Sigma$ are in descending order. Let $\Sigma_{d_x}$ be the top-left $d_x \times d_x$ block of $\Sigma$ and $U_{d_x}$ be the left $d_x$ columns of $U$. By the Eckart-Young-Mirsky theorem, $\tM_t = \Sigma_{d_x}^{1/2} U_{d_x}^{\top}$ provides the best approximation of $\eN_t$ with $\tM_t^{\top} \tM_t$ 
among $d_x \times d_h$ matrices in terms of the Frobenius norm distance.

\vspace{3pt}
\noindent\textbf{Why singular value truncation in the first $\ell$ steps?}
The latent states are used to identify the latent system dynamics, so whether they are sufficiently excited, namely having full-rank covariance, makes a big difference: if not, the system matrices can only be identified partially. Proposition~\ref{prp:full-rank-cov} below confirms that the optimal latent state $\sz_t = \sM_t h_t$ indeed has full-rank covariance for $t \ge \ell$.
\begin{proposition}  \label{prp:full-rank-cov}
    If system~\eqref{eqn:po-ltv} satisfies Assumptions~\ref{asp:ctrl} (controllability) and~\ref{asp:bounded} (regularity), then under control $(u_t)_{t=0}^{T-1}$, where $u_t \sim \gauss(0, \sigma_u^2 I)$, $\sigma_{\min}(\Cov(\sz_t)) = \Omega(\nu^2)$, $\sM_t$ has rank $d_x$ and $\sigma_{\min}(\sM_t) = \Omega(\nu t^{-1/2})$ for all $\ell\le t \le T$.
\end{proposition}

\begin{proof} 
    For $\ell \le t \le T$, unrolling the Kalman filter gives
    \begin{align*}
        \sz_t =\;& \sA_{t-1} \sz_{t-1} + \sB_{t-1} u_{t-1} + \sL_{t} i_{t} \\
        =\;& \sA_{t-1}(\sA_{t-2} \sz_{t-2} + \sB_{t-2} u_{t-2} + \sL_{t-1} i_{t-1} )  + \sB_{t-1} u_{t-1} + \sL_{t} i_{t} \\ 
        =\;& [\sB_{t-1}, \ldots, \sA_{t-1} \sA_{t-2} \cdots \sA_{t-\ell+1} \sB_{t-\ell}] [u_{t-1}; \ldots; u_{t-\ell}]  + \sA_{t-1} \sA_{t-2} \cdots \sA_{t-\ell} \sz_{t-\ell} \\
        &\quad + [\sL_{t}, \ldots, \sA_{t-1} \sA_{t-2} \cdots \sA_{t-\ell+1} \sL_{t-\ell+1}] [i_{t}; \ldots; i_{t-\ell+1}], 
    \end{align*}
    where $(u_\tau)_{\tau=t-\ell}^{t-1}$, $\sz_{t-\ell}$ and $(i_\tau)_{\tau=t-\ell+1}^{t}$ are independent.
    The matrix multiplied by $[u_{t-1}; \ldots; u_{t-\ell}]$ is precisely the controllability matrix $\Phi_{t-1, \ell}^{c}$.
    Then 
    \begin{align*}
        \Cov(\sz_t) = \E[\sz_t (\sz_t)^{\top}] 
        \succcurlyeq \Phi_{t-1, \ell}^{c} \E[[u_{t-1}; \ldots; u_{t-\ell}] [u_{t-1}; \ldots; u_{t-\ell}]^{\top}] (\Phi_{t-1, \ell}^{c})^{\top} = \sigma_u^2 \Phi_{t-1, \ell}^{c} (\Phi_{t-1, \ell}^{c})^{\top}.
    \end{align*}
    By the controllability assumption (Assumption~\ref{asp:ctrl}), $\Cov(\sz_t)$ has full rank and 
    \begin{align*}
        \sigma_{\min}(\Cov(\sz_t)) \ge \sigma_u^2 \nu^2.
    \end{align*}
    On the other hand, since $\sz_t = \sM_t h_t$, 
    \begin{align*}
        \Cov(\sz_t) = \E[\sM_t h_t h_t^{\top} (\sM_t)^{\top}] \preccurlyeq \sigma_{\max}(\E[h_t h_t^{\top}]) \sM_t (\sM_t)^{\top}.
    \end{align*}
    Since $h_t = [y_{0:t}; u_{0:(t-1)}]$ and $(\Cov(y_t))_{t=0}^{T}, (\Cov(u_t))_{t=0}^{T-1}$ have $\bigo(1)$ operator norms 
    by Lemma~\ref{lem:cov-cat}, $\|\Cov(h_t)\| = \|\E[h_t h_t^{\top}]\| = \bigo(t)$.
    Hence, 
    \begin{align*}
        0 < \sigma_u^2 \nu^2 \le \sigma_{\min}(\Cov(\sz_t)) = \bigo(t) \sigma_{d_x}^2(\sM_t).
    \end{align*}
    This implies that $\rank(\sM_t) = d_x$ and $\sigma_{\min}(\sM_t) = \Omega(\nu t^{-1/2})$.
\end{proof}

Proposition~\ref{prp:full-rank-cov} implies that for all $\ell \le t\le T$, $\sN_t$ has rank $d_x$, so if $d_x$ is not provided, this gives a way to discover it.
For $\ell \le t \le T$, Proposition~\ref{prp:full-rank-cov} guarantees that as long as $\tM_t$ is close enough to $\sM_t$, it also has full rank, and so does $\Cov(\tM_t h_t)$. Hence, we simply take the final estimate $\eM_t = \tM_t$.
Without further assumptions, however, there is no such a full-rank guarantee for $(\Cov(\sz_t))_{t=0}^{\ell-1}$ and $(\sM_t)_{t=0}^{\ell-1}$. We make the following minimal assumption to ensure that the minimum positive singular values $(\sigma_{\min}^{+}(\Cov(\sz_t)))_{t=0}^{\ell-1}$ are uniformly lower bounded. Note that $(\Cov(\sz_t))_{t=0}^{\ell-1}$ are not required to have full rank.

\begin{assumption}  \label{asp:m-min-pos-sv}
    For $0\le t\le \ell-1$, $\sigma_{\min}^{+}(\sM_t) \ge \beta > 0$.
\end{assumption}

Still, for $0\le t\le \ell-1$, Assumption~\ref{asp:m-min-pos-sv} does not guarantee the full-rankness of $\Cov(\tM_t h_t)$, not even a lower bound on its minimum positive singular value; that is why we introduce \svt{} that truncates the singular values of $\tM_t$ by a threshold $\theta > 0$. Concretely, we take $\eM_t = (\I_{[\theta, +\infty)}(\Sigma_{d_x}^{1/2}) \odot \Sigma_{d_x}^{1/2}) U_{d_x}^{\top}$. Then, $\eM_t$ has the same singular values as $\tM_t$ except that those below $\theta$ are zeroed. We take $\theta = \Theta(\ell^{3/2} (d_y + d_u) d_x^{3/4} n^{-1/4} \log^{1/4}(\ell/p))$ to ensure a sufficient lower bound on the minimum positive singular value of $\eM_t$,  without increasing the \mbox{statistical errors}.

\section{Theoretical guarantees and proofs}

Theorem~\ref{thm:main} below offers a finite-sample guarantee for our approach, confirming cost-driven state representation learning (Algorithm~\ref{alg:corel}) as a viable path to solving \mbox{LQG control}. 

\begin{theorem}  \label{thm:main}
    Given an unknown LQG control problem defined by~\eqref{eqn:po-ltv} and~\eqref{eqn:q-cost}, under Assumptions~\ref{asp:stability},~\ref{asp:ctrl},~\ref{asp:cost-obs},~\ref{asp:bounded} and~\ref{asp:m-min-pos-sv}, 
    for a given $p \in (0, 1)$, if we run \corel{} (Algorithm~\ref{alg:corel}) for $n \ge \poly(T, d_x, d_y, d_u, \log(1/p))$, then with probability at least $1 - p$, state representation function $(\eM_t)_{t=0}^{T}$ is $\poly(\ell, d_x, d_y, d_u, \log(\ell/p)) n^{-1/4}$-optimal 
    in the first $\ell$ steps, and $\poly(\nu^{-1}, T, d_x, d_y, d_u, \log(T/p)) n^{-1/2}$-optimal in the next $(T - \ell)$ steps. In addition, the overall output policy $\epi = (\eM_t, \eK_t)_{t=0}^{T-1}$ satisties 
    \begin{align*}
        J(\epi) - J(\spi) =\;& \poly(\beta^{-1}, \ell, d_x, d_y, d_u, \log(\ell/p)) c^{\ell} n^{-1/4} \\
        & \qquad + \poly(\nu^{-1}, m, T, d_x, d_y, d_u, \log(nT/p)) n^{-1},
    \end{align*}
    where $c > 1$ is a dimension-free constant depending polynomially on $\ell$ and other problem parameters. 
\end{theorem}

Theorem~\ref{thm:main} provides finite-sample guarantees for both the state representation function $(\eM_t)_{t=0}^{T}$ and the overall policy $\epi = (\eM_t, \eK_t)_{t=0}^{T}$.
From Theorem~\ref{thm:main}, we observe a separation of the sample complexities \emph{before} and \emph{after} time step $\ell$ for the state representation  function, resulting from the loss of the full-rankness of  $(\Cov(\sz_t))_{t=0}^{\ell-1}$ and $(\sM_t)_{t=0}^{\ell-1}$. 

Specifically, quadratic regression guarantees that $\eN_t$ converges to $\sN_t$ at a rate of $n^{-1/2}$ for all $0\le t \le T$ (Lemma~\ref{lem:qr}). Before time step $\ell$, $\eM_t$ suffers a square root decay of the rate  $n^{-1/4}$ because $\sM_t$ may not have rank $d_x$ (Lemma~\ref{lem:mat-fac}). Since $(\ez_t)_{t=0}^{\ell-1}$ may not have full-rank covariances, $(\sA_t)_{t=0}^{\ell-1}$ are only recovered partially. As a result, $(\eK_t)_{t=0}^{\ell-1}$ may not stabilize $(\sA_t, \sB_t)_{t=0}^{\ell-1}$. This issue poses a significant challenge for the analysis and significantly worsens 
dependence on $\ell$ in the policy suboptimality gap (Lemma~\ref{lem:rank-deficient}), which  
implies that if $n$ is not large  enough, the policy may be inferior to zero control in the first $\ell$ steps, as system $(\sA_t, \sB_t)_{t=0}^{\ell-1}$ is uniformly exponential stable (Assumption~\ref{asp:stability}), 
while zero control has a suboptimality gap linear in $\ell$.
After time step $\ell$, $\eM_t$ retains the $n^{-1/2}$ sample complexity, from which the same order of sample complexity guarantee for $(\eA_t, \eB_t)$ follows, resulting in the $\poly(T) n^{-1}$ term in the policy suboptimality gap.

Next, we provide a key proposition and several technical lemmas for analyzing our algorithm, before we present the proof of Theorem~\ref{thm:main}, which contains the exact polynomial orders of the finite-sample guarantees.

\subsection{Proposition on multi-step cumulative costs}

The following proposition establishes the relationship between the multi-step cumulative costs and the state estimates by the Kalman filter.

\begin{proposition}  \label{prp:multi-step-cost}
    Let $(\spz_t)_{t=0}^{T}$ be the state estimates by the Kalman filter under the normalized parameterization.
    If we apply $u_t \sim \gauss(0, \sigma_u^2 I)$ for all $0\le t\le T-1$, then for $0\le t\le T$, 
    \begin{align*}
        \oc_t := c_t + c_{t+1} + \ldots + c_{t+k-1}
        = \|\spz_t\|^2 + \nlsum_{\tau = t}^{t+k-1} \|u_{\tau}\|_{\sR_{\tau}}^2 + \ob_t + \ooe_t,
    \end{align*}
    where $k=1$ for $0\le t \le \ell - 1$ and $k = m \wedge (T - t + 1)$ for $\ell \le t\le T$, $\ob_t = \bigo(k)$, and $\ooe_t$ is a zero-mean subexponential random variable with $\|\ooe_t\|_{\psi_1} = \bigo(k d_x^{1/2})$.
\end{proposition}

\begin{proof}
    By Proposition~\ref{prp:state-est-lds}, $\spz_{t+1} = \spA_t \spz_t + \spB_t u_t + \spL_{t+1} i'_{t+1}$, where $\spL_{t+1}, i'_{t+1}$ are the Kalman gain and the innovation under the normalized parameterization, respectively. Under Assumptions~\ref{asp:stability} and~\ref{asp:bounded}, $(i'_{t})_{t=0}^{T}$ are Gaussian random vectors whose covariances have $\bigo(1)$ operator norms, and $(\spL_t)_{t=0}^{T}$ have $\bigo(1)$ operator norms~\citep{zhang2021boundedness}. Hence, The covariance of $\spL_{t+1} i'_{t+1}$ has $\bigo(1)$ operator norm. 
    Since $u_t \sim \gauss(0, \sigma_u^2 I)$, $j_t := \spB_t u_t + i'_t$ can be viewed as a Gaussian noise vector whose covariance has $\bigo(1)$ operator norm. By Proposition~\ref{prp:state-est-lds},
    \begin{align*}
        c_t = \|\spz_t\|_{\spQ_t}^2 + \|u_t\|_{\sR_t}^2 + b'_t + e'_t,
    \end{align*}
    where $e'_t := \gamma'_t + \eta'_t$ is subexponential with $\|e'_t\|_{\psi_1} = \bigo(d_x^{1/2})$. 
    Let $\Phi'_{t, t_0} = \spA_{t-1} \spA_{t-2} \cdots \spA_{t_0}$ for $t > t_0$ and $\Phi'_{t, t } = I$.
    Then, for $\tau \ge t$, 
    \begin{align*}
        \spz_\tau = \Phi'_{\tau, t} \spz_t + \nlsum_{s = t}^{\tau - 1} \Phi'_{\tau, s} j_{s} := \Phi'_{\tau, t} \spz_t + j'_{\tau, t},
    \end{align*}
    where $j'_{t, t} = 0$ and for $\tau > t$, $j'_{\tau, t}$ is a Gaussian random vector with bounded covariance due to uniform exponential stability (Assumption~\ref{asp:stability}).
    Therefore, 
    \begin{align*}
        \oc_t =\;& \nlsum_{\tau=t}^{t+k-1} c_{\tau} \\
        =\;& \nlsum_{\tau=t}^{t+k-1} \big( \|\Phi'_{\tau, t} \spz_t + j'_{\tau, t}\|_{\spQ_{\tau}}^2 + \|u_{\tau}\|_{\sR_{\tau}}^2 + b'_{\tau} + e'_{\tau} \big) \\
        =\;& (\spz_t)^{\top} \Big(\nlsum_{\tau=t}^{t+k-1} (\Phi'_{\tau, t})^{\top} \spQ_t \Phi'_{\tau, t}\Big) \spz_t + \nlsum_{\tau=t}^{t+k-1} \|u_{\tau}\|_{\sR_\tau}^2 \\
        &\quad + \nlsum_{\tau=t}^{t+k-1} (\|j'_{\tau, t}\|_{\spQ_{\tau}}^2 + (j'_{\tau, t})^{\top} \spQ_{\tau} \Phi'_{\tau, t} \spz_t + b'_{\tau} + e'_\tau) \\
        =\;& \|\spz_t\|^2 + \nlsum_{\tau=t}^{t+k-1} \|u_{\tau}\|_{\sR_\tau}^2 + \ob_t + \ooe_t,
    \end{align*}
    where $\sum_{\tau=t}^{t+k-1} (\Phi'_{\tau, t})^{\top} \spQ_t \Phi'_{\tau, t} = I$ is due to the normalized parameterization, $\ob_t := \sum_{\tau=t}^{t+k-1} (b_{\tau} + \E[\|j'_{\tau}\|_{\spQ_{\tau}}^2]) = \bigo(k)$, and 
    \begin{align*}
        \ooe_t := \sum_{\tau=t}^{t+k-1} \big(\|j'_{\tau}\|_{\spQ_{\tau}}^2 - \E[\|j'_{\tau}\|_{\spQ_{\tau}}^2] + (j'_{\tau})^{\top} \spQ_{\tau} \Phi'_{\tau, t} \spz_t + e'_\tau \big)
    \end{align*}
    has zero mean and is subexponential with $\|\ooe_t\|_{\psi_1} = \bigo(k d_x^{1/2})$.
\end{proof}

\subsection{Quadratic regression bound}
\label{sec:qr}

As noted in \S\ref{sec:learn-repr}, the quadratic regression can be converted to linear regression using $\|h\|_P^2 = \ip{hh^{\top}}{P}_F = \ip{\svec(hh^{\top})}{\svec(P)}$. 
To analyze this linear regression with an intercept, we need the following lemma. We note that a similar lemma without considering the intercept has been proved in~\citep[Proposition 1]{jadbabaie2021time}.

\begin{lemma}  \label{lem:eig-lin-id}
    Let $(h_0^{(i)})_{i=1}^n$ be $n$ independent observations of the $d$-dimensional random vector $h_0 \sim \gauss(0, I_d)$. Let $f_0^{(i)} := \svec(h_0^{(i)} (h_0^{(i)})^{\top})$ and $\of_0^{(i)} := [f_0^{(i)}; 1]$.
    There exists an absolute constant $a > 0$, such that as long as $n \ge a r^4 \log (a r^2 / p)$, with probability at least $1 - p$, 
    \begin{align*}
        \sigma_{\min}\Big(\nlsum_{i=1}^{n} \of_0^{(i)} (\of_0^{(i)})^{\top}\Big) = \Omega(d^{-1} n).
    \end{align*}
\end{lemma}

\begin{proof}
    Let $f_0 = \svec(h_0 h_0^{\top})$ and $\of_0 = [f_0; 1]$.  
    We first show that $\lambda_{\min}(\E[\of_0 \of_0^{\top}])$ is lower bounded.
    Consider 
    \begin{align*}
        \oSigma = \E[\of_0 \of_0^{\top}] = \begin{bmatrix}
            2I_{d(d-1)/2} & 0 & 0 \\
            0 & 2I_d + 1_d 1_d^{\top} & 1_d \\
            0 & 1_d^{\top} & 1
        \end{bmatrix}.
    \end{align*}
    To lower bound its smallest eigenvalue, let us compute its inverse.
    By the Sherman-Morrison formula, 
    \begin{align*}
        (2 I_d + 1_d 1_d^{\top})^{-1} 
        = (2I_d)^{-1} - \frac{(2 I_d)^{-1} 1_d 1_d^{\top} (2I_d)^{-1}}{1 + 1^{\top} (2I_d)^{-1} 1}
        = \frac{1}{2} I_d - \frac{1}{2d + 4} 1_d 1_d^{\top}.
    \end{align*}
    Then, by the inverse of a block matrix,  
    \begin{align*}
        \begin{bmatrix}
            2I_d + 1_d 1_d^{\top} & 1_d \\
            1_d^{\top} & 1
        \end{bmatrix}^{-1}
        = \frac{1}{2} \begin{bmatrix}
            I_d & -1_d \\
            -1_d^{\top} & d+2
        \end{bmatrix}.
    \end{align*}
    Therefore,
    \begin{align*}
        \|\oSigma^{-1}\|_2 \le \;& \frac{1}{2} (\|I_{d(d-1)/2}\|_2 + \|I_d\|_2 + 2 \|-1_d\|_2 + \| d + 2\|_2) \\
        =\;& \frac{1}{2} d + d^{1/2} + 2 \le 4d.
    \end{align*}
    Hence,
    \begin{align*}
        \lambda_{\min}(\oSigma) = \|\oSigma^{-1}\|_2^{-1} \ge (4d)^{-1} = \Omega(d^{-1}).
    \end{align*}
    Then, with the similar concentration arguments to those in~\citep[Appendix A.1]{jadbabaie2021time}, we can show that with probability at least $1 - p$, 
    \begin{align*}
        \lambda_{\min}(F_0^{\top} F_0) = \Omega(d^{-1} n),
    \end{align*}
    which completes the proof. 
\end{proof}

Lemma~\ref{lem:eig-lin-id} lower bounds the minimum singular value of a matrix that contains the fourth powers of elements in standard Gaussian random vectors.
The following lemma is the main result in \S~\ref{sec:qr}.

\begin{lemma}[Quadratic regression]  \label{lem:qr}
     Define random variable $c := (\sh)^{\top} \sN \sh + \starb + e$, where $\sh \sim \gauss(0, \Sigma_{\ast})$ is a $d$-dimensional Gaussian random vector, $\sN \in \R^{d\times d}$ is a positive semidefinite matrix, $\starb \in \R$ is a constant and $e$ is a zero-mean subexponential random variable with $\|e\|_{\psi_1} \le E$. Assume that $\|\sN\|_2 = \bigo(1)$ and that $\lambda_{\min}(\Sigma_{\ast}^{1/2}) \ge \beta = \Omega(1)$.
    Define $h := \sh + \delta$ where the perturbation vector $\delta$ can be correlated with $\sh$ and its $\ell_2$ norm is sub-Gaussian with $\E[\|\delta\|] \le \epsilon$, $\| \|\delta\| \|_{\psi_2} \le \epsilon$.
    Assume that $\epsilon \le \min((d\|\Sigma_{\ast}\|_2)^{1/2}, a (\beta \wedge 1) d^{-3/2} \|\Sigma_{\ast} \|_2^{-1/2} (\log(n/p))^{-1})$ for some absolute constant $a > 0$.
    Suppose we get $n$ observations $h^{(i)}$ and $c^{(i)}$ of $h$ and $c$, where $((\sh)^{(i)})_{i=1}^{n}$ are independent and $(\delta^{(i)})_{i=1}^{n}$ can be correlated. Consider the regression problem 
    \begin{align}  \label{eqn:qr}
        \eN, \eb \in \argmin_{N = N^{\top}, b} \nlsum_{i=1}^{n} ( c^{(i)} - \|h^{(i)}\|_{N}^2 - b)^2.
    \end{align}
    There exists an absolute constant $a_0 > 0$, such that as long as $n \ge a_0 d^4 \log (a_0 d^2 / p) \log (1/p)$, with probability at least $1 - p$, $\|\eN - \sN\|_F$ and $|\eb - \starb|$ are bounded by
    \begin{align*}
        \bigo\big(\epsilon \|\Sigma_{\ast}\|_2^{1/2} (d \log(n/p) + d^{3/2} E \log^2(n/p))
        + d^2 \|\Sigma_{\ast}\|_2 E n^{-1/2} \log^{1/2}(1/p)\big).
    \end{align*}
\end{lemma}

\begin{proof}
    Regression~\eqref{eqn:qr} can be written as
    \begin{align*} 
        \min_{\svec(N), b} \nlsum_{i=1}^{n} \big(c^{(i)} - (\svec(h^{(i)} (h^{(i)})^{\top})^{\top} \svec(N) - b\big)^2.
    \end{align*}
    Let $f^{(i)} := \svec(h^{(i)} (h^{(i)})^{\top})$ denote the covariates and $\of^{(i)} := [f^{(i)}; 1]$ denote the extended covariates. 
    Define $(\starf)^{(i)}$ and $(\sof)^{(i)}$ similarly by replacing $h^{(i)}$ with $(\sh)^{(i)}$.
    Then, regression~\eqref{eqn:qr} can be written as 
    \begin{align}  \label{eqn:qr-svec}
        \min_{\svec(N), b} \nlsum_{i=1}^{n} \big(c^{(i)} - (f^{(i)})^{\top} \svec(N) - b\big)^2.
    \end{align}
    Let $\oF := [\of^{(1)}, \ldots, \of^{(n)}]^{\top}$ be the $n \times \frac{d(d+3)}{2}$ matrix whose $i$th row is $(\of^{(i)})^{\top}$. Define $\soF$ similarly by replacing $\of^{(i)}$ with $(\sof)^{(i)}$.
    Solving linear regression~\eqref{eqn:qr-svec} gives
    \begin{align*}
        \oF^{\top} \oF [\svec(\eN); \eb] = \nlsum_{i=1}^{n} \of^{(i)} c^{(i)}.
    \end{align*}
    Substituting $c^{(i)} = ((\sof)^{(i)})^{\top} [\svec(\sN); \starb] + e^{(i)}$ into the above equation yields
    \begin{align} \label{eqn:en}
        \oF^{\top} \oF [\svec(\eN); \eb] = \oF^{\top} \soF [\svec(\sN); \starb] + \oF^{\top} \xi, 
    \end{align}
    where $\xi$ denotes the vector whose $i$th element is $e^{(i)}$. 
    Rearranging the terms, we have 
    \begin{equation}  \label{eqn:ff-nn}
        \oF^{\top} \oF [\svec(\eN - \sN); \eb - \starb]
        = \oF^{\top} (\soF - \oF) [\svec(\sN); \starb] + \oF^{\top} \xi.
    \end{equation}

    Next, we show that $\oF^{\top} \oF$ is invertible with high probability. 
    We can represent $\sh$ by $\Sigma_{\ast}^{1/2} h_0$, where $h_0 \sim \gauss(0, I_d)$ is an $d$-dimensional standard Gaussian random vector. Correspondingly, an independent observation $(\sh)^{(i)}$ can be expressed as $\Sigma_{\ast}^{1/2} h_0^{(i)}$, where $h_0^{(i)}$ is an independent observation of $h_0$. It follows that $\sh (\sh)^{\top} = \Sigma_{\ast}^{1/2} h_0 h_0^{\top} \Sigma_{\ast}^{1/2}$ and $(\sh)^{(i)} ((\sh)^{(i)})^{\top} = \Sigma_{\ast}^{1/2} h_0^{(i)} (h_0^{(i)})^{\top} \Sigma_{\ast}^{1/2}$. 
    Define $f_0 := \svec(h_0 h_0^{\top})$, $f_0^{(i)} := \svec(h_0^{(i)} (h_0^{(i)})^{\top})$, and $F_0 := [f_0^{(1)}, \ldots, f_0^{(n)}]^{\top}$ be an $n \times \frac{r(r+1)}{2}$ matrix whose $i$th row is $(f^{(i)}_0)^{\top}$. 
    Define $\of_0$, $\of_0^{(i)}$ and $\oF_0$ as the extended counterparts.
    Then, 
    \begin{align*}
        f = \svec(h h^{\top}) = \svec(\Sigma_{\ast}^{1/2} h_0 h_0^{\top} \Sigma_{\ast}^{1/2})
        = (\Sigma_{\ast}^{1/2} \otimes_{s} \Sigma_{\ast}^{1/2}) \svec(h_0 h_0^{\top})
        = \sPhi f_0,
    \end{align*}
    where $\sPhi := \Sigma_{\ast}^{1/2} \otimes_s \Sigma_{\ast}^{1/2}$ is a $\frac{d(d+1)}{2} \times \frac{d(d+1)}{2}$ matrix.
    Then, $\sF = F_0 (\sPhi)^{\top}$.
    By the properties of the symmetric Kronecker product~\citep{schacke2004kronecker}, 
    \begin{align*}
        \sigma_{\min}(\sPhi) = \lambda_{\min}(\Sigma_{\ast}^{1/2})^2 = \lambda_{\min}(\Sigma_{\ast}) = \beta > 0.
    \end{align*}

    By Lemma~\ref{lem:eig-lin-id}, there exist absolute constants $a_0, a_1 > 0$, such that if $n \ge a_0 d^4 \log(a_0 d^2 / p)$, with probability at least $1 - p$, $\lambda_{\min}(\oF_0^{\top} \oF_0) \ge a_1 d^{-1} n$. 
    Since $\of = \diag(\sPhi, 1) \of_0$ and $\soF = \oF_0 \diag((\sPhi)^{\top}, 1)$,
    \begin{align*}
        \lambda_{\min}((\soF)^{\top} \soF) \ge \sigma_{\min}(\diag(\sPhi, 1))^2 a_1 d^{-1} n
        = (\beta^2 \wedge 1) a_1 d^{-1} n.
    \end{align*}
    By Weyl's inequality for singular values, 
    \begin{align*}
        |\sigma_{\min}(\oF) - \sigma_{\min}(\soF)| \le \|\oF - \soF\|_2 = \|F - \sF\|_2.
    \end{align*}
    Hence, we want to bound $\|\sF - F\|_2$, which satisfies 
    \begin{align*}
        \|\sF - F\|_2^2 \le \|\sF - F\|_F^2
        = \nlsum_{i=1}^{n} \|(\sh)^{(i)} ((\sh)^{(i)})^{\top} - h^{(i)} (h^{(i)})^{\top}\|_{F}^2.
    \end{align*}
    Since $\sh (\sh)^{\top} - h h^{\top}$ has at most rank two, we have
    \begin{align*}
        \|\sh (\sh)^{\top} - h h^{\top} \|_F
        \le\;& \sqrt{2} \|\sh (\sh)^{\top} - h h^{\top} \|_2 \\
        =\;& \sqrt{2} \| \sh (\sh - h)^{\top} + (\sh - h) h^{\top} \|_2 \\
        \le\;& \sqrt{2} (\|\sh\| + \|h\|) \|\delta\|.
    \end{align*}
    Since $\sh \sim \gauss(0, \Sigma_{\ast})$, $\|\sh\|$ is sub-Gaussian with its mean and sub-Gaussian norm bounded by $\bigo((d \|\Sigma_{\ast}\|)^{1/2})$. Since $\|\delta\|$ is sub-Gaussian with its mean and sub-Gaussian norm bounded by $\epsilon \le (d \|\Sigma_{\ast}\|)^{1/2}$, we conclude that $\|\sh (\sh)^{\top} - h h^{\top} \|_F$ is subexponential with its mean and subexponential norm bounded by $\bigo(\epsilon (d \|\Sigma_{\ast}\|)^{1/2})$. Hence, with probability at least $1 - p$,
    \begin{align*}
        \|\sh (\sh)^{\top} - h h^{\top} \|_F^2 = \bigo(\epsilon^2 d \|\Sigma_{\ast} \|_2 \log^2(1/p)).
    \end{align*}

    Therefore, by the union bound over $n$ observations, 
    \begin{align*}
        \| \sF - F \|_F^2 = \nlsum_{i=1}^{n} \|(\sh)^{(i)} ((\sh)^{(i)})^{\top} - h^{(i)} (h^{(i)})^{\top}\|_{F}^2
        = \bigo(\epsilon^2 d \|\Sigma_{\ast}\|_2 n \log^2(n/p) ),
    \end{align*}
    which implies that 
    \begin{align*}
        \|\sF - F\|_2 = \bigo(\epsilon (d \|\Sigma_{\ast} \|_2 n)^{1/2} \log(n/p)).
    \end{align*}
    Hence, there exists some absolute constant $a > 0$, such that as long as 
    \begin{align*}
        \epsilon \le a (\beta \wedge 1) d^{-3/2} \|\Sigma_{\ast} \|_2^{-1/2} (\log(n/p))^{-1},
    \end{align*}
    we have 
    \begin{align*}
        |\sigma_{\min}(F) - \sigma_{\min}(\sF)| \le (\beta \wedge 1) a_1 d^{-1/2} n^{1/2} / 2.
    \end{align*}
    Therefore, we further have 
    \begin{align*}
        \lambda_{\min}(\oF^{\top} \oF) = \Omega((\beta^2 \wedge 1) d^{-1} n) = \Omega(d^{-1} n).
    \end{align*}

    Now, we return to~\eqref{eqn:ff-nn}. By inverting $\oF^{\top} \oF$, we obtain 
    \begin{equation} \label{eqn:a_b_def}
        \begin{aligned}
            \|[\svec(\eN - \sN); \eb - \starb]\|
            =\;& \|\oF^{\dagger} (\soF - \oF) [\svec(\sN); \starb] + \oF^{\dagger} \xi\| \\
            \le\;& \underbrace{\|\oF^{\dagger} (\soF - \oF) [\svec(\sN); \starb]\|}_{(a)} + \underbrace{\|\oF^{\dagger} \xi\|}_{(b)}.
        \end{aligned}
    \end{equation}
    Term $(a)$ is upper bounded by
    \begin{align*}
        \sigma_{\min}(\oF)^{-1} \|(\soF - \oF) [\svec(\sN); \starb]\|
        =\;& \bigo(\sigma_{\min}(\oF)^{-1}) \|(\sF - F) \svec(\sN)\| \\
        =\;& \bigo((d^{1/2} n^{-1/2}) \|(\sF - F) \svec(\sN)\|.
    \end{align*}
    Using arguments similar to those in~\citep[Section B.2.13]{mhammedi2020learning}, we have 
    \begin{align*}
        &\|(\sF - F) \svec(\sN)\|^2 \\
        =\;& \sum_{i=1}^{n} \ip{\svec((\sh)^{(i)} ((\sh)^{(i)})^{\top}) - \svec(h^{(i)} (h^{(i)})^{\top})}{\svec(\sN)}^2 \\
        \overset{(i)}{=}\;& \nlsum_{i=1}^{n} \ip{(\sh)^{(i)} ((\sh)^{(i)})^{\top} - h^{(i)} (h^{(i)})^{\top}}{\sN}_F^2 \\
        \overset{(ii)}{\le}\;& \|\sN\|_2^2 \nlsum_{i=1}^{n} \|(\sh)^{(i)} ((\sh)^{(i)})^{\top} - h^{(i)} (h^{(i)})^{\top}\|_{\ast}^2 \\
        \overset{(iii)}{\le}\;& 2 \|\sN\|_2^2 \nlsum_{i=1}^{n} \|(\sh)^{(i)} ((\sh)^{(i)})^{\top} - h^{(i)} (h^{(i)})^{\top}\|_{F}^2 \\
        =\;& 2 \|\sN\|_2^2 \| \sF - F\|_F^2,
    \end{align*}
    where $\ip{\cdot}{\cdot}_F$ denotes the Frobenius product between matrices in $(i)$, $\|\cdot\|_{\ast}$ denotes the nuclear norm in $(ii)$, and $(iii)$ follows from the fact that the matrix $(\sh)^{(i)}((\sh)^{(i)})^{\top} - h^{(i)} (h^{(i)})^{\top}$ has at most rank two.
    Hence, term $(a)$ in \eqref{eqn:a_b_def} is bounded by 
    \begin{align*}
        \bigo\big(d^{1/2} n^{-1/2} \|\sN\|_2 \cdot \epsilon (d \|\Sigma_{\ast}\|_2)^{1/2} n^{1/2} \log(n/p)\big)
        = \bigo\big(\epsilon d \|\Sigma_{\ast}\|_2^{1/2} \log(n/p)\big).
    \end{align*}

    Now we consider term $(b)$ in \eqref{eqn:a_b_def}:
    \begin{align*}
        (b) = \| \oF^{\dagger} \xi \| 
        \le\lambda_{\min}(\oF^{\top}\oF)^{-1} \|\oF^{\top} \xi\|
        = \bigo(d n^{-1}) (\|(\oF - \soF)^{\top} \xi\| + \|(\soF)^{\top} \xi\|).
    \end{align*}
    Since $\xi$ is a vector of zero-mean subexponential variables with subexponential norms bounded by $E$, $\|\xi\| = \bigo(E n^{1/2} \log(n/p))$. Hence, we have
    \begin{align*}
        \|(\oF - \soF)^{\top} \xi\|
        \le\;& \|\oF - \soF\|_2 \|\xi\| \\
        =\;& \bigo(\epsilon (d \|\Sigma_{\ast}\|_2 n)^{1/2} \log(n/p))\cdot \bigo(E n^{1/2} \log(n/p)) \\
        =\;& \bigo(\epsilon (d \|\Sigma_{\ast}\|_2)^{1/2} E n \log^2(n/p) ).
    \end{align*}

    To bound $\|(\soF)^{\top} \xi\| = \|\diag(\sPhi, 1) \oF_0^{\top} \xi\|$, note that $\|\oF_0^{\top} \xi\| = \| \sum_{i=1}^{n} \of_0^{(i)} e^{(i)} \|$.
    Consider the $j$th component in the summation $\sum_{i=1}^{n} [\of_0^{(i)}]_j (e')^{(i)}$. 
    Recall that $f_0 = \svec(h_0 h_0^{\top})$, so $[\of_0]_j$ is either the square of a standard Gaussian random variable, $\sqrt{2}$ times the product of two independent standard Gaussian random variables, or one.
    Hence, $[f_0]_j$ is subexponential with mean and $\|[f_0]_j\|_{\psi_1}$ both bounded by $\bigo(1)$. 
    As a result, the product $[f_0]_j e$ is $\frac{1}{2}$-sub-Weibull, with the sub-Weibull norm being $\bigo(E)$. 
    By~\citep[Theorem 3.1]{hao2019bootstrapping},
    \begin{align*}
        \nlsum_{i=1}^{n} [f_0^{(i)}]_j e^{(i)} = \bigo(E n^{1/2} \log^{1/2}(1/p)).
    \end{align*}
    Hence, the norm of the $d(d+1)/2$-dimensional vector $F_0^{\top} \xi$ is $\bigo(d E n^{1/2} \log^{1/2}(1/p))$.
    By the properties of the symmetric Kronecker product~\citep{schacke2004kronecker}, $\|\sPhi\|_2 = \|\Sigma_{\ast}^{1/2}\|_2^2 = \|\Sigma_{\ast}\|_2$. Then, 
    \begin{align*}
        \|(\soF)^{\top} \xi\| = \bigo(d \|\Sigma_{\ast}\|_2 E n^{1/2} \log^{1/2}(1/p)).
    \end{align*}
    Eventually, term $(b)$ is bounded by 
    \begin{align*}
        &\bigo(d n^{-1}) \cdot \bigo((d \|\Sigma_{\ast}\|_2)^{1/2} E n \log^2(n/p) \epsilon 
        + d n^{1/2} \|\Sigma_{\ast}\|_2 E \log^{1/2}(1/p)) \\
        =\;& \bigo\big( \epsilon d^{3/2} \|\Sigma_{\ast}\|_2^{1/2} E \log^2(n/p)
        + d^2 \|\Sigma_{\ast}\|_2 E n^{-1/2} \log^{1/2}(1/p)\big).
    \end{align*}

    Combining the bounds on $(a)$ and $(b)$, we have 
    \begin{align*}
        &\|[\svec(\eN - \sN); \eb - \starb]\| \\
        =\;& \bigo\big(\epsilon \|\Sigma_{\ast}\|_2^{1/2} (d \log(n/p) + d^{3/2} E \log^2(n/p) )
        + d^2 \|\Sigma_{\ast}\|_2 E n^{-1/2} \log^{1/2}(1/p)\big),
    \end{align*}
    which concludes the proof.
\end{proof}

\subsection{Matrix factorization bound}
\label{sec:mat-fac}

Given two $m \times n$ matrices $A, B$, we are interested in bounding $\min_{S^{\top} S = I} \|S A - B\|_{F}$ using $\|A^{\top} A - B^{\top} B\|_F$. The minimum problem is known as the orthogonal Procrustes problem, solved in~\citep{schoenemann1964solution,schonemann1966generalized}. Specifically, the minimum is attained at $S = UV^{\top}$, where $U \Sigma V^{\top} = B A^{\top}$ is its singular value decomposition.

If $m \le n$ and $\rank(A) = m$, then the following lemma from~\citep{tu2016low} establishes that the distance between $A$ and $B$ is of the same order of $\|A^{\top} A - B^{\top} B\|_F$.

\begin{lemma}[{\citep[Lemma 5.4]{tu2016low}}]  \label{lem:mat-fac-fr}
    For $m \times n$ matrices $A, B$, let $\sigma_m(A)$ denote its $m$th largest singular value. Then 
    \begin{align*}
        \min_{S^{\top} S = I} \|S A - B\|_F
        \le (2 (\sqrt{2} - 1))^{-1/2} \sigma_{m}(A)^{-1} \|A^{\top} A - B^{\top} B\|_F.
    \end{align*}
\end{lemma}

If $\sigma_{\min}(A)$ equals zero, the above bound becomes vacuous. In general, the following lemma shows that the distance is of the order of the square root of the $\|A^{\top} A - B^{\top} B\|_F$, with a multiplicative $\sqrt{d}$ factor, where $d = \min(2m, n)$.

\begin{lemma}  \label{lem:mat-fac}
    For $m\times n$ matrices $A, B$, $\min_{S^{\top} S = I} \|S A - B\|_F^2 \le \sqrt{d} \|A^{\top} A - B^{\top} B\|_F$, where $d = \min(2m, n)$.
\end{lemma}

\begin{proof} 
    Let $U \Sigma V^{\top} = BA^{\top}$ be its singular value decomposition. By substituting the solution $UV^{\top}$ of the orthogonal Procrustes problem, the square of the attained minimum equals 
    \begin{align*}
        \|U V^{\top} A - B\|_F^2 
        &= \ip{U V^{\top} A - B}{U V^{\top} A - B}_{F} \\
        &= \|A\|_F^2 + \|B\|_F^2 - 2 \ip{U V^{\top} A}{B}_F \\
        &\overset{(i)}{=} \|A\|_F^2 + \|B\|_F^2 - 2 \tr(\Sigma) \\
        &= \|A^{\top} A\|_{\ast} + \|B^{\top} B\|_{\ast} - 2 \|BA^{\top}\|_{\ast},
    \end{align*}
    where $(i)$ is due to the property of $U, V$.

    To establish the relationship between $\|A^{\top} A\|_{\ast} + \|B^{\top} B\|_{\ast} - 2 \|BA^{\top}\|_{\ast}$ and $\|A^{\top} A - B^{\top} B\|_F$, we need to operate in the space of singular values.
    For $m \times n$ matrix $M$, let $(\sigma_1(M), \ldots, \sigma_{d'}(M))$ be its singular values in descending order, where $d' = m \wedge n$.

    In terms of singular values,
    \begin{align*}
        \|A^{\top} A\|_{\ast} + \|B^{\top} B\|_{\ast} - 2 \|BA^{\top}\|_{\ast}
        \overset{(i)}{=}\;& \|A^{\top} A + B^{\top} B\|_{\ast} - 2 \|BA^{\top}\|_{\ast} \\
        \overset{(ii)}{=}\;& \nlsum_{i=1}^{d} \sigma_i(A^{\top} A + B^{\top} B) - 2 \nlsum_{i=1}^{d'} \sigma_i(B A^{\top}), 
    \end{align*}
    where $(i)$ holds since $A^{\top}A$ and $B^{\top} B$ are positive semidefinite matrices, and in $(ii)$ $d := \min(2m, n)$ since $\rank(A^{\top} A + B^{\top} B) \le n$ and $\rank(A^{\top} A + B^{\top} B) \le \rank(A^{\top} A) + \rank(B^{\top} B) \le 2m$.
    If $x \ge y > 0$, then $x - y \le \sqrt{x^2 - y^2}$. For all $1\le i\le d$, $2\sigma_i(B A^{\top}) \le \sigma_i(A^{\top} A + B^{\top} B)$~\citep{bhatia1990singular}. Take $\sigma_i(A^{\top} A + B^{\top} B)$ as $x$ and $2\sigma_i(B A^{\top})$ as $y$; it follows that 
    \begin{align*}
        \sigma_i(A^{\top} A + B^{\top} B) - 2\sigma_i(B A^{\top})
        \le \sqrt{\sigma_i^2(A^{\top} A + B^{\top} B) - 4\sigma_i^2(B A^{\top})}.
    \end{align*}
    Let $\sigma_{i}(BA^{\top}) := 0$ for $d'< i \le d$.
    Combining the above yields
    \begin{align*}
        \|A^{\top} A\|_{\ast} + \|B^{\top} B\|_{\ast} - 2 \|BA^{\top}\|_{\ast}
        \le\;& \sum_{i=1}^d \sqrt{\sigma_i^2(A^{\top} A + B^{\top} B) - 4\sigma_i^2(B A^{\top})} \\
        \overset{(i)}{\le}\;& \sqrt{d} \sqrt{\nlsum_{i=1}^d \sigma_i^2(A^{\top} A + B^{\top} B) - 4 \nlsum_{i=1}^{d'} \sigma_i^2 (B A^{\top}) } \\
        \overset{(ii)}{\le}\;& \sqrt{d} \sqrt{\|A^{\top} A + B^{\top} B\|_F^2 - 4\ip{A^{\top} A}{B^{\top} B}_{F}} \\
        \le \;& \sqrt{d} \|A^{\top} A - B^{\top} B\|_F,
    \end{align*}
    where $(i)$ is due to the Cauchy-Schwarz inequality, and $(ii)$ uses 
    \begin{align*}
        \nlsum_{i=1}^{d'} \sigma_i^2(B A^{\top})
        = \|B A^{\top}\|_F^2
        = \tr(AB^{\top} B A^{\top})
        = \tr(A^{\top}A B^{\top} B)
        = \ip{A^{\top} A}{B^{\top} B}_F.
    \end{align*}
    This completes the proof.
\end{proof}

\subsection{Perturbed linear regression bound}
\label{sec:lr}

Identifying the dynamics of the latent model requires solving linear regression~\eqref{eqn:sys-id}. A standard assumption in analyzing linear regression $y = \sA x + e$ is that $\Cov(x)$ has \emph{full rank}. However, as discussed in \S\ref{sec:learn-repr}, we need to handle rank-deficient $\Cov(x)$ in the first $\ell$ steps of system identification. Moreover, the latent state estimates $\ez_t$ contain errors. Both issues are addressed in the following lemma.

\begin{lemma}[Perturbed rank-deficient linear regression]  \label{lem:lr-rd}
    Define random vector $\sy := \sA \sx + e$, where $\sx \sim \gauss(0, \Sigma_{\ast})$ and $e \sim \gauss(0, \Sigma_e)$ are $d_1$ and $d_2$ dimensional Gaussian random vectors, respectively.
    Define $x := \sx + \delta_x$ and $y := \sy + \delta_y$ where the perturbation vectors $\delta_x$ and $\delta_y$ can be correlated with $\sx$ and $\sy$.
    Assume that $\|\sA\|_2$, $\|\Sigma_{\ast}\|_2$ and $\|\Sigma_{e}\|_2$ are $\bigo(1)$.
    Let $\sigma_{\min}^{+}(\Sigma_{\ast}^{1/2}) \ge \beta > 0$.
    Suppose we get $n$ observations $x^{(i)}$ and $y^{(i)}$ of $x$ and $y$, where $((\sx)^{(i)})_{i=1}^{n}$ are independent and $(\delta_x^{(i)}, \delta_y^{(i)})_{i=1}^{n}$ can be correlated.
    Assume that $\sigma_{\min}^{+}(\sum_{i=1}^{n} x^{(i)} (x^{(i)})^{\top}) \ge \theta^2 n$ for some $\theta > 0$ 
    that has at least $n^{-1/4}$ dependence on $n$, and that $\|\sum_{i=1}^{n} \delta_x^{(i)} (\delta_x^{(i)})^{\top}\|_2 = \bigo(\epsilon_x^2 n)$, $\|\sum_{i=1}^{n} \delta_y^{(i)} (\delta_y^{(i)})^{\top}\|_2 = \bigo(\epsilon_y^2 n)$ for $\epsilon_x, \epsilon_y > 0$.
    Consider the minimum Frobenius norm solution
    \begin{align*}
        \eA \in \argmin_{A} \nlsum_{i=1}^{n} \|y^{(i)} - A x^{(i)}\|^2.
    \end{align*}
    Then, there exists an absolute constant $c > 0$, such that if $n \ge c (d_1 + d_2 + \log(1/p))$, with probability at least $1-p$,
    \begin{align*}
        \|(\eA - \sA) \Sigma_{\ast}^{1/2}\|_2
        = \bigo(\beta^{-1} (1 + \theta^{-1} (\epsilon_x + \epsilon_y)) \epsilon_x + \epsilon_y
        + n^{-1/2} (d_1 + d_2 + \log(1/p))^{1/2}).
    \end{align*}
\end{lemma}

\vspace*{3pt}
\begin{proof} 
    Let $r = \rank(\Sigma_{\ast})$ and $\Sigma_{\ast} = D D^{\top}$ where $D \in \R^{d_1 \times r}$.
    We can view $\sx$ as generated from an $r$-dimensional standard Gaussian $g \sim \gauss(0, I_r)$, by $\sx = D g$; $x^{(i)}$ can then be viewed as $D g^{(i)} + \delta_x^{(i)}$, where $(g^{(i)})_{i=1}^{n}$ are independent observations of $g$. Let $X$ denote the matrix whose $i$th row is $(x^{(i)})^{\top}$; $\sX, Y, G, E, \Delta_x, \Delta_y$ are defined similarly. 

    To solve the regression problem, we set its gradient to be zero and substitute $Y = \sX (\sA)^{\top} + E + \Delta_y$ to obtain
    \begin{align*}
        \eA X^{\top} X = \sA (\sX)^{\top} X + E^{\top} X + \Delta_y^{\top} X.
    \end{align*}
    Substituting $X$ by $G D^{\top} + \Delta_x$ gives
    \begin{align*}
        &\eA D G^{\top} G D^{\top} + \eA (D G^{\top} \Delta_x + \Delta_x^{\top} G D^{\top} + \Delta_x^{\top} \Delta_x) \\
        =\;& \sA D G^{\top} G D^{\top} + \sA D G^{\top} \Delta_x  + (E^{\top} + \Delta_y^{\top}) (G D^{\top} + \Delta_x).
    \end{align*}
    By rearranging the terms, we have 
    \begin{align*}
        &(\eA - \sA) D G^{\top} G D^{\top} \\
        =\;& \sA D G^{\top} \Delta_x - \eA (D G^{\top} \Delta_x + \Delta_x^{\top} G D^{\top} + \Delta_x^{\top} \Delta_x) + (E^{\top} + \Delta_y^{\top}) (G D^{\top} + \Delta_x).
    \end{align*}
    Hence, 
    \begin{align*}
        \|(\eA - \sA) D\|_2
        = &\big\| \big( \sA D G^{\top} \Delta_x - \eA (D G^{\top} \Delta_x + \Delta_x^{\top} G D^{\top} + \Delta_x^{\top} \Delta_x) \\
        &\quad + (E^{\top} + \Delta_y^{\top}) (G D^{\top} + \Delta_x) \big) (D^{\top})^{\dagger} (G^{\top} G)^{-1} \big\|_2.
    \end{align*}
    For $\|\eA\|_2$, we make the following claim, whose proof is deferred to~\S\ref{sec:pf-a-norm}.
    \begin{claim}  \label{clm:a-norm}
        As long as $n \ge 16(d_1 + d_2 + \log(1/p))$, with probability at least $1 - 4p$, 
        \begin{align*}
            \|\eA\|_2 = \bigo\big(\|\sA\|_2 (1 + \theta^{-1} \epsilon_x) + \theta^{-1} \epsilon_y\big).
        \end{align*}
    \end{claim}
    The proof of Claim~\ref{clm:a-norm} is deferred to~\S\ref{sec:pf-a-norm}, where we analyze the inversion of $X^{\top} X$.
    Note that we may also bound $\|\eA\|_2$ by analyzing the inversion of $DG^{\top}GD^{\top}$, potentially without the requirement on $\theta$, but this requires stronger condition on $\epsilon_x$ and $\epsilon_y$, and the bound is not directly applicable to the minimum Frobenius norm solution. The requirement on $\theta$ remains beneficial for numerical stability. 

    Then, since $\|D^{\dagger}\|_2 = (\sigma_{\min}^{+}(\Sigma_{\ast}^{1/2}))^{-1} \le \beta^{-1}$, 
    \begin{align*}
        \|(\eA - \sA) D\|_2
        =\;& \bigo(\beta^{-1} \|\eA\|_2) \big(\|D\|_2 \|G\|_2 \|\Delta_x\|_2 + \|\Delta_x\|_2^2 ) \|(G^{\top} G)^{-1}\|_2\big) \\
        &\quad + \bigo(1) \big(\| G^{\dagger} E \|_2 + \|G^{\dagger} \Delta_y\|_2
        + \beta^{-1} (\|E\|_2 + \|\Delta_y\|_2) \|\Delta_x\|_2 \|(G^{\top} G)^{-1}\|_2\big).
    \end{align*}
    By~\citep[Theorem 6.1]{wainwright2019high}, the Gaussian ensemble $G$ satisfies that with probability at least $1-p$, 
    \begin{align*}
        \|G\|_2 \le\;& (n \|I_r\|_2 )^{1/2} + (\tr(I_r))^{1/2}
        + (2 \|I_r\|_2 \log (1/p))^{1/2} \\
        \le \;& n^{1/2} + d_1^{1/2} + (2 \log(1/p))^{1/2}, \\
        \sigma_{\min}(G) \ge\;& (n \lambda_{\min}(I_r) )^{1/2} - (\tr(I_r))^{1/2}
        - (2 \lambda_{\min}(I_r) \log (1/p))^{1/2} \\
        \ge \;& n^{1/2} - d_1^{1/2} - (2 \log(1/p))^{1/2}).
    \end{align*}
    Since $n \ge 8d_1 + 16 \log(1/p)$, we have $\|G\|_2 = \bigo(n^{1/2})$ and $\sigma_{\min}(G) = \Omega(n^{1/2})$.
    It follows that $\|(G^{\top} G)^{-1}\|_2 = \bigo(n^{-1})$ and $\|G^{\dagger}\|_2 = \bigo(n^{-1/2})$.
    Similarly, $\|E\|_2 = \bigo((\|\Sigma_e\|_2 n)^{1/2})$.
    Note that $\Delta_x^{\top} \Delta_x = \sum_{i=1}^{n} \delta_x^{(i)} (\delta_x^{(i)})^{\top}$. By our assumption, $\|\Delta_x\|_2 = \bigo(\epsilon_x n^{1/2})$. Similarly, $\|\Delta_y\|_2 = \bigo(\epsilon_y n^{1/2})$.
    Hence, 
    \begin{align*}
        \|(\eA - \sA) D\|_2
        =\;& \bigo(\beta^{-1} \|\eA\|_2) (\|D\|_2 \epsilon_x + \epsilon_x^2)
        + \bigo(1) (\|G^{\dagger} E\|_2 + \epsilon_y + \beta^{-1} (\|\Sigma_e^{1/2}\|_2 + \epsilon_y) \epsilon_x) \\
        =\;& \bigo(\beta^{-1} \|\eA\|_2 \epsilon_x + \epsilon_y + \|G^{\dagger} E\|_2),
    \end{align*}
    where we consider $\epsilon_x$ and $\epsilon_y$ as quantities much smaller than one such that terms like $\epsilon_x^2, \epsilon_x \epsilon_y$ are absorbed into $\epsilon_x, \epsilon_y$. It remains to control $\|G^{\dagger} E\|_2$. It is proved in~\citep[Section B.2.11]{mhammedi2020learning} via a covering number argument that with probability at least $1-p$,
    \begin{align*}
        \|G^{\dagger} E\|_2 =\;& \bigo(1) \sigma_{\min}(G)^{-1} \|\Sigma_e^{1/2}\|_2 (d_1 + d_2 + \log(1/p))^{1/2} \\
        =\;& \bigo( n^{-1/2} (d_1 + d_2 + \log(1/p))^{1/2}).
    \end{align*}
    Overall, we obtain that
    \begin{align*}
        \|(\eA - \sA) \Sigma_{\ast}^{1/2}\|_2
        =\;&\|(\eA - \sA) D\|_2 \\
        =\;& \bigo(\beta^{-1} \|\eA\|_2 \epsilon_x + \epsilon_y + n^{-1/2} (d_1 + d_2 + \log(1/p))^{1/2}) \\
        =\;& \bigo(\beta^{-1} (1 + \theta^{-1} (\epsilon_x + \epsilon_y)) \epsilon_x + \epsilon_y
        + n^{-1/2} (d_1 + d_2 + \log(1/p))^{1/2}),
    \end{align*}
    which completes the proof.
\end{proof}

\subsubsection{Proof of Claim~\ref{clm:a-norm}}
\label{sec:pf-a-norm}

    The minimum Frobenius norm solution $\eA$ is given by the following closed-form expression in terms of the pseudoinverse (Moore-Penrose inverse):
    \begin{align} 
        \eA =\;& Y^{\top} X (X^{\top} X)^{\dagger} \label{eqn:min_norm_equation} \\
        =\;& (\sA (\sX)^{\top} + E^{\top} + \Delta_y^{\top}) X (X^{\top} X)^{\dagger} \nonumber \\
        =\;& \sA (X^{\dagger} \sX)^{\top} + (X^{\dagger} E)^{\top} + (X^{\dagger} \Delta_y)^{\top} \nonumber \\
        =\;& \sA (X^{\dagger} X)^{\top} - \sA (X^{\dagger} \Delta_x)^{\top} + (X^{\dagger} E)^{\top} + (X^{\dagger} \Delta_y)^{\top}. \nonumber
    \end{align}
    Then 
    \begin{align*}
        \|\eA\|_2 \le\;& \|\sA\|_2 + \|X^{\dagger} E\|_2
        + (\|\sA\|_2 \|\Delta_x\|_2 + \|\Delta_y\|_2) (\sigma_{\min}^{+}(X))^{-1},
    \end{align*}
    where we note that $\|X^{\dagger}\|_2 = \sigma_{\min}^{+}(X)^{-1}$ when $X \neq 0$. 
    Since $\sigma_{\min}^{+}(X) = (\sigma_{\min}^{+}(X^{\top} X))^{1/2} \ge \theta n^{1/2}$, 
    \begin{align*}
        (\sigma_{\min}^{+}(X))^{-1} \le \theta^{-1} n^{-1/2}.
    \end{align*}
    By our assumption, 
    \begin{align*}
        \|\Delta_x\|_2 = \bigo(\epsilon_x n^{1/2}), \quad 
        \|\Delta_y\|_2 = \bigo(\epsilon_y n^{1/2}).
    \end{align*}
    Similar to the proof of Lemma~\ref{lem:lr-rd}, by~\citep[Section B.2.11]{mhammedi2020learning}, with probability at least $1-p$,
    \begin{align*}
        \|X^{\dagger} E\|_2 =\;& \bigo(1) \sigma_{\min}^{+}(X)^{-1} \|\Sigma_e^{1/2}\|_2 (d_1 + d_2 + \log(1/p))^{1/2} \\
        =\;& \bigo(\theta^{-1} n^{-1/2} (d_1 + d_2 + \log(1/p))^{1/2}).
    \end{align*}
    Combining the bounds above, we obtain
    \begin{align*}
        \|\eA\|_2
        \le\;& \|\sA\|_2 + \bigo(\theta^{-1} n^{-1/2} (d_1 + d_2 + \log(1/p))^{1/2})
        + \bigo((\|\sA\|_2 \epsilon_x + \epsilon_y) n^{1/2} \theta^{-1} n^{-1/2}) \\
        =\;& \bigo(\|\sA\|_2 (1 + \theta^{-1} \epsilon_x) + \theta^{-1} \epsilon_y)
        + \bigo(\theta^{-1} n^{-1/2} (d_1 + d_2 + \log(1/p))^{1/2}).
    \end{align*}
    Hence, as long as $\theta$ has at least $n^{-1/4}$ dependence on $n$, 
    \begin{align*}
        \|\eA\|_2 =\;& \bigo\big(\big(\|\sA\|_2 (1 + \theta^{-1} \epsilon_x) + \theta^{-1} \epsilon_y\big)\big).
    \end{align*}

\endproof

In Lemma~\ref{lem:lr-rd}, if $\Sigma_{\ast}$ has full rank and $\lambda_{\min}(\Sigma_{\ast}) \ge \beta > 0$ and $\lambda_{\min}(\sum_{i=1}^{n} x^{(i)} (x^{(i)})^{\top}) = \Omega(\beta^2 n)$, then 
\begin{align*}
    \| \eA - \sA \|_2
    = \bigo(\beta^{-1}((\beta^{-1} \epsilon_x + \epsilon_y) \log^{1/2}(n/p)
    + n^{-1/2} (d_1 + d_2 + \log(1/p))^{1/2})).
\end{align*}
The following lemma shows that we can strengthen the result by removing the $\beta^{-1}$ factor \mbox{before $\epsilon_x$}. 

\begin{lemma}[Perturbed linear regression]  \label{lem:lr}
    Define random variable $\sy = \sA \sx + e$, where $\sx \sim \gauss(0, \Sigma_{\ast})$ and $e \sim \gauss(0, \Sigma_e)$ are $d_1$ and $d_2$ dimensional random vectors. Assume that $\|\sA\|_2$, $\|\Sigma_{\ast}\|_2$ and $\|\Sigma_{e}\|_2$ are $\bigo(1)$, and $\lambda_{\min}(\Sigma_{\ast}^{1/2}) \ge \beta > 0$. Define $x := \sx + \delta_x$ and $y := \sy + \delta_y$ where the perturbation vectors $\delta_x$ and $\delta_y$ can be correlated with $\sx$ and $\sy$.
    Suppose we get $n$ independent observations $x^{(i)}, y^{(i)}$ of $x$ and $y$. 
    Assume that $\lambda_{\min}(\sum_{i=1}^{n} x^{(i)} (x^{(i)})^{\top}) = \Omega(\beta^2 n)$, $\|\sum_{i=1}^{n} \delta_x^{(i)} (\delta_x^{(i)})^{\top}\|_2 = \bigo(\epsilon_x^2 n)$, $\|\sum_{i=1}^{n} \delta_y^{(i)} (\delta_y^{(i)})^{\top}\|_2 = \bigo(\epsilon_y^2 n)$ for $\beta, \epsilon_x, \epsilon_y > 0$.
    Consider the minimum Frobenius norm solution  
    \begin{align*}
        \eA \in \argmin_{A} \nlsum_{i=1}^{n} (y^{(i)} - A x^{(i)})^2.
    \end{align*}
    Then, there exists an absolute constant $c > 0$, such that if $n \ge c (d_1 + d_2 + \log(1/p))$, with probability at least $1-p$,
    \begin{align*}
        \|\eA - \sA\|_2
        = \bigo\big(\beta^{-1} \big(\epsilon_x + \epsilon_y + n^{-1/2} (d_1 + d_2 + \log(1/p))^{1/2}\big)\big).
    \end{align*}
\end{lemma}

\vspace*{2pt}

\begin{proof}
    Following the proof of Claim~\ref{clm:a-norm}, we have 
    \begin{align*}
        \|\eA - \sA\|_2 \le \|\sA\| \|X^{\dagger} \Delta_x\|_2 + \|X^{\dagger} E\|_2 + \|X^{\dagger} \Delta_y\|_2.
    \end{align*}
    Combining with the bounds on $\|X^{\dagger} \Delta_x\|_2$, $\|X^{\dagger} \Delta_y\|_2$ and $\|X^{\dagger} E\|_2$ concludes the proof.
\end{proof}

\vspace*{-10pt}
\subsection{Certainty equivalent linear quadratic control}

As shown in Lemma~\ref{lem:lr-rd}, if the input of linear regression does not have full-rank covariance, then the parameters can only be identified in certain directions. 
The following lemma studies the performance of the certainty equivalent optimal controller in this case.

\begin{lemma}[Rank deficient linear quadratic control]  \label{lem:rank-deficient}
    Consider the finite-horizon time-varying linear dynamical system given by the first equation in~\eqref{eqn:po-ltv}
    with quadratic costs given by~\eqref{eqn:q-cost}, under Assumption~\ref{asp:stability}, Assumption~\ref{asp:cost-obs} with $\ell = T$, and Assumption~\ref{asp:bounded} on relevant problem parameters.
    Let $\Sigma_t := \Cov(x_t)$ 
    under control $u_t \sim \gauss(0, \sigma_u^2 I)$ for all $t\ge 0$, which may not have full rank.
    Let $(\eA_t, \eB_t)_{t=0}^{T-1}, (\eQ_t)_{t=0}^{T}$ satisfy $\|(\eA_t - \sA_t) (\Sigma_t)^{1/2}\|_2 \le \eps$, $\|\eB_t - \sB_t\|_2 \le \eps$, and $\|\eQ_t - \sQ_t\|_2 \le \eps$, $\|\eA_t\|_2 \le c$ for \mbox{dimension-free} constant $c > 0$, and $\eQ_t \succ 0$ for all $t \ge 0$.
    Let $(\sK_t)_{t=0}^{T-1}$ and $(\eK_t)_{t=0}^{T-1}$ be the optimal feedback gains of system $((\sA_t, \sB_t, \sR_t)_{t=0}^{T-1}, (\sQ_t)_{t=0}^{T})$ and system $((\eA_t, \eB_t, \sR_t)_{t=0}^{T-1}, (\eQ_t)_{t=0}^{T})$, respectively.

    Let $J(K)$ denote the expected cumulative cost in system $((\sA_t, \sB_t, \sR_t)_{t=0}^{T-1}, (\sQ_t)_{t=0}^{T})$ under state feedback control $u_t = K_t x_t$ for $t \ge 0$, and $J^{\delta}(\eK)$ denote the the expected cumulative cost in system $((\sA_t, \sB_t, \sR_t)_{t=0}^{T-1}, (\sQ_t)_{t=0}^{T})$ under control $u_t = K_t (x_t + \delta_t)$, where $\delta_t$ is a perturbation, for $t \ge 0$.
    Under control $u_t = K_t (x_t + \delta_t)$ for all $t \ge 0$, let $\Xi_t(K) := \Cov(x_t)$ in system $((\sA_t, \sB_t, \sR_t)_{t=0}^{T-1}, (\sQ_t)_{t=0}^{T})$, and assume that $\delta_t$ follows an arbitrary zero-mean Gaussian distribution, can be correlated with $x_t$, and satisfies $\|\Cov(\delta_t)\|_2^{1/2} \le \varphi \max_{\tau \le t}\|\Xi_{\tau}(K)\|_2^{1/2}$ for $t\ge 0$.
    There exist dimension-free constants $\kappa, a_2, a_3, a_4 > 1$ that depends polynomially on $c$ and other problem parameters, such that under the condition of $\eps, \varphi$ being small enough to ensure $\|\eB_t\|_2, \|\eQ_t\|_2$ are of order $\bigo(1)$, $\lambda_{\min}(\eQ_t) = \Omega(1)$ for all $t \ge 0$ and $\eps + \varphi = \bigo((\kappa a_2 + a_3^2)^{-T})$,
    we have
    \begin{align*}
        J^{\delta}(\eK) - J(\sK) = \bigo((\eps + \varphi) d_x a_4^T).
    \end{align*}
\end{lemma}

\begin{proof}
    First of all, we note that $\Sigma_t$, the covariance of $x_t$ under control $u_t \sim \gauss(0, \sigma_u^2 I)$ in system $(\sA_t, \sB_t)_{t=0}^{T-1}$, has $\bigo(1)$ operator norm due to Assumption~\ref{asp:stability} and has covered all the possible directions of the state covariance $\Xi_t(K)$ in system $(\sA_t, \sB_t)_{t=0}^{T-1}$ under control $u_t = K_t (x_t + \delta_t)$, for all $t \ge 0$ for any $K$. Concretely, we claim the following.
    \begin{claim}  \label{clm:coverage}
    For any $K$ and $t\geq 0$, there exists some dimension-free $\const_t > 0$ 
    that depends on $K$ and other problem parameters, such that $\Xi_t(K) \preccurlyeq \const_t \Sigma_t$.        
    \end{claim}
    \vspace*{2pt}
    The proof of Claim~\ref{clm:coverage} is deferred to \S\ref{sec:pf-coverage}, which also applies to more general control inputs $(u_t)_{t\ge 0}$.
    Now that $\const_t$ exists, a sufficiently large $\const_t$ is given by 
    \begin{align}  \label{eqn:xi-sigma}
       \const_t= \sigma_{\min}^{+}(\Sigma_t)^{-1} \|\Xi_t(K)\|_2 = \bigo(\|\Xi_t(K)\|_2)    
    \end{align}
    as $\Xi_t(K) \preccurlyeq \|\Xi_t(K)\|_2 I$ and $\sigma_{\min}^{+}(\Sigma_t) I \preccurlyeq \Sigma_t$.  
    We shall revisit the bounds on $(\const_t)_{t\ge 0}$ later in the proof of Claim~\ref{clm:xi-exi} in~\S\ref{sec:pf-xi-exi}.
    By the definition of the operator norm, we have $\|(\eA_t - \sA_t) \Xi_t^{1/2}(\eK)\|_2 \le \|(\eA_t - \sA_t) (\const_t \Sigma_t)^{1/2}\|_2 = \bigo(\eps \const_t^{1/2})$ for all $0 \le t\le T-1$.

    Next, we establish bounds on $(\sK_t)_{t=0}^{T-1}$ and $(\eK_t)_{t=0}^{T-1}$. 
    Since $(\sA_t)_{t=0}^{T-1}$ is uniformly exponentially stable, the optimal feedback gains $(\sK_t)_{t=0}^{T-1}$ have $\bigo(1)$ operator norms,
    due to the backward recursion with $(\sA_t)_{t=0}^{T-1}$ as multipliers in ~\eqref{eqn:p-riccati} and~\eqref{eqn:K-riccati}.
    However, as $(\eA_t)_{t=0}^{T-1}$ may not be uniformly exponentially stable, such arguments do not apply to $(\eK_t)_{t=0}^{T-1}$;
    in the worse case, since $((\eA_t, \eB_t, \sR_t)_{t=0}^{T-1}, (\eQ_t)_{t=0}^{T})$ have $\bigo(1)$ operator norms, we can only guarantee that there exists a dimension-free constant $\kappa > 1$ that depends polynomially on $c$ and other problem parameters, such that $\|\eK_t\|_2 = \bigo(\kappa^{T - t})$, from the backward recursion in ~\eqref{eqn:p-riccati} and~\eqref{eqn:K-riccati}.

    Let $\eXi_t(K)$ denote the covariance of $x_t$ under feedback controller $u_t = K_t (x_t + \delta_t)$ in system $(\eA_t, \eB_t)_{t=0}^{T-1}$.
    To bound $\Xi_t(K)$, $\eXi_t(K)$ and their difference for either  $K = \sK$ and $K = \eK$, 
    let us define for $t > t_0$, 
    \begin{align*}
        \ePhi_{t, t_0}(K) :=\;& (\eA_{t-1} + \eB_{t-1} K_{t-1}) 
        \cdots (\eA_{t_0} + \eB_{t_0} K_{t_0}), \\
        \sPhi_{t, t_0}(K) :=\;& (\sA_{t-1} + \sB_{t-1} K_{t-1}) 
        \cdots (\sA_{t_0} + \sB_{t_0} K_{t_0}),
    \end{align*}
    with $\ePhi_{t, t}(K) = \sPhi_{t, t}(K):= I$, $\ePhi_{t, t_0}(K) = \sPhi_{t, t_0}(K) = 0$ for $t < t_0$, and introduce the notation $x_t(\tau, t_0, x, K)$, where $\tau \in [t_0, t]$ denotes the time step at which the system dynamics \emph{switch} from $(\sA_t, \sB_t)_{t=0}^{T-1}$ to $(\eA_t, \eB_t)_{t=0}^{T-1}$, as:
    \begin{align*}
        x_t(\tau, t_0, x, K)
        :=\;& \ePhi_{t, \tau}(K) \sPhi_{\tau, t_0}(K) x + \nlsum_{t_1 = \tau}^{t-1} \ePhi_{t, t_1+1}(K) (w_{t_1} + \eB_{t_1} K_{t_1} \delta_{t_1}) \\
        &\quad+ \nlsum_{t_1 = t_0}^{\tau - 1} \ePhi_{t, \tau}(K) \sPhi_{\tau, t_1+1}(K) (w_{t_1} + \sB_{t_1} K_{t_1} \delta_{t_1}).
    \end{align*}
    Then, we can precisely relate and define
    \begin{gather*}
        \Xi_t(K) := \Cov(x_t(t, 0, x_0, K)), \quad 
        \eXi_t(K) := \Cov(x_t(0, 0, x_0, K)).
    \end{gather*}
    For bounding $\Xi_t(K)$, $\eXi_t(K)$ and their difference, we have the following claim.
    \begin{claim}  \label{clm:xi-exi}
        There exists dimension-free constants $a_2, a_3 > 1$ that depend polynomially on $c$ and other problem parameters, such that for $K = \eK$, for all $t \ge 0$, $\|\Xi_t(\eK)\|_2 = \bigo(a_2^{T})$, $\|\eXi_t(\eK)\|_2 = \bigo(a_2^{T})$, and $\|\Xi_t(\eK) - \eXi_t(\eK)\|_2 = \bigo((\eps \varphi) (\kappa a_2^2)^{T})$; 
        and that for $K = \sK$, for all $t \ge 0$, $\|\Xi_t(\sK)\|_2 = \bigo(t)$, $\|\eXi_t(\sK)\|_2 = \bigo(t)$, and $\|\Xi_t(\sK) - \eXi_t(\sK)\|_2 = \bigo((\eps + \varphi) t a_3^{2t})$.
    \end{claim}
    \vspace*{2pt}
    The proof of Claim~\ref{clm:xi-exi} is deferred to~\S\ref{sec:pf-xi-exi}. 
    A core idea of the proof is to use inverse telescoping to express $\Xi_t(K) - \eXi_t(K)$.

    Recall that $J^{\delta}(K)$ denotes the expected cumulative cost in system $((\sA_t, \sB_t, \sR_t)_{t=0}^{T-1}, (\sQ_t)_{t=0}^{T})$ under state feedback controller $u_t = K_t(x_t + \delta_t)$ for $t \ge 0$. 
    Similarly, let $\eJ^{\delta}(K)$ denote the corresponding expected cumulative cost in system $((\eA_t, \eB_t, \sR_t)_{t=0}^{T-1}, (\eQ_t)_{t=0}^{T})$.
    Notice that 
    \begin{align*}
        J^{\delta}(K)
        = \E \Bigl[\nlsum_{t=0}^{T} c_t\Bigr]
        =\;& \E \Bigl[\nlsum_{t=0}^{T} x_t^{\top} (\sQ_t)' x_t + \nlsum_{t=0}^{T-1} \delta_t^{\top} K_t^{\top} \sR_t K_t (2 x_t + \delta_t) \Bigr] \\
        =\;& \nlsum_{t=0}^{T} \ip{(\sQ_t)'}{\Xi_t(K)}_F + \Delta(K), 
    \end{align*}
    where $(\sQ_t)' := (\sQ_t + K_t^{\top} \sR_t K_t)$ for $0\le t\le T-1$, $(\sQ_T)' := \sQ_T$, and
    \begin{align*}
        \Delta(K) := \nlsum_{t=0}^{T-1}\ip{K_t^{\top}\sR_t K_t}{\E[(2x_t(t, 0, x_0, K) + \delta_t)\delta_t^{\top}]}_F.
    \end{align*}
    By Lemma~\ref{lem:cs-rvec},
    \begin{align*}
        \|\E[x_t(t, 0, x_0, K)\delta_t^{\top}]\|_2 \le \|\Xi_{t}(K)\|_2^{1/2} \cdot \|\Cov(\delta_t)\|_2^{1/2}
        \le \varphi \|\Xi_{t}(K)\|_2^{1/2} \cdot \max_{\tau \le t} \|\Xi_{\tau}(K)\|_2^{1/2}.
    \end{align*}
    Hence,
    \begin{align*}
        |\Delta(K)|
        \le \nlsum_{t=0}^{T-1} d_x \|K_t^{\top} \sR_t K_t\|_2 \cdot 3 \varphi \|\Xi_t(K)\|_2 
        = \bigo(\varphi d_x \nlsum_{t=0}^{T-1} \|K_t\|_2^2 \|\Xi_{t}(K)\|_2^{1/2} \cdot \max_{\tau \le t} \|\Xi_{\tau}(K)\|_2).
    \end{align*}
    Similarly, 
    \begin{align*}
        \eJ^{\delta}(K)
        = \nlsum_{t=0}^{T} \ip{\eQ'_t}{\eXi_t(K)}_F + \eDelta(K),
    \end{align*}
    where $\eQ'_t = \eQ_t + K_t^{\top} \sR_t K_t$ for $0\le t\le T - 1$, $\eQ'_T = \eQ_T$, and 
    \begin{align*}
        \eDelta(K) := \nlsum_{t=0}^{T-1} \ip{K_t^{\top}\sR_t K_t}{\E[(2x_t(0, 0, x_0, K) + \delta_t)\delta_t^{\top}]}_F.
    \end{align*}
    Moreover, 
    \begin{align*}
        |\eDelta(K)|
        = \bigo(\varphi d_x \nlsum_{t=0}^{T-1} \|K_t\|_2^2 \|\eXi_{t}(K)\|_2^{1/2} \cdot \max_{\tau \le t}\|\Xi_{\tau}(K)\|_2^{1/2}).
    \end{align*}
    Hence, for $K = \eK$, we have
    \begin{align*}
        &\big|J^{\delta}(\eK) - \eJ^{\delta}(\eK)\big| \\
        =\;& \Big|\nlsum_{t=0}^{T} \ip{(\sQ_t)' - \eQ'_t}{\Xi_t(\eK)}_F
        + \nlsum_{t=0}^{T} \ip{\eQ'_t}{\Xi_t(\eK) - \eXi_t(\eK)}_F + \Delta(\eK) - \eDelta(\eK)\Big| \\
        \overset{(i)}{=}\;& d_x \nlsum_{t=0}^{T} \bigo(\eps \cdot a_2^T)
        + d_x \nlsum_{t=0}^{T} \bigo\big(\kappa^{2(T-t)} \cdot (\eps + \varphi) (\kappa a_2^2)^{T}\big)
        + \varphi d_x \nlsum_{t=0}^{T-1} \bigo(a_2^T \kappa^{2(T-t)}) \\ 
        =\;& \bigo\big(\eps d_x T a_2^T + (\eps + \varphi) d_x (\kappa^3 a_2^2)^{T} + \varphi d_x (\kappa^2 a_2)^{T} \big) \\
        =\;& \bigo((\eps + \varphi) d_x (\kappa^3 a_2^2)^T),
    \end{align*}
    where $(i)$ uses the fact that $\|\eQ'_t\|_2 = \bigo(\|\eK_t\|_2^2) = \bigo(\kappa^{2(T-t)})$.
    For $K = \sK$, we have 
    \begin{align*}
        &\big|J^{\delta}(\sK) - \eJ^{\delta}(\sK)\big| \\
        =\;& \Big|\nlsum_{t=0}^{T} \ip{(\sQ_t)' - \eQ'_t}{\Xi_t(\sK)}_F
        + \nlsum_{t=0}^{T} \ip{\eQ'_t}{\Xi_t(\sK) - \eXi_t(\sK)}_F + \Delta(\sK) - \eDelta(\sK) \Big| \\
        \overset{(i)}{=}\;& d_x \nlsum_{t=0}^{T} \bigo(\eps \cdot t) + d_x \nlsum_{t=0}^{T} \bigo\big((\eps + \varphi) t a_3^{2t} \big)
        + \varphi d_x \nlsum_{t=0}^{T} \bigo(t) \\
        =\;& \bigo\big(\eps d_x T^2 + (\eps + \varphi) d_x T a_{3}^{2T} + \varphi d_x T^2\big) \\
        =\;& \bigo((\eps + \varphi) d_x a_3^{3T}),
    \end{align*}
    where $(i)$ uses the fact that  $\|\eQ'_t\|_2 = \bigo(\|\sK_t\|_2^2) = \bigo(1)$.
    Define $a_4 = \max(\kappa^3 a_2^2, a_3^3)$. Then, whether $K = \sK$ or $K = \eK$, \mbox{we have} 
    \begin{align*}
        \big|J^{\delta}(K) - \eJ^{\delta}(K)\big|=\;& \bigo((\eps + \varphi) d_x a_4^T).
    \end{align*} 

    Let $\eXi_t^{0}(K)$ denote the covariance of $x_t$ under feedback controller $u_t = K_t x_t$ in system $(\eA_t, \eB_t)_{t=0}^{T-1}$, without perturbation.
    To bridge the difference in the expected cumulative costs with and without perturbations, we claim the following. 
    \begin{claim}   \label{clm:exi-0}
        For all $t \ge 0$, 
        \begin{align*}
            \|\eXi_t(\eK) - \eXi_t^{0}(\eK)\|_2 = \bigo(\varphi a_2^T (\kappa a_2)^{T-t}).
        \end{align*}
    \end{claim}
    \vspace*{2pt}
    The proof of Claim~\ref{clm:exi-0}, similar to that of Claim~\ref{clm:xi-exi}, is deferred to~\S\ref{sec:pf-exi-0}.
    Hence,
    \begin{align*}
        |\eJ^{\delta}(\eK) - \eJ(\eK)|
        =\;& \Big|\nlsum_{t=0}^{T} \ip{\eQ'_t}{\eXi_t(\eK) - \eXi^0_t(\eK)}_F + \eDelta(\eK)\Big| \\
        =\;& \bigo\Big(d_x \nlsum_{t=0}^{T} \kappa^{2(T-t)} \cdot \varphi a_2^T (\kappa a_2)^{T-t}
        + \varphi d_x \nlsum_{t=0}^{T} a_2^T \kappa^{2(T-t)} \Big) \\
        =\;& \bigo(\varphi d_x (\kappa^3 a_2^2)^T + \varphi d_x (\kappa^2 a_2)^T) \\
        =\;& \bigo(\varphi d_x (\kappa^3 a_2^2)^T).
    \end{align*}

    Combining the above results, we have 
    \begin{align*}
        &J^{\delta}(\eK) - J(\sK) \\
        =\;& J^{\delta}(\eK) - \eJ^{\delta}(\eK) + \eJ^{\delta}(\eK) - \eJ(\eK)
        + \eJ(\eK) - \eJ(\sK) + \eJ(\sK) - J(\sK) \\ 
        \overset{(i)}{=}\;& \bigo((\eps + \varphi) d_x (\kappa^3 a_2^2)^T + \varphi d_x (\kappa^3 a_2^2)^T + 0 + (\eps + \varphi) d_x a_3^{3T}) \\
        =\;& \bigo((\eps + \varphi) d_x a_4^T),
    \end{align*}
    where $\eJ(\eK) - \eJ(\sK) \le 0$ in $(i)$ due to the optimality of $\eK$ in the  system $((\eA_t, \eB_t, \sR_t)_{t=0}^{T-1}, (\eQ_t)_{t=0}^{T})$.

\vspace*{-3pt}
\end{proof}

\subsubsection{Proof of Claim~\ref{clm:coverage}}
\label{sec:pf-coverage}
    We prove the claim by induction.
    For time step $0$, clearly, $\Xi_0 = \Sigma_0$ and $\const_0 = 1$. 
    Suppose for time step $t$, $\Xi_t \preccurlyeq \const_t \Sigma_t$.
    Since $x_{t+1} = (\sA_t + \sB_t K_t) x_t + \sB_t K_t \delta_t + w_t$, we have that for step $t+1$, 
    \begin{align*}
        \Sigma_{t+1} =\;& \sA_t \Sigma_t (\sA_t)^{\top} + \sigma_u^2 \sB_t (\sB_t)^{\top} + \Sigma_{w_t}, \\
        \Xi_{t+1}(K) =\;& (\sA_t + \sB_t K_t) \Xi_t(K) (\sA_t + \sB_t K_t)^{\top} + \Sigma_{w_t} 
        + \sB_t K_t \Cov(\delta_t) (\sB_t K_t)^{\top} \\
        &\quad + \E\Big[(\sA_t + \sB_t K_t) x_t (\sB_t K_t \delta_t)^{\top} \Big]
        + \E\Big[\sB_t K_t \delta_t x_t^{\top} (\sA_t + \sB_t K_t)^{\top}\Big].
    \end{align*}
    Since for any $F, G$ and positive semidefinite $P$, we have $F P G^{\top} + G P F^{\top} \preccurlyeq FPF^{\top} + GPG^{\top}$ as  $(F - G) P (F - G)^{\top} \succcurlyeq 0$. 
    Then, 
    \begin{align*}
        &\E\Big[(\sA_t + \sB_t K_t) x_t (\sB_t K_t \delta_t)^{\top} + \sB_t K_t \delta_t x_t^{\top} (\sA_t + \sB_t K_t)^{\top}\Big] \\
        \preccurlyeq\;& (\sA_t + \sB_t K_t) \Xi_t (\sA_t + \sB_t K_t)^{\top} + \sB_t K_t \Cov(\delta_t) (\sB_t K_t)^{\top},
    \end{align*}
    and
    \begin{align*}
        (\sA_t + \sB_t K_t) \Xi_t (\sA_t + \sB_t K_t)^{\top}
        \preccurlyeq 2 \sA_t \Xi_t (\sA_t)^{\top} + 2\sB_t K_t \Xi_t K_t^{\top} (\sB_t)^{\top}.
    \end{align*}
    Hence, 
    \begin{align*}
        \Xi_{t+1}(K)
        \preccurlyeq\;& 2 (\sA_t + \sB_t K_t) \Xi_t (\sA_t + \sB_t K_t)^{\top} + 2 \sB_t \Delta_t^u (\sB_t)^{\top} + \Sigma_{w_t} \\
        \preccurlyeq\;& 4 \sA_t \Xi_t(K) (\sA_t)^{\top} + 4\sB_t K_t \Xi_t K_t^{\top} (\sB_t)^{\top}
        + 2 \sB_t K_t \Cov(\delta_t) (\sB_t K_t)^{\top} + \Sigma_{w_t}.
    \end{align*}
    Hence, for $\const_{t+1} \ge 1$,
    \begin{align*}
        \const_{t+1} \Sigma_{t+1} - \Xi_{t+1}(K)
        \succcurlyeq\;& (\const_{t+1} - 4 \const_t) \sA_t \Sigma_t (\sA_t)^{\top} + \const_{t+1} \sigma_u^2 \sB_t (\sB_t)^{\top} \\
        &\quad - 4 \const_t\sB_t K_t \Sigma_t K_t^{\top} (\sB_t)^{\top} - 2 \sB_t K_t \Cov(\delta_t) (\sB_t K_t)^{\top}. 
    \end{align*}
    To ensure $\const_{t+1} \Sigma_{t+1} \succcurlyeq \Xi_{t+1}(K)$, it suffices to take $\const_{t+1} \ge 4 \const_t (1 + \sigma_u^{-2} \|K_t\|_2^2 \|\Sigma_t\|_2) + 2 \sigma_u^{-2} \|K_t\|_2^2 \|\Cov(\delta_t)\|_2$, which completes the proof. 
\endproof

\subsubsection{Proof of Claim~\ref{clm:xi-exi}} 
\label{sec:pf-xi-exi}
    By the definition of $x_t(\tau, t_0, x, K)$, we have that for $t_2 \le t_1 \le t$, 
    \begin{align*}
        x_t(t_1, 0, x_0, K) 
        = x_t(t_1, t_2, x_{t_2}(t_2, 0, x_0, K), K).
    \end{align*}
    Hence, for $0\le \tau< t$,
    \begin{align*}
        &\Cov(x_t(\tau + 1, 0, x_0, K)) - \Cov(x_t(\tau, 0, x_0, K)) \\
        =\;& \Cov(x_t(\tau + 1, \tau, x_{\tau}(\tau, 0, x_0, K), K)) 
        - \Cov(x_t(\tau, \tau, x_{\tau}(\tau, 0, x_0, K), K)).
    \end{align*}
    Notice that
    \begin{align*}
        x_t(\tau + 1, \tau, x_{\tau}(\tau, 0, x_0, K), K)
        =\;& \ePhi_{t, \tau + 1}(K) (\sA_{\tau} + \sB_{\tau} K_{\tau}) x_{\tau}(\tau, 0, x_0, K)
        + \ePhi_{t, \tau + 1}(K) (w_{\tau} + \sB_{\tau} K_{\tau} \delta_{\tau}) \\ 
        &\quad + \nlsum_{t_0=\tau + 1}^{t-1} \ePhi_{t, t_0 + 1}(K) (w_{t_0} + \eB_{t_0} K_{t_0} \delta_{t_0})
    \end{align*}
    and 
    \begin{align*}
        x_t(\tau, \tau, x_{\tau}(\tau, 0, x_0, K), K)
        = \ePhi_{t, \tau}(K) x_{\tau}(\tau, 0, x_0, K)
        +\nlsum_{t_0=\tau}^{t-1} \ePhi_{t, t_0 + 1}(K) (w_{t_0} + \eB_{t_0} K_{t_0} \delta_{t_0}).
    \end{align*}
    As $x_t(\tau, 0, x_0, K)$ and $(w_{t_0})_{t_0\ge \tau}$ have the same distributions in the above expressions of $x_t(\tau + 1, \tau, x_{\tau}(\tau, 0, x_0, K), K)$ and $x_t(\tau, \tau, x_{\tau}(\tau, 0, x_0, K), K)$, and covariances are expectations, we can couple $x_t(\tau, 0, x_0, K)$ and $(w_{t_0})_{t_0\ge \tau}$ therein without changing the two covariances. On the other hand, $(\delta_{t_0})_{t_0\ge \tau}$ may not have the same distributions in the two trajectories generating $x_t(\tau + 1, \tau, x_{\tau}(\tau, 0, x_0, K), K)$ and $x_t(\tau, \tau, x_{\tau}(\tau, 0, x_0, K), K)$, and we let $(\delta_{t_0}^{(\tau+1)})_{t_0 \ge \tau}, (\delta_{t_0}^{(\tau)})_{t_0 \ge \tau}$ denote the perturbations in the respective trajectory.
    As a shorthand, define
    \begin{align*}
        \delta_{t, \tau} :=\;&x_t(\tau + 1, \tau, x_{\tau}(\tau, 0, x_0, K), K)
        - x_t(\tau, \tau, x_{\tau}(\tau, 0, x_0, K), K) \\
        =\;& \ePhi_{t, \tau + 1}(K) (\sA_{\tau} - \eA_{\tau} + (\sB_{\tau} - \eB_{\tau}) K_{\tau}) x_{\tau}(\tau, 0, x_0, K)
        + \ePhi_{t, \tau+1}(K) (\sB_{\tau} - \eB_{\tau}) K_{\tau} \delta_{\tau}^{(\tau+1)} \\
        &\quad + \nlsum_{t_0=\tau}^{t-1} \ePhi_{t, t_0+1}(K) \eB_{t_0} K_{t_0} (\delta_{t_0}^{(\tau+1)} - \delta_{t_0}^{(\tau)}).
    \end{align*}
    Then, by Lemma~\ref{lem:cov-diff},
    \begin{align}
        &\|\Cov(x_t(\tau + 1, \tau, x_{\tau}(\tau, 0, x_0, K), K))
        - \Cov(x_t(\tau, \tau, x_{\tau}(\tau, 0, x_0, K), K))\|_2 \nonumber \\
        \le \;& \|\Cov(\delta_{t, \tau})\|_{2} 
        + 2 \|\Cov(\delta_{t, \tau})\|_{2}^{1/2} \cdot \|\Cov(x_t(\tau, \tau, x_{\tau}(\tau, 0, x_0, K), K))\|_2^{1/2}. \label{eqn:os-cov-diff}
    \end{align}
    By Lemma~\ref{lem:cov-sum},
    \begin{align}
        \|\Cov(\delta_{t, \tau})\|_2^{1/2}
        \le\;& \|\ePhi_{t, \tau + 1}(K) (\sA_{\tau} - \eA_{\tau} + (\sB_{\tau} - \eB_{\tau}) K_{\tau}) \Xi_{\tau}^{1/2}(K) \|_2 \label{eqn:cov-delta} \\
        &\quad + \|\ePhi_{t, \tau+1}(K) (\sB_{\tau} - \eB_{\tau}) K_{\tau} \|_2 \cdot \big\| \Cov\big(\delta_{\tau}^{(\tau+1)}\big)\big\|_2^{1/2} \nonumber \\
        &\quad + \nlsum_{t_0=\tau}^{t-1} \|\ePhi_{t, t_0+1}(K) \eB_{t_0} K_{t_0}\|_2 \cdot \Big(\big\|\Cov\big(\delta_{t_0}^{(\tau+1)}\big)\big\|_2^{1/2}
        + \big\|\Cov\big(\delta_{t_0}^{(\tau)}\big)\big\|_2^{1/2}\Big) \nonumber
    \end{align} 
    and
    \begin{align}
        &\|\Cov(x_t(\tau, \tau, x_{\tau}(\tau, 0, x_0, K), K))\|_2^{1/2} \nonumber \\
        =\;& \Big\|\Cov\Big(\ePhi_{t, \tau}(K) x_{\tau}(\tau, 0, x_0, K)
        + \nlsum_{t_0 = \tau}^{t-1} \ePhi_{t, t_0 + 1}(K) (w_{t_0} + \eB_{t_0} K_{t_0} \delta_{t_0})\Big) \Big\|_2^{1/2} \nonumber \\
        \le\;& \|\ePhi_{t, \tau}(K) \Xi_{\tau}^{1/2}(K)\|_2 + \nlsum_{t_0 = \tau}^{t-1} \| \ePhi_{t, t_0 + 1}(K) \Sigma_{w_{t_0}}^{1/2} \|_2 \label{eqn:cov-x-tau-tau} \\
        &\quad + \nlsum_{t_0=\tau}^{t-1} \|\ePhi_{t, t_0+1}(K) \eB_{t_0} K_{t_0} \|_2 \cdot \|\Cov(\delta_{t_0})\|_2^{1/2}. \nonumber
    \end{align}

    Next, we consider $K = \eK$ and $K = \sK$ separately.

    For $K = \eK$, let $(\eP_t)_{t=0}^{T}$ be the optimal value matrices for system $((\eA_t, \eB_t, \sR_t)_{t=0}^{T-1}, (\eQ_t)_{t=0}^{T})$, given by~\eqref{eqn:p-riccati} with the system matrices $((\sA_t, \sB_t, \sR_t)_{t=0}^{T-1}, (\sQ_t)_{t=0}^{T})$ replaced. Since $\|\eA_t\|_2 \le c$ and $\|\eB_t\|_2, \|\sR_t\|_2, \|\eQ_t\|_2$ are of order $\bigo(1)$ for all $t \ge 0$, there exists some dimension-free constant $a_2 > 1$ depending polynomially on $c$ and other problem parameters, such that $\|\eP_t\|_2 = \bigo(a_2^{T - t})$, due to the backward recursion in~\eqref{eqn:p-riccati}.
    For any $0\le \tau\le t$ and $x_{\tau} \in \R^{d_x}$, by the definition of $(\hat{P}_t)_{t\ge 0}$, 
    \begin{align*}
        \|x_{\tau} \|_{\eP_{\tau}}^2 \ge \|\ePhi_{t, \tau}(\eK) x_{\tau}\|_{\eP_{t}}^2.
    \end{align*}
    Maximizing both sides for $x_{\tau} \in \s^{d_x - 1}$ and by the definition of the operator norm, we have 
    \begin{align*}
        \|\ePhi_{t, \tau}(\eK)^{\top} \eP_{t} \ePhi_{t, \tau}(\eK) \|_2 \le \|\eP_{\tau}\|_2.
    \end{align*}
    Since $\lambda_{\min}(\eQ_t) = \Omega(1)$, $\lambda_{\min}(\eP_t) = \Omega(1)$. 
    Hence,
    \begin{align*}
        \|\ePhi_{t, \tau}(\eK)\|_2 = \bigo\Big(a_2^{(T-\tau)/2}\Big).
    \end{align*}

    Now we are ready to bound $\Xi_t(\eK)$, $\const_t$, $\|\Xi_t(\eK) - \eXi_t(\eK)\|_2$ and $\eXi_t(\eK)$ by induction. 
    First, $\|\Xi_0(\eK)\|_2 = \|\Sigma_0\|_2 = \bigo(1) = \bigo(a_2^T)$ and $\const_0 = 1 = \bigo(a_2^T)$.
    Suppose $\|\Xi_\tau(\eK)\|_2 = \bigo(a_2^T)$ and $\const_{\tau} = \bigo(a_{5}^{T})$ for all $\tau < t$.
    Then, for all $\tau < t$, $\|\Cov(\delta_{\tau})\|_2 = \bigo(\varphi^2 \max_{t_0\le \tau}\|\Xi_{t_0}(\eK)\|_2) = \bigo(\varphi^2 a_2^{T})$, and~\eqref{eqn:cov-delta} is \mbox{bounded by} 
    \begin{align*}
        &\bigo\big( \eps \const_{\tau}^{1/2} \big(\kappa a_2^{1/2}\big)^{T - \tau} 
        + \eps \varphi a_2^{T/2} \big(\kappa a_2^{1/2} \big)^{T - \tau}
        + \varphi a_2^{T/2} \nlsum_{t_0=\tau}^{t-1} \big(\kappa a_2^{1/2}\big)^{T - t_0} \big) \\ 
        =\;& \bigo\big( \eps \const_{\tau}^{1/2} \big(\kappa a_2^{1/2}\big)^{T - \tau} + \varphi a_2^{T/2} \big(\kappa a_2^{1/2}\big)^{T - \tau} \big), \\
        =\;& \bigo\big( (\eps + \varphi) a_2^{T/2} \big(\kappa a_2^{1/2}\big)^{T - \tau} \big),
    \end{align*}
    where we recall that $\|\eK_t\|_2 = \bigo(\kappa^{T - t})$ and the assumption that $\|\eB_t - \sB_t\|_2 \le \eps = \bigo((\kappa a_2)^{-T}) = \bigo(1)$. Similarly, for $\tau < t$, we have~\eqref{eqn:cov-x-tau-tau} is bounded by 
    \begin{align*}
        &\bigo(\const_{\tau}^{1/2} a_2^{(T - \tau)/2} + a_2^{(T-\tau)/2} + \varphi a_2^{T/2} \nlsum_{t_0=\tau}^{t-1}(\kappa a_2^{1/2})^{T-t_0} ) \\
        =\;& \bigo(a_{5}^{T/2} a_2^{(T - \tau)/2} + a_2^{(T-\tau)/2} + \varphi a_2^{T/2} (\kappa a_2^{1/2})^{T-\tau}) \\
        =\;& \bigo(a_2^{T/2} a_2^{(T-\tau)/2}),
    \end{align*}
    due to $\|\Sigma_{w_t}\|_2 = \bigo(1)$ and $\varphi = \bigo((\kappa a_2)^{-T}) = \bigo(\kappa^{-T})$.
    Hence, by~\eqref{eqn:os-cov-diff} and our assumption that $\eps + \varphi = \bigo((\kappa a_2)^{-T}) = \bigo(\kappa^{-T})$, we have
    \begin{align*}
        \|\Cov(x_t(\tau + 1, 0, x_0, \eK)) - \Cov(x_t(\tau, 0, x_0, \eK))\|_2
        = \bigo((\eps + \varphi) a_2^T (\kappa a_2)^{T-\tau}).
    \end{align*}
    Using inverse telescoping, we have 
    \begin{align*}
        \Xi_t(\eK) - \eXi_t(\eK)
        =\;& \Cov(x_t(t, 0, x_0, \eK)) - \Cov(x_t(0, 0, x_0, \eK)) \\
        =\;& \nlsum_{\tau=0}^{t-1} (\Cov(x_t(\tau + 1, 0, x_0, \eK)) - \Cov(x_t(\tau, 0, x_0, \eK))).
    \end{align*}
    Then,  
    \begin{align*}
        \|\Xi_t(\eK) - \eXi_t(\eK)\|_2 =\;& \bigo\Big((\eps + \varphi) a_2^T \nlsum_{\tau = 0}^{t - 1} (\kappa a_2)^{T - \tau}\Big) \\
        =\;& \bigo((\eps + \varphi) (\kappa a_2^2)^{T}).
    \end{align*}

    On the other hand, since
    \begin{align*}
        \eXi_t(\eK)
        = \Cov(x_t(0, 0, x_0, \eK))
        = \Cov\Big(\ePhi_{t, 0}(\eK) x_0 + \nlsum_{\tau=0}^{t-1} \ePhi_{t, \tau+1}(\eK) (w_{\tau} + \eB_{\tau} K_{\tau} \delta_{\tau})\Big),
    \end{align*}
    and the operator norms of $\Sigma_0$ and $(\Sigma_{w_t})_{t \ge 0}$ are $\bigo(1)$, by Lemma~\ref{lem:cov-sum}, we have
    \begin{align*}
        \|\eXi_t(\eK)\|_2^{1/2} =\;& \bigo(a_2^{T/2} + a_2^{T/2} + \varphi a_2^{T/2} \nlsum_{\tau=0}^{t-1} \big(\kappa a_2^{1/2}\big)^{T-\tau}) \\
        =\;& \bigo(a_2^{T/2} + a_2^{T/2} + \varphi (\kappa a_2)^{T}) \\
        =\;& \bigo(a_2^{T/2}),
    \end{align*}
    due to $\varphi = \bigo((\kappa a_2)^{-T}) = \bigo((\kappa a_2^{1/2})^{-T})$.

    Therefore, $\|\Xi_t(\eK)\|_2 \le \|\eXi_t(\eK)\|_2 + \|\Xi_t(\eK) - \eXi_t(\eK)\|_2 = \bigo(a_{5}^{T} + (\eps + \varphi) (\kappa a_2^2)^{T})$. Since $\eps + \varphi = \bigo((\kappa a_2)^{-T})$, 
    we have $\|\Xi_t(\eK)\|_2 = \bigo(a_{5}^T)$.
    Then, by~\eqref{eqn:xi-sigma}, $\const_t = \bigo(\|\Xi_t(\eK)\|_2) = \bigo(a_{5}^T)$.
    By induction, we have that for all $0\le t\le T$, $\const_t = \bigo(a_2^T)$ and $\|\eXi_{t}(\eK) - \Xi_t(\eK)\|_2 = \bigo( (\eps + \varphi) (\kappa a_2^2)^{T})$.

    For $K = \sK$, we first bound $\Xi_t(\sK)$. 
    Let $(\sP_t)_{t = 0}^{T}$ be the optimal value matrices for system $((\sA_t, \sB_t, \sR_t)_{t=0}^{T-1}, (\sQ_t)_{t=0}^{T})$, given by~\eqref{eqn:p-riccati}. Since $((\sA_t, \sB_t, \sR_t)_{t=0}^{T-1}, (\sQ_t)_{t=0}^{T})$ have $\bigo(1)$ operator norms and $(\sA_t)_{t = 0}^{T-1}$ is uniformly exponentially stable, $\|\sP_t\|_2 = \bigo(1)$ for all $t\ge 0$, due to the backward recursion given by~\eqref{eqn:p-riccati}. 
    For any $0\le \tau\le t$ and $x_{\tau}\in \R^{d_x}$, since the value  at time step $\tau$ is no less than the value at time step $t \ge \tau$, we \mbox{have that}
    \begin{align*}
        \|x_{\tau}\|_{\sP_{\tau}}^2 \ge \|\sPhi_{t, \tau}(\sK) x_{\tau} \|_{\sP_{t}}^2.
    \end{align*}
    Maximizing both sides for $x_{\tau} \in \s^{d_x - 1}$ and by the definition of the operator norm, we have
    \begin{align*}
        \|(\sPhi_{t, \tau}(\sK))^{\top} \sP_{t} \sPhi_{t, \tau}(\sK) \|_2 \le \|\sP_{\tau}\|_2.
    \end{align*}
    Since $\lambda_{\min}(\sQ_t) = \Omega(1)$, $\lambda_{\min}(\sP_t) = \Omega(1)$. Hence,
    \begin{align}  \label{eqn:phi-o1}
        \|\sPhi_{t, \tau}(\sK)\|_2 = \bigo(1). 
    \end{align}
    Now we show that $\|\Xi_t(\sK)\|_2 = \bigo(t)$ by induction. 
    First, $\|\Xi_0(\sK)\|_2 = \|\Sigma_0\|_2 = \bigo(1)$.
    Assume $\|\Xi_{\tau}(\sK)\|_2 = \bigo(\tau^{1/2})$ for all $\tau < t$. Then, since 
    \begin{align*}
        \Xi_t(\sK)
        = \Cov(x_t(t, 0, x_0, \sK))
        = \Cov\Big(\sPhi_{t, 0}(\sK) x_0 + \nlsum_{\tau=0}^{t-1} \sPhi_{t, \tau+1}(\sK) (w_{\tau} + \sB_{\tau} \sK_{\tau} \delta_{\tau})\Big),
    \end{align*}
    by Lemma~\ref{lem:cov-sum},
    \begin{align*}
        \|\Xi_t(\sK)\|_2^{1/2}
        = \Big\|\Cov \Big(\sPhi_{t, 0}(\sK) x_0 + \nlsum_{\tau = 0}^{t - 1} \sPhi_{t, 0}(\sK) w_{\tau} \Big)\Big\|_2^{1/2}
        + \nlsum_{\tau = 0}^{t - 1} \|\Cov(\sPhi_{t, 0}(\sK) \sB_{\tau} \sK_{\tau} \delta_{\tau})\|_2^{1/2}.
    \end{align*}
    Since the operator norms of $\Sigma_0$, $(\Sigma_{w_t})_{t\ge 0}$ are $\bigo(1)$, by the independence of $x_0$ and $(w_{\tau})_{\tau \ge 0}$, \mbox{we have} 
    \begin{align*}
        \Big\|\Cov \Big(\sPhi_{t, 0}(\sK) x_0 + \nlsum_{\tau = 0}^{t - 1} \sPhi_{t, 0}(\sK) w_{\tau} \Big)\Big\|_2^{1/2} = \bigo(t^{1/2}).
    \end{align*}
    Note that here we leverage independence instead of applying  Lemma~\ref{lem:cov-sum}, which yields a worse bound of $\bigo(t)$.
    Then, since $\varphi = \bigo((\kappa a_2)^{-T}) = \bigo(t^{-1})$, we have 
    \begin{align*}
        \|\Xi_t(\sK)\|_2^{1/2} =\;& \bigo(t^{1/2} + \varphi \nlsum_{t_1=0}^{t-1} \max_{\tau \le t_1}\|\Xi_{\tau}(\sK)\|_2^{1/2}) \\
        =\;& \bigo(t^{1/2} + \varphi t^{3/2}) \\
        =\;& \bigo(t^{1/2}).
    \end{align*}
    Hence, for all $t\ge 0$, $\|\Xi_t(\sK)\|_2 = \bigo(t)$.
    Then, by~\eqref{eqn:xi-sigma}, $\Xi_t(\sK) \le \const_t \Sigma_t$ for $\const_t = \bigo(\|\Xi_t(\sK)\|_2) = \bigo(t)$.

    Next, we bound $\ePhi_{t, \tau}(\sK)$ in order to bound~\eqref{eqn:cov-delta} and~\eqref{eqn:cov-x-tau-tau}.
    Since $\|\eA_t\|_2 \le c$ and $\|\eB_t\|_2, \|\sK_t\|_2$ are order $\bigo(1)$, there exists a dimension-free constant $a_3 > 1$ depending polynomially on $c$ and other problem parameters, such that $\|\ePhi_{t, \tau}(\sK)\|_2 = \bigo(a_3^{t - \tau})$.
    Then,~\eqref{eqn:cov-delta} is bounded by 
    \begin{align*}
        &\bigo\big(a_3^{t-\tau} \eps t^{1/2} + a_3^{t-\tau} \eps \cdot \varphi {\tau}^{1/2} + \varphi \nlsum_{t_0=\tau}^{t-1} a_3^{t - t_0} \cdot t_0^{1/2} \big)\big) \\
        =\;& \bigo(\eps t^{1/2} a_3^{t-\tau} + \varphi t^{1/2} a_3^{t - \tau}) \\
        =\;& \bigo((\eps + \varphi) t^{1/2} a_3^{t - \tau}),
    \end{align*}
    and~\eqref{eqn:cov-x-tau-tau} is bounded by 
    \begin{align*}
        \bigo\big(a_3^{t-\tau} \cdot t^{1/2} + \nlsum_{t_0=\tau}^{t-1} (a_3^{t-t_0} + a_3^{t - t_0} \varphi t_0^{1/2}) \big)
        = \bigo(t^{1/2} a_3^{t-\tau}).
    \end{align*}
    Hence, by~\eqref{eqn:os-cov-diff}, 
    \begin{align*}
        \|\Cov(x_t(\tau + 1, 0, x_0, \sK)) - \Cov(x_t(\tau, 0, x_0, \sK))\|_2
        = \bigo((\eps + \varphi) t a_3^{2(t-\tau)}).
    \end{align*}
    Using inverse telescoping, we have 
    \begin{align*}
        \Xi_t(\sK) - \eXi_t(\sK)
        =\;& \Cov(x_t(t, 0, x_0, \sK)) - \Cov(x_t(0, 0, x_0, \sK)) \\
        =\;& \nlsum_{\tau=0}^{t-1} (\Cov(x_t(\tau + 1, 0, x_0, \sK)) - \Cov(x_t(\tau, 0, x_0, \sK))).
    \end{align*}
    Then, 
    \begin{align*}
        \|\Xi_t(\sK) - \eXi_t(\sK)\|_2
        = \bigo\Big( (\eps + \varphi) t \nlsum_{\tau = 0}^{t-1} a_3^{2(t-\tau)} \Big)
        = \bigo( (\eps + \varphi) t a_3^{2t} ).
    \end{align*}
    Finally, since $\eps + \varphi = \bigo(a_3^{-2T})$, we have 
    \begin{align*}
        \|\eXi_t(\sK)\|_2 \le \|\Xi_t(\sK)\|_2 + \|\Xi_t(\sK) - \eXi_t(\sK)\|_2 = \bigo(t),
    \end{align*}
    completing the proof. 
    \endproof

\subsubsection{Proof of Claim~\ref{clm:exi-0}}
\label{sec:pf-exi-0}
    Define
    \begin{align*}
        x_t(\tau, t_0, x) := \ePhi_{t, t_0}(\eK) x + \nlsum_{t_1 = t_0}^{t-1} \ePhi_{t, t_1+1}(\eK) w_{t_1}
        + \nlsum_{t_1 = t_0}^{\tau-1} \ePhi_{t, t_1+1}(\eK) \eB_{t_1} \eK_{t_1} \delta_{t_1},
    \end{align*}
    where $\tau \le t$ denotes the time step at which the controller switches from $u_t = \eK_t (x_t + \delta_t)$ to $u_t = \eK_t x_t$, $t_0 \le \tau$ denotes the initial time step, and $x$ is the state at  $t_0$. 
    Then, we know  
    \begin{align*}
        \eXi_t(\eK) = \Cov(x_t(t, 0, x_0)), \quad 
        \eXi^{0}_t(\eK) = \Cov(x_t(0, 0, x_0)).
    \end{align*}
    By definition, for $t_2 \le t_1 < t$,
    \begin{align*}
        x_t(t_1, 0, x_0) = x_t(t_1, t_2, x_{t_2}(t_2, 0, x_0)).
    \end{align*}
    Hence, for $0\le \tau < t$, 
    \begin{align*}
        &\Cov(x_t(\tau + 1, 0, x_0)) - \Cov(x_t(\tau, 0, x_0)) \\
        =\;& \Cov(x_t(\tau + 1, \tau, x_{\tau}(\tau, 0, x_0)))   - \Cov(x_t(\tau, \tau, x_{\tau}(\tau, 0, x_0))).
    \end{align*}
    As a shorthand, define
    \begin{align*}
        \delta_{t, \tau}^{\prime} :=
        &x_t(\tau + 1, \tau, x_{\tau}(\tau, 0, x_0)) - x_t(\tau, \tau, x_{\tau}(\tau, 0, x_0)) \\
        =\;& \ePhi_{t, \tau+1}(\eK) ((\eA_{\tau} + \eB_{\tau} \eK_{\tau}) x_{\tau}(\tau, 0, x_0) + \eB_{\tau} \eK_{\tau} \delta_{\tau}) \\
        &\quad + \nlsum_{t_0 = \tau}^{t-1} \ePhi_{t, t_0+1}(\eK) w_{t_0} - \ePhi_{t, \tau}(\eK) x_{\tau}(\tau, 0, x_0)
        - \nlsum_{t_0 = \tau}^{t-1} \ePhi_{t, t_0+1}(\eK) w_{t_0} \\
        =\;& \ePhi_{t, \tau+1}(\eK) \eB_{\tau} \eK_{\tau} \delta_{\tau},
    \end{align*}
    where we couple the $x_{\tau}(\tau, 0, x_0)$ and $(w_{t_0})_{t_0\ge \tau}$ in the two trajectories generating $x_t(\tau + 1, \tau, x_{\tau}(\tau, 0, x_0))$ and $x_t(\tau, \tau, x_{\tau}(\tau, 0, x_0))$. 
    Then, by Lemma~\ref{lem:cov-diff},
    \begin{align*}
        &\|\Cov(x_t(\tau + 1, \tau, x_{\tau}(\tau, 0, x_0)))
        - \Cov(x_t(\tau, \tau, x_{\tau}(\tau, 0, x_0)))\|_2 \\
        \le\;& \|\Cov(\delta_{t, \tau}^{\prime})\|_2
        + 2 \|\Cov(\delta_{t, \tau}^{\prime})\|_2^{1/2} \cdot \|\Cov(x_t(\tau, \tau, x_{\tau}(\tau, 0, x_0)))\|_2^{1/2}.
    \end{align*}
    Since $\|\Xi_t(\eK)\|_2 = \bigo(a_2^T)$ by Claim~\ref{clm:xi-exi}, 
    \begin{align*}
        \|\Cov(\delta_{t, \tau}^{\prime})\|_2^{1/2}
        =\;& \|\ePhi_{t, \tau+1}(\eK) \eB_{\tau} \eK_{\tau} \Cov(\delta_{\tau})^{1/2} \|_2 \nonumber \\
        =\;& \bigo\big((a_2^{1/2})^{T-\tau} \cdot 1 \cdot \kappa^{T-\tau} \cdot \varphi \|\Xi_t(\eK)\|_2^{1/2}\big) \nonumber \\
        =\;& \bigo(\varphi a_2^{T/2} (\kappa a_2^{1/2})^{T-\tau}).
    \end{align*}
    By Lemma~\ref{lem:cov-sum}, 
    \begin{align*}
        \|\Cov(x_t(\tau, \tau, x_{\tau}(\tau, 0, x_0)))\|_2^{1/2}
        =\;& \Big\|\Cov\Big(\ePhi_{t, \tau}(\eK) x_{\tau}(\tau, 0, x_0)  + \nlsum_{t_0 = \tau}^{t-1} \ePhi_{t, t_0+1}(\eK) w_{t_0} \Big)\Big\|_2^{1/2} \nonumber \\
        \le\;& \|\ePhi_{t, \tau}(\eK) \eXi_{\tau}^{1/2}(\eK)\|_2 + \nlsum_{t_0 = \tau}^{t-1} \| \ePhi_{t, t_0 + 1}(\eK) \Sigma_{w_{t_0}}^{1/2} \|_2 \nonumber \\
        \overset{(i)}{=}\;& \bigo( a_2^{(T-\tau)/2} a_2^{T/2} + \nlsum_{t_0=\tau}^{t-1} a_2^{(T-t_0)/2} ) \nonumber \\
        =\;& \bigo(a_2^{T/2} a_2^{(T - \tau)/2}),
    \end{align*}
    where $(i)$ follows from $\|\eXi(\eK)\|_2 = \bigo(a_2^T)$ by Claim~\ref{clm:xi-exi}.
    Hence, by our assumption that $\varphi = \bigo((\kappa a_2)^{-T}) = \bigo(\kappa^{-T})$, 
    \begin{align*}
        \|\Cov(x_t(\tau + 1, 0, x_0)) - \Cov(x_t(\tau, 0, x_0))\|_2
        = \bigo(\varphi a_2^T (\kappa a_2)^{T-\tau}).
    \end{align*}
    Finally, using inverse telescoping, we have 
    \begin{align*}
        \|\eXi_t(\eK) - \eXi_t^{0}(\eK)\|_2
        =\;& \|\Cov(x_t(t, 0, x_0)) - \Cov(x_t(0, 0, x_0))\|_2 \\
        \le\;& \nlsum_{\tau = 0}^{t-1} \|\Cov(x_t(\tau + 1, 0, x_0)) - \Cov(x_t(\tau, 0, x_0))\|_2 \\
        =\;& \bigo( \varphi a_2^{T} \nlsum_{\tau = 0}^{t-1} (\kappa a_2)^{T-\tau} ) \\
        =\;& \bigo(\varphi a_2^{T} (\kappa a_2)^{T-t}),
    \end{align*}
    which completes the proof.
\endproof

If the input of linear regression has full-rank covariance, then by Lemma~\ref{lem:lr}, the system parameters can be fully identified. To certify the performance of the certainty equivalent optimal controller, we need the performance difference lemma~\citep{kakade2002approximately}. To specify it for our LQR problem, we define the value function $V_t^K(x)$ as the expected cumulative cost starting from state $x$ at time step $t$ under control $u_t = K_t x_t$ for $t \ge 0$.
Let $P_T^K = Q_T$, and for $t\ge 0$, define value matrices $P_t^K$ through the backward recursion
\begin{align}  \label{eqn:p-def}
    P_t^K := \sQ_t + K_t^{\top} \sR_t K_t + (\sA_t + \sB_t K_t)^{\top} P_{t+1}^K (\sA_t + \sB_t K_t).
\end{align}
Then, it is easy to verify that $V_t^K(x)$ satisfies
\begin{align}
    V_t^K(x) := \E\Big[\nlsum_{\tau = t}^{T} c_t \Big]
    = x^{\top} P_t^K x + \nlsum_{\tau = t}^{T-1} \E[w_{\tau}^{\top} P_{\tau+1}^K w_{\tau}].  \label{eqn:vtk}
\end{align}
The following lemma is the counterpart of~\citep[Lemma 10]{fazel2018global} in the LTV setting with system noises.

\begin{lemma}[Performance difference lemma]  \label{lem:pdl}
    Consider the finite-horizon linear quadratic control problem $((\sA_t, \sB_t, \sR_t)_{t=0}^{T-1}, (\sQ_t)_{t=0}^{T})$.
    Let $V_t^{u}(x)$ be the expected cumulative cost starting from state $x$ at time step $t$ under control $(u_t)_{t\ge 0}$.
    Define $E_t^{K} := (\sR_t + (\sB_t)^\top  P_{t+1}^{K} \sB_t) K_t + (\sB_t)^{\top} P_{t+1}^{K} \sA_t$ for $t \ge 0$. Then, for all $t \ge 0$ and $x \in \R^{d_x}$,
    \begin{align*}
        V_t^{u}(x) - V_t^{K}(x) = \nlsum_{\tau=t}^{T-1} \E[D_{\tau}^{K}(x_{\tau}^u, u_t)],
    \end{align*}
    where $x_{\tau}^u$ denotes the state at time step $\tau \ge t$ under control $(u_t)_{t\ge 0}$ starting from $x_t^u = x$, and for $0\le t \le T-1$, 
    \begin{align*}
        D_t^K(x, u) := 2 (u - K_t x)^{\top} E_t^{K} x
        + (u - K_t x)^{\top} (\sR_t + (\sB_t)^{\top} P_{t+1}^{K} \sB_t) (u - K_t x). 
    \end{align*}
\end{lemma}
\begin{proof}
    For $\tau \ge t$, let $c_{\tau}^u$ denote the cost at time step $t$ under control $(u_t)_{t\ge 0}$ starting from $x_t^u = x$.
    Then, by the definition of $V_t^{u}(x)$, we have
    \begin{align*}
        V_t^u(x) - V_t^K(x)
        =\;& \E\Big[\nlsum_{\tau=t}^{T} c_{\tau}(x_{\tau}^u, u_{\tau})\Big] - V_t^K(x) \\ 
        =\;& \E\Big[\nlsum_{\tau=t}^{T} \Big(c_{\tau}(x_{\tau}^u, u_{\tau}) + V_{\tau}^{K}(x_{\tau}^u) - V_{\tau}^{K}(x_{\tau}^u)\Big)\Big] - V_t^K(x) \\
        \overset{(i)}{=}\;& \E\Big[\nlsum_{\tau=t}^{T-1} \Big(c_{\tau}(x_{\tau}^u, u_{\tau}) + V_{\tau + 1}^{K}(x_{\tau + 1}^u) - V_{\tau}^{K}(x_{\tau}^u)\Big)\Big],
    \end{align*}
    where $(i)$ is due to $x = x_t^u$ and $c_T^u = \|x_T^u\|_{\sQ_T}^2 = V_{T}^K(x_T^u)$.
    Define 
    \begin{align*}
        D_t^K(x, u) := c_t(x, u) + \E[V_{t+1}^K(\sA_t x + \sB_t u + w_t)] - V_t^K(x).
    \end{align*}
    Then, by the law of total expectation, $V_t^{u}(x) - V_t^{K}(x) = \nlsum_{\tau=t}^{T-1} \E[D_{\tau}^{K}(x_{\tau}^u, u_t)]$. 
    Notice that 
    \begin{align*}
        V_t^K(x) = x^{\top} (\sQ_t + K_t^{\top} \sR_t K_t) x + \nlsum_{\tau=t}^{T-1}\E[w_{\tau}^{\top} P_{\tau+1}^K w_{\tau}]
        + x^{\top} (\sA_t + \sB_t K_t)^{\top} P_{t+1}^{K} (\sA_t + \sB_t K_t) x.
    \end{align*}
    Hence, we have 
    \begin{align*}
        D_t^K(x, u)
        =\;& x^{\top} \sQ_t x + u^{\top} \sR_t u + \nlsum_{\tau=t}^{T-1}\E[w_{\tau}^{\top} P_{\tau+1}^K w_{\tau}]
        + (\sA_t x + \sB_t u)^{\top} P_{t+1}^K (\sA_t x + \sB_t u) - V_t^K(x) \\
        =\;& 2 (u - K_t x)^{\top} E_t^K x
        + (u - K_t x)^{\top} (\sR_t + (\sB_t)^{\top} P_{t+1}^{K} \sB_t) (u - K_t x). 
    \end{align*}
\end{proof}

The following lemma shows that the certainty equivalent optimal controller for LQ  control in the full-rank case has a much better guarantee compared to the rank-deficient case.

\begin{lemma}[Linear quadratic control]  \label{lem:lq}
    Consider the finite-horizon time-varying linear dynamical system given by the first equation in~\eqref{eqn:po-ltv} with quadratic costs given by~\eqref{eqn:q-cost}, under Assumption~\ref{asp:cost-obs} with $\ell = 1$ and Assumption~\ref{asp:bounded} on the  relevant problem parameters.
    Let $(\sK_t)_{t=0}^{T-1}$ be the optimal feedback gains and $(\sP_t)_{t=0}^{T}$ be the solution to the RDE~\eqref{eqn:p-riccati} of system $((\sA_t, \sB_t, \sR_t)_{t=0}^{T-1}, (\sQ_t)_{t=0}^{T})$.
    Let $(\eK_t)_{t=0}^{T-1}$ be the optimal feedback gains and $(\eP_t)_{t=0}^{T}$ be the solution to the RDE~\eqref{eqn:p-riccati} of system $((\eA_t, \eB_t, \sR_t)_{t=0}^{T-1}, (\eQ_t)_{t=0}^{T})$, where $\|\eA_t - \sA_t\|_2 \le \eps$, $\|\eB_t - \sB_t\|_2 \le \eps$, $\|\eQ_t - \sQ_t\|_2 \le \eps$, and $(\eQ_t)_{t=0}^{T}$ are positive semidefinite.
    For any feedback gains $(K_t)_{t=0}^{T-1}$, let $J(\eK)$ denote the expected cumulative cost in system $((\sA_t, \sB_t, \sR_t)_{t=0}^{T-1}, (\sQ_t)_{t=0}^{T})$ under the state feedback control $u_t = K_t x_t$ for $t \ge 0$.
    Then, there exists a dimension-free constant $\eps_0 > 0$ with $\eps_0^{-1}$ depending polynomially on problem parameters, such that as long as $\eps \le \eps_0$, $\|\eP_t - \sP_t\|_2 = \bigo(\eps)$, $\|\eK_t - \sK_t\|_2 = \bigo(\eps)$ for all $t \ge 0$, and 
    \vspace*{-2pt}
    \begin{align*}
        J(\eK_t) - J(\sK_t) = \bigo((d_x \wedge d_u) T \eps^2).
    \end{align*}    
\end{lemma}

\vspace*{2pt}
\begin{proof}
    This problem is studied in~\citep{mania2019certainty} for the infinite-horizon LTI setting; here we extend their result to the finite-horizon LTV setting.

    We adopt the following compact formulation of a finite-horizon LTV system, as introduced in~\citep[Section 3]{zhang2021derivative}, to reduce our setting to the infinite-horizon LTI one:
    \begin{gather*}
        x = [x_0; \ldots; x_T], \quad
        u = [u_0; \ldots; u_{T-1}], \quad 
        w = [x_0; w_0; \ldots; w_{T-1}], \\
        \sA = \begin{bmatrix}
            0_{d_x \times d_x T} & 0_{d_x \times d_x} \\
            \diag(\sA_0, \ldots, \sA_{T-1}) & 0_{d_x T \times d_x}
        \end{bmatrix}, \quad 
        \sB = \begin{bmatrix}
            0_{d_x \times d_u T} \\
            \diag(\sB_0, \ldots, \sB_{T-1})
        \end{bmatrix}, \quad 
        \sQ = \diag(\sQ_0, \ldots, \sQ_T), \\ 
        \sR = \diag(\sR_0, \ldots, \sR_{T-1}), \quad 
        K = [\diag(K_0, \ldots, K_{T-1}), 0_{d_u T \times d_x}].
    \end{gather*}
    The control inputs using state feedback controller $(K_t)_{t=0}^{T-1}$ can be characterized by $u = K x$. Let $(P_t^K)_{t=0}^{T}$ be the associated value matrix starting from step $t$. Then 
    \begin{align*}
        P_t^K = (\sA_t + \sB_t K_t)^{\top} P_{t+1}^K (\sA_t + \sB_t K_t) + \sQ_t + K_t^{\top} \sR_t K_t,
    \end{align*}
    and $P^K := \diag(P_0^K, \ldots, P_T^K)$ is the solution to 
    \begin{align*}
        P^K = (\sA + \sB K)^{\top} P^K (\sA + \sB K) + \sQ + K^{\top} \sR K.
    \end{align*} 
    Similarly, the optimal value matrix 
    \begin{align*}
        \sP := \diag(\sP_0, \ldots, \sP_T),
    \end{align*}
    produced by the RDE~\eqref{eqn:p-riccati} in system $(\sA, \sB, \sR,\sQ)$,  
    satisfies 
    \begin{align*}
        \sP = (\sA)^{\top} (\sP - \sP \sB ((\sB)^{\top} \sP \sB
        + \sR)^{-1} (\sB)^{\top} \sP) \sA + \sQ.
    \end{align*}
    Let $\eP = \diag(\eP_0, \ldots, \eP_T)$ be the optimal cumulative cost matrices in system $(\eA, \eB, \eQ, \sR)$ 
    by the RDE~\eqref{eqn:p-riccati}.
    With a slight abuse of notation, define $\sK := [\diag(\sK_0, \ldots, \sK_{T-1}), 0_{d_u T \times d_x}]$, where $(\sK_t)_{t=0}^{T-1}$ is the optimal controller in system $(\sA, \sB, \sQ, \sR)$, and define $\eK$ similarly for system $(\eA, \eB, \eQ, \sR)$.
    By definition, $\sA$ is stable (hence, stabilizable), 
    $(\sA, \sQ)$ is observable in the sense of LTI systems, and $\sR \succ 0$.
    Therefore, by~\citep[Propositions 1 and 2]{mania2019certainty}, there exists a dimension-free constant $\eps_0 > 0$ with $\eps_0^{-1}$ depending polynomially on problem parameters such that as long as $\eps \le \eps_0$,
    \begin{align*}
        \|\eP - \sP\|_2 = \bigo(\eps), \quad \|\eK - \sK\|_2 = \bigo(\eps),
    \end{align*}
    and that $\eK$ stabilizes system $(\sA, \sB)$.
    By Lemma~\ref{lem:pdl},
    \begin{align*}
        J(\eK) - J(\sK)
        \overset{(i)}{=}\;& \nlsum_{t=0}^{T-1} \tr( \Sigma_t
        (\eK_t - \sK_t)^{\top} (\sR_t + (\sB_t)^{\top} \sP_{t+1} \sB_t)(\eK_t - \sK_t)) \\
        \overset{(ii)}{=}\;& \tr(\Sigma (\eK - \sK)^{\top} (\sR + (\sB)^{\top} \sP \sB) (\eK - \sK)),
    \end{align*}
    where in $(i)$ the term $E_t^{K} = 0$ in Lemma~\ref{lem:pdl} for $K = \sK$, and in $(ii)$ we define $\Sigma := \diag(\Sigma_0, \ldots, \Sigma_T)$ with $\Sigma_t$ being $\E[x_t x_t^{\top}]$ in system $(\sA, \sB)$ under state feedback controller $\eK$. 
    As a result,
    \begin{align*}
        J(\eK) - J(\sK)
        \le \|\Sigma\|_2 \|\sR + (\sB)^{\top} \sP \sB\|_2 \| \eK - \sK\|_F^2.
    \end{align*}
    Since $\eK$ stabilizes system $(\sA, \sB)$, $\|\Sigma\|_2 = \bigo(1)$.
    Since 
    \begin{align*}
        \|\eK - \sK\|_F \le ((d_x \wedge d_u) T )^{1/2} \|\eK - \sK\|_2, 
    \end{align*}
    we conclude that 
    \begin{align*}
        J(\eK) - J(\sK) = \bigo((d_x \wedge d_u) T \eps^2),
    \end{align*}
    and thus complete the proof. 
\end{proof}

Finally, we are ready to present the main lemma of this section below, which combines Lemmas~\ref{lem:rank-deficient} to~\ref{lem:lq} to provide an end-to-end guarantee for the certainty equivalent optimal controller in linear quadratic control with both rank-deficient and full-rank stages.

\begin{lemma}  \label{lem:e2e}
    Consider the finite-horizon linear time-varying system given by the first equation in~\eqref{eqn:po-ltv} with quadratic costs given by~\eqref{eqn:q-cost}, under Assumption~\ref{asp:stability}, Assumption~\ref{asp:cost-obs} and Assumption~\ref{asp:bounded} on the relevant problem parameters.
    Let $\Sigma_t := \Cov(x_t)$ under control $u_t \sim \gauss(0, \sigma_u^2 I)$ for all $t\ge 0$, which may not have full rank.
    Let $(\eA_t, \eB_t)_{t=0}^{T-1}, (\eQ_t)_{t=0}^{T}$ satisfy $\|(\eA_t - \sA_t)\Sigma_t\|_2 \le \eps_1$, $\|\eB_t - \sB_t\|_2 \le \eps_1$, $\|\eQ_t - \sQ_t\|_2 \le \eps_1$, $\|\eA_t\|_2 \le c$ for \mbox{dimension-free} constant $c > 0$, $\eQ_t \succ 0$ for $0\le t \le \ell - 1$, and $\|\eA_t - \sA_t\|_2 \le \eps_2$, $\|\eB_t - \sB_t\|_2 \le \eps_2$, $\|\eQ_t - \sQ_t\|_2 \le \eps_2$, $\eQ_t \succcurlyeq 0$ for $t\ge \ell$, where $\eps_1, \eps_2 \ge 0$.
    Let $(\sK_t)_{t=0}^{T-1}$ and $(\eK_t)_{t=0}^{T-1}$ be the optimal feedback gains of system $((\sA_t, \sB_t, \sR_t)_{t=0}^{T-1}, (\sQ_t)_{t=0}^{T})$ and system $((\eA_t, \eB_t, \sR_t)_{t=0}^{T-1}, (\eQ_t)_{t=0}^{T})$, respectively. 

    For any feedback gains $(K_t)_{t=0}^{T-1}$, let $J(K)$ denote the expected cumulative cost in system $((\sA_t, \sB_t, \sR_t)_{t=0}^{T-1}$, $(\sQ_t)_{t=0}^{T})$ under state feedback control $u_t = K_t x_t$ for $t \ge 0$, and $J^{\delta}(K)$ denote the expected cumulative cost in system $((\sA_t, \sB_t, \sR_t)_{t=0}^{T-1}, (\sQ_t)_{t=0}^{T})$ under control $u_t = K_t (x_t + \delta_t)$, where $\delta_t$ is a perturbation, for $t \ge 0$.
    Under control $u_t = K_t (x_t + \delta_t)$ for all $t \ge 0$, let $\Xi_t(K) := \Cov(x_t)$ in the  system $((\sA_t, \sB_t, \sR_t)_{t=0}^{T-1}, (\sQ_t)_{t=0}^{T})$, and assume that $\delta_t$ follows an arbitrary zero-mean Gaussian distribution, can be correlated with $x_t$, and satisfies $\|\Cov(\delta_t)\|_2^{1/2} \le \varphi_1 \max_{\tau \le t}\|\Xi_{\tau}(K)\|_2^{1/2}$ for $0\le t \le \ell - 1$ and  $\|\Cov(\delta_t)\|_2^{1/2} \le \varphi_2 \max_{\tau \le t}\|\Xi_{\tau}(K)\|_2^{1/2}$ for $t \ge \ell$.

    There exist dimension-free constants $\kappa, a_2, a_3, a_4 > 1$ that depend polynomially on $c$ and other problem parameters and dimension-free constant $\eps_0 > 0$ with $\eps_0^{-1}$ depending polynomially on problem parameters, such that under the conditions that 1) $\eps_1, \varphi_1$ are small enough to ensure $\|\eB_t\|_2, \|\eQ_t\|_2$ are of order $\bigo(1)$, $\lambda_{\min}(\eQ_t) = \Omega(1)$ for $0\le t \le \ell - 1$ and $\eps_1 + \varphi_1 = \bigo((\kappa a_2 + a_3^2)^{-T})$, and that 2) $\eps_2, \varphi_2$ are small enough to ensure $\eps_2 \le \eps_0$ and $\eps_2 + \varphi_2 = \bigo(\eps_1 + \varphi_1) = \bigo({T}^{-1})$, we have
    \begin{align*}
        J^{\delta}(\eK) - J(\sK) 
        = \bigo((\eps_1 + \varphi_1) d_x a_4^{\ell} + \eps_2^2 d_x (T - \ell)^2).
    \end{align*}
    \normalsize
\end{lemma}

\vspace*{2pt}

\begin{proof}
    Let us begin by introducing some notation.
    For feedback gains $(K_t)_{t=0}^{T-1}$ and $d_x \times d_x$ matrix $P_{\ell} \succ 0$, let $J_1^{\delta}(K, P_{\ell})$ denote the expected cumulative cost under $u_t = K_t (x_t + \delta_t)$ for $t \ge 0$, in system $((\sA_t, \sB_t, \sR_t)_{t=0}^{T-1}, (\sQ_t)_{t=0}^{T})$ for $0\le t \le \ell$, with $c_{\ell} = \|x_{\ell}\|_{P_{\ell}}^2$; let $J_1(K, P_{\ell})$ denote the corresponding expected cumulative cost under $u_t = K_t x_t$ for $t \ge 0$.
    Let $x_t^{\delta}$ denote the state at step $t$ in system $((\sA_t, \sB_t, \sR_t)_{t=0}^{T-1}, (\sQ_t)_{t=0}^{T})$ under control $u_t = \eK_t (x_t + \delta_t)$.
    For feedback gains $(K_t)_{t=0}^{T-1}$, let $J_2^{\delta}(K)$ denote the expected cumulative cost under $u_t = K_t (x_t + \delta_t)$ for $t \ge \ell$, in system $((\sA_t, \sB_t, \sR_t)_{t=0}^{T-1}, (\sQ_t)_{t=0}^{T})$ starting from $x_{\ell}^{\delta}$; let $J_2(K)$ denote the corresponding expected cumulative cost under $u_t = K_t x_t$ for $t \ge \ell$.

    Let $(P_{t}^{\eK})_{t=0}^{T}$ denote the value matrices in system $((\sA_t, \sB_t, \sR_t)_{t=0}^{T-1}, (\sQ_t)_{t=0}^{T})$ under control $u_t = \eK_t x_t$ for $t \ge 0$, given by recursion~\eqref{eqn:p-def}.
    By definition, we have
    \begin{align}  \label{eqn:j12}
        J^{\delta}(\eK)
        = J_{1}^{\delta}(\eK, P_{\ell}^{\eK}) - \E\big[\|x_{\ell}^{\delta}\|_{P_{\ell}^{\eK}}^2\big] + J_2^{\delta}(\eK).
    \end{align}
    By the property of the value function given in~\eqref{eqn:vtk},
    \begin{align*}
        J_2(\eK) = \E\big[\|x_{\ell}^{\delta}\|_{P_{\ell}^{\eK}}^2\big] + \nlsum_{t=\ell}^{T-1} \E[w_t^{\top} P_{t+1}^{\eK} w_t].
    \end{align*}
    Then, by substituting $\E\big[\|x_{\ell}^{\delta}\|_{P_{\ell}^{\eK}}^2\big]$ from the above equation into~\eqref{eqn:j12}, we have
    \begin{align*}
        J^{\delta}(\eK)
        = J_{1}^{\delta}(\eK, P_{\ell}^{\eK}) + \nlsum_{t=\ell}^{T-1} \E[w_t^{\top} P_{t+1}^{\eK} w_t]
        + J_2^{\delta}(\eK) - J_2(\eK).
    \end{align*}
    On the other hand, let $(\sP_t)_{t=0}^{T}$ and $(\eP_t)_{t=0}^{T}$ denote the optimal value matrices in system $((\sA_t, \sB_t, \sR_t)_{t=0}^{T-1}$, $(\sQ_t)_{t=0}^{T})$ and system $((\eA_t, \eB_t, \sR_t)_{t=0}^{T-1}, (\eQ_t)_{t=0}^{T})$, respectively, given by RDE~\eqref{eqn:p-riccati}.
    Similarly, we have
    \begin{align*}
        J(\sK) = J_{1}(\sK, \sP_{\ell}) + \nlsum_{t=\ell}^{T-1} \E[w_t^{\top} \sP_{t+1} w_t].
    \end{align*}
    Therefore,
    \begin{align*}
        & J^{\delta}(\eK) - J(\sK) \\
        =\;& J_{1}^{\delta}(\eK, P_{\ell}^{\eK}) + \nlsum_{t=\ell}^{T-1} \E[w_t^{\top} P_{t+1}^{\eK} w_t] + J_2^{\delta}(\eK) - J_2(\eK)
        - J_{1}(\sK, \sP_{\ell}) - \nlsum_{t=\ell}^{T-1} \E[w_t^{\top} \sP_{t+1} w_t] \\
        =\;& J_{1}^{\delta}(\eK, P_{\ell}^{\eK}) - J_1(\sK, \sP_{\ell}) + J_2^{\delta}(\eK) - J_2(\eK)
        + \nlsum_{t=\ell}^{T-1} \ipc{\Cov(w_t)}{P_{t+1}^{\eK} - \sP_{t+1}}_F \\
        =\;& \underbrace{J_1^{\delta}(\eK, \eP_{\ell}) - J_1(\sK, \sP_{\ell})}_{(a)} + \underbrace{J_{1}^{\delta}(\eK, P_{\ell}^{\eK}) - J_1^{\delta}(\eK, \eP_{\ell})}_{(b)}.
    \end{align*}
    Let $\eps_0$ be the one required in Lemma~\ref{lem:lq}. Then, applying Lemma~\ref{lem:lq} to the last $T - \ell$ steps, we have $\|\eP_{\ell} - \sP_{\ell}\|_2 = \bigo(\eps_2) = \bigo(\eps_1)$. 
    Hence, for the first $\ell$ steps, by Lemma~\ref{lem:rank-deficient}, there exists a dimension-free constant $a_4 > 1$ that depends polynomially on problem parameters, such that  
    \begin{align*}
        (a) = J_1^{\delta}(\eK, \eP_{\ell}) - J_1(\sK, \sP_{\ell}) = \bigo((\eps_1 + \varphi_1) d_x a_4^{\ell} ).
    \end{align*}
    By definition, 
    \begin{align*}
        (b) = J_{1}^{\delta}(\eK, P_{\ell}^{\eK}) - J_1^{\delta}(\eK, \eP_{\ell}) 
        = \E\Big[ \|x_{\ell}^{\delta}\|_{P_{\ell}^{\eK}}^2 - \|x_{\ell}^{\delta}\|_{\eP_{\ell}}^2 \Big] 
        = \E\Big[ \|x_{\ell}^{\delta}\|_{P_{\ell}^{\eK}}^2 - \|x_{\ell}^{\delta}\|_{\sP_{\ell}}^2\Big] + \E\Big[ \|x_{\ell}^{\delta}\|_{\sP_{\ell}}^2 - \|x_{\ell}^{\delta}\|_{\eP_{\ell}}^2\Big].
    \end{align*}
    Given feedback gains $(K_t)_{t=0}^{T-1}$, for $t > \tau$, define $\sPhi_{t, \tau}(K) := (\sA_{t-1} + \sB_{t-1} K_{t-1}) \cdots (\sA_{\tau} + \sB_{\tau} K_{\tau})$ and define $\sPhi_{\tau, \tau}(K) := I$. By the arguments for deriving~\eqref{eqn:phi-o1} in~\S\ref{sec:pf-xi-exi} for the case of $K = \sK$, we have that $\|\sPhi_{t, \ell}(\sK)\|_2 = \bigo(1)$ for all $t \ge \ell$. 
    Applying Lemma~\ref{lem:lq} to the last $T - \ell$ steps, we have $\|\eK_{t} - \sK_{t}\|_2 = \bigo(\eps_2)$ for $t \ge \ell$.
    Hence, using $\eps_2 = \bigo(T^{-1})$ and the expression that 
    \begin{align*}
        \|\sPhi_{t, \ell}(\eK)\|_2
        \le \big\|\sPhi_{t, \ell}(\sK) + \nlsum_{\tau=\ell}^{t-1} \sPhi_{t, \tau + 1}(\sK) \sB_{\tau} (\eK_{\tau} - \sK_{\tau}) \sPhi_{\tau, \ell}(\eK)\big\|_2,
    \end{align*}
    we can show by induction that $\|\sPhi_{t, \ell}(\eK)\|_2 = \bigo(1)$.
    By the performance difference lemma (Lemma~\ref{lem:pdl}), for any $x \in \R^{d_x}$, 
    \begin{align}
        \|x\|_{P_{\ell}^{\eK}}^2 - \|x\|_{\sP_{\ell}}^2 
        =\;& \nlsum_{t=\ell}^{T-1} x^{\top} (\sPhi_{t, \ell}(\eK))^{\top} (\eK_t - \sK_t)^{\top}
        (\sR_t + (\sB_t)^{\top} \sP_{t+1} \sB_t) (\eK_t - \sK_t) \sPhi_{t, \ell}(\eK) x, \label{eqn:pek-sp}
    \end{align}
    where the term $E_t^K = 0$ in Lemma~\ref{lem:pdl} for $K = \sK$.
    Hence,
    \begin{align*}
        &\|x_{\ell}^{\delta}\|_{P_{\ell}^{\eK}}^2 - \|x_{\ell}^{\delta}\|_{\sP_{\ell}}^2 \\
        \le\;& d_x \|\Cov(x_{\ell}^{\delta})\|_2 \cdot \Big\|\nlsum_{t=\ell}^{T-1}  (\sPhi_{t, \ell}(\eK))^{\top}
        (\eK_t - \sK_t)^{\top} (\sR_t + (\sB_t)^{\top} \sP_{t+1} \sB_t) (\eK_t - \sK_t) \sPhi_{t, \ell}(\eK) \Big\|_2 \\
        =\;& \bigo(\eps_2^2 d_x a_2^{\ell} (T - \ell)),
    \end{align*}
    where we note that $(\sR_t, \sB_t, \sP_{t+1})_{t\ge \ell}$ have $\bigo(1)$ operator norms.
    On the other hand, $\|\sP_{\ell} - \eP_{\ell}\|_2 = \bigo(\eps_2)$, following which we have 
    \begin{align*}
        \E\Big[ \|x_{\ell}^{\delta}\|_{\sP_{\ell}}^2 - \|x_{\ell}^{\delta}\|_{\eP_{\ell}}^2\Big] = \bigo(\eps_2 d_x a_2^{\ell}).
    \end{align*}
    Since $\eps_2 = \bigo(T^{-1})$, combining the above two bounds yields 
    \begin{align*}
        (b) = \bigo(\eps_2^2 d_x a_2^{\ell} (T - \ell) + \eps_2 d_x a_2^{\ell})
        = \bigo(\eps_2 d_x a_2^{\ell}).
    \end{align*}
    Note that we can also bound $(b)$ using the simulation lemma~\citep{kearns2002near}, which yields a worse bound.

    Next, since $J_2(\eK) \ge J_2(\sK)$ due to the optimality of $\sK$, we have 
    \begin{align*}
        (c) = J_2^{\delta}(\eK) - J_2(\eK) \le J_2^{\delta}(\eK) - J_2(\sK).
    \end{align*}
    By Lemma~\ref{lem:pdl},
    \begin{align*}
        J_2^{\delta}(\eK) - J_2(\sK)
        \overset{(i)}{=}\;& \nlsum_{t=\ell}^{T-1} \ipc{\sR_t + (\sB_t)^{\top} \sP_{t+1} \sB_t}{\Cov((\eK_t - \sK_t) x_t^{\delta} + \eK_t \delta_t)}_F \\
        \overset{(ii)}{=}\;& \bigo\Big(\nlsum_{t=\ell}^{T-1} d_x (\eps_2^2 \|\Cov(x_t^{\delta})\|_2 + \|\Cov(\delta_t)\|_2)\Big) \\
        \overset{(ii)}{=}\;& \bigo((\eps_2^2 + \varphi_2^2) d_x a_2^{\ell}),
    \end{align*}
    where in $(i)$ the term $E_t^K = 0$ in Lemma~\ref{lem:pdl} for $K = \sK$, $(ii)$ follows from that for $t \ge \ell$, $\|\eK_t\|_2 = \bigo(\|\sK_t\|_2)$ due to $\|\eK_t - \sK_t\|_2 = \bigo(\eps_2)$ and that $\|\sK_t\|_2 = \bigo(1)$ as argued in the proof of Lemma~\ref{lem:rank-deficient}, and in $(iii)$ we use $\|\Cov(x_{t}^{\delta})\|_2 = \bigo(a_2^{\ell})$ for $t \le \ell$, $\|\Cov(\delta_t)\|_2 \le \varphi_2^2 \max_{\tau \le t}\|\Cov(x_{\tau}^{\delta})\|_2$ for $t \ge \ell$, and that due to Lemma~\ref{lem:cov-sum} and $\varphi_2 = \bigo(T^{-1})$, for $t \ge \ell$, 
    \begin{align*}
        \|\Cov(x_{t}^{\delta})\|_2^{1/2}
        =\;& \Big\|\Cov\Big(\sPhi_{t, \ell}(\eK) x_{\ell}^{\delta}   + \nlsum_{\tau = \ell}^{t-1} \sPhi_{t, \tau + 1}(\eK) (w_{\tau} + \sB_{\tau} \eK_{\tau} \delta_{\tau})\Big) \Big\|_2^{1/2} \\
        =\;& \bigo(a_2^{\ell/2} + (t - \ell) \varphi_2 a_2^{\ell/2}) = \bigo(a_2^{\ell/2}).
    \end{align*}
    Hence,
    \begin{align*}
        (c) = \bigo\big((\eps_2^2 + \varphi_2^2) d_x a_2^{\ell} (T-\ell)\big).
    \end{align*}

    For $(d)$, by~\eqref{eqn:pek-sp}, we have $\|P_t^{\eK} - \sP_t\|_2 = \bigo(\eps_2^2 (T - \ell))$ for $t \ge \ell$. Hence,
    \begin{align*}
        (d) \le d_x \nlsum_{t=\ell}^{T-1} \|\Cov(w_t)\|_2 \cdot \|P_{t+1}^{\eK} - \sP_{t+1}\|_2
        = \bigo(\eps_2^2 d_x (T - \ell)^2).
    \end{align*}

    Finally, combining the above bounds on $(a), (b), (c), (d)$, we have  
    \begin{align*}
        J^{\delta}(\eK) - J(\sK)
        =\;& \bigo( (\eps_1 + \varphi_1) d_x a_4^{\ell}) + \bigo(\eps_2 d_x a_2^{\ell})
        + \bigo((\eps_2^2 + \varphi_2^2) d_x a_2^{\ell} (T-\ell)) + \bigo(\eps_2^2 d_x (T - \ell)^2) \\
        =\;& \bigo((\eps_1 + \varphi_1) d_x a_4^{\ell} + \eps_2^2 d_x (T - \ell)^2),
    \end{align*}
    which completes the proof. 
\end{proof}

\vspace*{-10pt}
\subsection{Proof of Theorem~\ref{thm:main}}
\label{sec:main-proof}

Now we are ready to prove Theorem~\ref{thm:main}.
Algorithm~\ref{alg:corel} has three main steps: state representation function learning (Algorithm~\ref{alg:repl}), latent model identification (Algorithm~\ref{alg:sys-id}), and planning by RDE~\eqref{eqn:p-riccati}. Correspondingly, the analysis below is organized around these three steps.

\vspace{3pt}
\noindent\textbf{Recovery of the state representation function.}
By Proposition~\ref{prp:multi-step-cost}, with $u_t = \gauss(0, \sigma_u^2 I)$ for all $0\le t\le T-1$ to system~\eqref{eqn:po-ltv}, the $k$-step cumulative cost starting from step $t$, where $k = 1$ for $0\le t\le \ell-1$ and $k = m\wedge (T - t + 1)$ for $\ell \le t\le T$, is given under the normalized parameterization by
\begin{align*}
    \oc_t := c_t + c_{t+1} + \ldots + c_{t+k-1}
    = \|\spz_t\|^2 + \nlsum_{\tau = t}^{t+k-1} \|u_{\tau}\|_{\sR_{\tau}}^2 + \ob_t + \ooe_t,
\end{align*}
where $\ob_t = \bigo(k)$, and $\ooe_t$ is a zero-mean subexponential random variable with $\|\ooe_t\|_{\psi_1} = \bigo(k d_x^{1/2})$.

Then, it is clear that Algorithm~\ref{alg:repl} recovers latent states $\spz_t = \spM_t h_t$, where $0\le t\le T$, by a combination of quadratic regression and low-rank approximate factorization. Below we drop the superscript prime for notational simplicity, but keep in mind that the optimal state representation function $(\sM_t)_{t=0}^{T}$, the corresponding latent states $(\sz_t)_{t=0}^{T}$, and the true latent system parameters $((\sA_t, \sB_t)_{t=0}^{T-1}, (\sQ_t)_{t=0}^{T})$ are all with respect to the \mbox{normalized parameterization}.

For all $0\le t \le T$, let $\sN_t := (\sM_t)^{\top} \sM_t$. 
By Assumption~\ref{asp:bounded}, $\|\sN_t\|_2 = \bigo(1)$. Since $h_t = [y_{0:t}; u_{0:(t-1)}]$ and $(\Cov(y_t))_{t=0}^{T}, (\Cov(u_t))_{t=0}^{T-1}$ have $\bigo(1)$ operator norms due to Assumptions~\ref{asp:stability} and~\ref{asp:bounded}, by Lemma~\ref{lem:cov-cat}, $\Cov(h_t) = \E[h_t h_t^{\top}] = \bigo(t)$.
Hence, for quadratic regression, Lemma~\ref{lem:qr} (detailed later) and the union bound over time steps guarantee that as long as $n \ge a T^4 (d_y + d_u)^4 \log(a T^3 (d_y + d_u)^2 / p) \log (T/p)$ for an absolute constant $a > 0$, with probability at least $1 - p$, for all $0\le t\le \ell - 1$,
\begin{align*}
    \| \eN_t - \sN_t \|_F
    = \bigo(k (t \vee 1)^3 (d_y + d_u)^2 d_x^{1/2} n^{-1/2} \log^{1/2}(\ell/p));
\end{align*}
and for all $\ell\le t\le T$,
\begin{align*}
    \| \eN_t - \sN_t \|_F
    = \bigo(k (t \vee 1)^3 (d_y + d_u)^2 d_x^{1/2} n^{-1/2} \log^{1/2}(T/p)).
\end{align*}

Now let us bound the distance between $\eM_t$ and $\sM_t$.
Recall that we use $d_h = (t+1) d_y + t d_u$ as a shorthand.
The estimate $\eN_t$ may not be positive semidefinite. Let $\eN_t = U \Lambda U^{\top}$ be its eigenvalue decomposition, with the $d_h \times d_h$ matrix $\Lambda$ having descending diagonal elements.
Let $\Sigma:= \max(\Lambda, 0)$. Then $\tN_t =: U \Sigma U^{\top}$ is the projection of $\eN_t$ onto the positive semidefinite cone~\citep[Section 8.1.1]{boyd2004convex} with respect to the Frobenius norm. Since $\sN_t \succcurlyeq 0$,  
\begin{align*}
    \|\tN_t - \sN_t\|_F \le \|\eN_t - \sN_t\|_F.
\end{align*}
The low-rank factorization is essentially a combination of low-rank approximation and matrix factorization. 
For $d_h < d_x$, $\tM_t =: [\Sigma^{1/2} U^{\top}; 0_{(d_x - d_h) \times d_h}]$ constructed by padding zeros satisfies $\tM_t^{\top} \tM_t = \tN_t$.
For $d_h \ge d_x$, construct $\tM_t =: \Sigma_{d_x}^{1/2} U_{d_x}^{\top}$, where $\Sigma_{d_x}$ is the top-left $d_x \times d_x$ block in $\Sigma$ and $U_{d_x}$ consists of $d_x$ columns of $U$ from the left. By the Eckart-Young-Mirsky theorem, $\tM_t^{\top} \tM_t = U_{d_x} \Sigma_{d_x} U_{d_x}^{\top}$ satisfies
\begin{align*}
    \|\tM_t^{\top} \tM_t - \tN_t\|_F = \min_{N \in \R^{d_h \times d_h}, \rank(N) \le d_x} \|N - \tN_t\|_F.
\end{align*}
Hence, 
\begin{align*}
    \|\tM_t^{\top} \tM_t - \sN_t\|_F
    \le \|\tM_t^{\top} \tM_t - \tN_t\|_F + \|\tN_t - \sN_t\|_F
    \le 2 \|\tN_t - \sN_t\|_F \le 2 \|\eN_t - \sN_t\|_F.
\end{align*}
From now on, we consider $0\le t \le \ell-1$ and $\ell\le t \le T$ separately, since, as we will show, in the latter case we have the additional condition that $\rank(\sM_t) = d_x$. 

For $0 \le t \le \ell-1$, $k = 1$. By Lemma~\ref{lem:mat-fac} (detailed later), there exists a $d_x \times d_x$ orthogonal matrix $S_t$, such that $\|\tM_t - S_t \sM_t \|_2 \le \|\tM_t - S_t \sM_t \|_F = \bigo((t\vee 1)^{3/2} (d_y + d_u) d_x^{3/4} n^{-1/4} \log^{1/4}(\ell/p))$. 
Recall that $\eM_t = \svt{}(\tM_t, \theta) = \big(\I_{[\theta, +\infty)}(\Sigma_{d_x}^{1/2}) \odot \Sigma_{d_x}^{1/2}\big) U_{d_x}^{\top}$. Then, 
\begin{align*}
    \|\eM_t - \tM_t\|_2 = \|(\I_{[\theta, +\infty)}(\Sigma_{d_x}^{1/2}) \odot \Sigma_{d_x}^{1/2} - \Sigma_{d_x}^{1/2}) U_{d_x}^{\top} \|_2 
    \le \|\I_{[\theta, +\infty)}(\Sigma_{d_x}^{1/2}) \odot \Sigma_{d_x}^{1/2} - \Sigma_{d_x}^{1/2}\|_2 \le \theta.
\end{align*}
Hence, the distance between $\eM_t$ and $\sM_t$ satisfies 
\begin{align}
    \|\eM_t - S_t \sM_t\|_2
    =\;& \|\eM_t - \tM_t + \tM_t - S_t \sM_t\|_2  \nonumber \\ 
    \le\;& \| \eM_t - \tM_t\|_2 + \|\tM_t - S_t \sM_t\|_2  \nonumber \\
    \le\;& \theta + \bigo((t\vee 1)^{3/2} (d_y + d_u) d_x^{3/4} n^{-1/4} \log^{1/4}(\ell/p))  \nonumber \\
    =\;& \bigo(\ell^{3/2} (d_y + d_u) d_x^{3/4} n^{-1/4} \log^{1/4}(\ell/p)),  \label{eqn:m-err-init}
\end{align}
by the choice of $\theta = \Theta(\ell^{3/2} (d_y + d_u) d_x^{3/4} n^{-1/4} \log^{1/4}(\ell/p))$.
As a result, since $\ez_t = \eM_t h_t$ and $\sz_t = \sM_t h_t$,
\begin{align*}
    \big\|\nlsum_{i=1}^{n} (\ez_t^{(i)} - S_t (\sz_t)^{(i)}) (\ez_t^{(i)} - S_t (\sz_t)^{(i)})^{\top}\big\|_2
    = \|\eM_t - S_t \sM_t\|_2^2 \cdot \big\|\nlsum_{i=1}^{n} h_t^{(i)} (h_t^{(i)})^{\top} \big\|_2.
\end{align*}

By~\citep[Theorem 6.1]{wainwright2019high}, with probability at least $1 - p$, as long as $n \ge (t\vee 1)(d_y + d_u) + \log(1/p)$,
\begin{align*}
    \big\|\nlsum_{i=1}^{n} h_t^{(i)} (h_t^{(i)})^{\top} \big\|_2 = \bigo(\|\Cov(h_t)\|_2 n) = \bigo(tn).
\end{align*}
Hence, define $\epsilon_z := \ell^{2} (d_y + d_u) d_x^{3/4} n^{-1/4} \log^{1/4}(1/p)$. Then, with probability at least $1 - p$, 
\begin{align*}
    \big\|\nlsum_{i=1}^{n} (\ez_t^{(i)} - S_t (\sz_t)^{(i)}) (\ez_t^{(i)} - S_t (\sz_t)^{(i)})^{\top}\big\|_2
    \le \bigo(\epsilon_z^2 n).
\end{align*}

On the other hand, threshold $\theta$ ensures a lower bound on $\sigma_{\min}^{+}(\nlsum_{i=1}^{n} \ez_t^{(i)} (\ez_t^{(i)})^{\top})$. As shown in the proof of Lemma~\ref{lem:lr-rd}, this property is important for ensuring the system identification outputs $\eA_t$ and $\eB_t$ have bounded norms.
Specifically,
\begin{align*}
    \sigma_{\min}^{+}(\nlsum_{i=1}^{n} \ez_t^{(i)} (\ez_t^{(i)})^{\top})
    = \sigma_{\min}^{+}(\nlsum_{i=1}^{n} \eM_t h_t^{(i)} (h_t^{(i)})^{\top} \eM_t^{\top})
    \overset{(i)}{\ge} (\sigma_{\min}^{+}(\eM_t))^2 \sigma_{\min}\big(\nlsum_{i=1}^{n} h_t^{(i)} (h_t^{(i)})^{\top}\big),
\end{align*}
where $(i)$ is due to~\citep[H.1.i, Chapter 9]{marshall1979inequalities} and that the minimum positive singular values of $\nlsum_{i=1}^{n} \eM_t h_t^{(i)} (h_t^{(i)})^{\top} \eM_t^{\top}$ and $\eM_t$ are both their $k$th singular value, where $k$ is the rank of $\eM_t$, by considering the singular value decomposition of $\eM_t$. By standard concentration in analyzing linear regression~\citep{wainwright2019high}, with probability at least $1-p$, as long as $n \ge a \log(1/p)$ for some absolute constant $a > 0$, $\sigma_{\min}\big(\nlsum_{i=1}^{n} h_t^{(i)} (h_t^{(i)})^{\top}\big) = \Omega(n)$. Hence, 
\begin{align*}
    \sigma_{\min}^{+}(\nlsum_{i=1}^{n} \ez_t^{(i)} (\ez_t^{(i)})^{\top}) \ge \theta^2 \cdot \Omega(n) = \Omega(\theta^2 n).
\end{align*}

For $\ell\le t\le T$, $k \le m$. By Proposition~\ref{prp:full-rank-cov}, $\Cov(\sz_t)$ has full rank, $\sigma_{\min}(\Cov(\sz_t)) = \Omega(\nu^2)$ and $\sigma_{\min}(\sM_t) = \Omega(\nu t^{-1/2})$. 
Recall that for $\ell \le t \le T$, we simply set $\eM_t = \tM_t$.
Then, by Lemma~\ref{lem:mat-fac-fr}, there exists a $d_x \times d_x$ orthogonal matrix $S_t$, \mbox{such that}
\begin{align}
    \|\eM_t - S_t \sM_t\|_F 
    =\;& \bigo(\sigma_{\min}^{-1}(\sM_t)) \|\tM_t^{\top} \tM_t - \sN_t\|_F  \nonumber \\
    =\;& \bigo(\nu^{-1} m t^{7/2} (d_y + d_u)^2 d_x^{1/2} n^{-1/2} \log^{1/2}(T/p)),  \label{eqn:m-err}
\end{align}
which is also an upper bound on $\|\eM_t -  S_t \sM_t\|_2$.
As a result,
\begin{align*}
    \big\|\nlsum_{i=1}^{n} (\ez_t^{(i)} - S_t (\sz_t)^{(i)}) (\ez_t^{(i)} - S_t (\sz_t)^{(i)})^{\top}\big\|_2
    =\;& \|\eM_t - S_t \sM_t\|_2^2 \cdot \big\|\nlsum_{i=1}^{n} h_t^{(i)} (h_t^{(i)})^{\top} \big\|_2 \\
    =\;& \bigo(\|\eM_t - S_t \sM_t\|_2^2 t n).
\end{align*}

Consider 
\begin{align*}
    \big\|\nlsum_{i=1}^{n} \ez_t^{(i)} (\ez_t^{(i)})^{\top} - S_t (\sz_t)^{(i)} ((\sz_t)^{(i)})^{\top} S_t^{\top}\big\|_2
    =\;& \nlsum_{i=1}^{n} \big(\big\|\ez_t^{(i)}\big\| + \big\|(\sz_t)^{(i)}\big\|\big) \cdot \big\|\ez_t^{(i)} - S_t (\sz_t)^{(i)}\big\| \\
    \overset{(i)}{=}\;& n d_x^{1/2} \log^{1/2}(n/p) \cdot \big\|\ez_t^{(i)} - S_t (\sz_t)^{(i)}\big\|,
\end{align*}
where $(i)$ holds with probability $1 - p$.
Hence, there exists an absolute constant $c > 0$, such that if $n \ge c \nu^{-6} m^2 T^6 (d_y + d_u)^5 d_x^{2} \log^2(n/p)$,
\begin{align*}
    &\sigma_{\min}\big(\nlsum_{i=1}^{n} \ez_t^{(i)} (\ez_t^{(i)})^{\top}\big) \\
    \ge\;& \sigma_{\min}\big(\nlsum_{i=1}^{n}(\sz_t)^{(i)} ((\sz_t)^{(i)})^{\top}\big)
    - \big\|\nlsum_{i=1}^{n} \ez_t^{(i)} (\ez_t^{(i)})^{\top} - S_t (\sz_t)^{(i)} ((\sz_t)^{(i)})^{\top} S_t^{\top}\big\| \\
    \ge\;& \sigma_{\min}\big(\nlsum_{i=1}^{n}(\sz_t)^{(i)} ((\sz_t)^{(i)})^{\top}\big) / 2 = \Omega(\nu^2 n).
\end{align*}
This is needed for the analysis of estimating $\eA_t, \eB_t$ in the next step.
For the analysis of estimating $\sQ_t$ by quadratic regression in Lemma~\ref{lem:qr}, we need the sub-Gaussianity of $\|\ez_t - S_t \sz_t\|_2$. 
Notice that 
\begin{align*}
    \|\ez_t - S_t \sz_t\|_2 \le\;& \|\eM_t - S_t \sM_t\|_2 \|h_t\|.
\end{align*}
Since the $\ell_2$-norm of $h_t = [y_{0:t}; u_{0:(t-1)}]$ is sub-Gaussian with its mean and sub-Gaussian norm bounded by $\bigo((t (d_y + d_u))^{1/2})$, we have that $\|\ez_t - S_t \sz_t\|_2$ is sub-Gaussian with its mean and sub-Gaussian norm bounded by 
\begin{align*}
    \bigo(\nu^{-1} m t^{4} (d_y + d_u)^{5/2} d_x^{1/2} n^{-1/2} \log^{1/2}(T/p)).
\end{align*}

\vspace{3pt}
\noindent\textbf{Identification of the latent model.}
The latent dynamics $(\sA_t, \sB_t)_{t=0}^{T-1}$ is identified in Algorithm~\ref{alg:sys-id}, using $(\ez_t^{(i)})_{i=1, t=0}^{N, T}$ produced by Algorithm~\ref{alg:repl}, by ordinary least squares.
Recall from Proposition~\ref{prp:state-est-lds} that $\sz_{t+1} = \sA_t \sz_t + \sB_t u_t + L_{t+1} i_{t+1}$. With the transforms on $\sz_t$ and $\sz_{t+1}$, we have 
\begin{align*}
    S_{t+1} \sz_{t+1} = (S_{t+1} \sA_{t} S_{t}^{\top}) S_t \sz_t + S_{t+1} \sB_t u_t + S_{t+1} L_{t+1} i_{t+1},
\end{align*}
and $(\sz_{t})^{\top} \sQ_{t} \sz_t = (S_t \sz_t)^{\top} S_t \sQ_t S_t^{\top} S_t \sz_t$.
Under control $u_t \sim \gauss(0, \sigma_u^2 I_{d_u})$ for $0\le t\le T-1$, we know that $\sz_t$ is a zero-mean Gaussian random vector; so is $S_t \sz_t$. Let $\sSigma_t = \E[S_t \sz_t (\sz_t)^{\top} S_t^{\top}]$ be its covariance. By Assumptions~\ref{asp:stability} and ~\ref{asp:bounded}, $(\sSigma_t)_{t=0}^{T}$ and $(L_{t+1} i_{t+1})_{t=0}^{T-1}$ have $\bigo(1)$ operator norms.

For $0\le t\le \ell-1$, we need a bound for the estimation error of rank-deficient and perturbed linear regression. 
By Lemma~\ref{lem:lr-rd}, there exists an absolute constant $c > 0$, such that as long as $n \ge c(d_x + d_u + \log(1/p))$, with probability at least $1 - p$, 
\begin{align*}
    &\|([\eA_t, \eB_t ] - [S_{t+1} \sA_t S_t^{\top}, S_{t+1} \sB_t]) \diag((\sSigma_t)^{1/2}, \sigma_u I_{d_u}) \|_2 \\
    =\;& \bigo\big( \beta^{-1} (1 + \theta^{-1} \epsilon_z) \epsilon_z + \epsilon_z  + n^{-1/2} (d_x + d_u + \log(1/p))^{1/2}\big).
\end{align*}
By substituting the expressions for $\epsilon_z$ and $\theta$, we have
\begin{align*}
    \| (\eA_t - S_{t+1} \sA_t S_t^{\top}) (\sSigma_t)^{1/2} \|_2
    = \bigo((1 + \beta^{-1})\ell^{5/2} (d_y + d_u) d_x^{3/4} n^{-1/4} \log^{1/4}(\ell/p)),
\end{align*}
which is also a bound on $\| \eB_t - S_{t+1} \sB_t \|_2$. Meanwhile, by Claim~\ref{clm:a-norm}, $\|\eA_t\|_2 = \bigo(\ell)$.

For $\ell\le t \le T-1$, by Lemma~\ref{lem:lr} and the union bound over time steps, with probability at least $1 - p$, 
\begin{align*}
    &\|[\eA_t, \eB_t ] - [S_{t+1} \sA_t S_t^{\top}, S_{t+1} \sB_t] \|_2 \\
    =\;& \bigo\big( (\nu^{-1} + \sigma_u^{-1}) \big(t^{1/2} \|\eM_t - S_t \sM_t\|_2 + n^{-1/2} (d_x + d_u + \log(T/p))^{1/2}\big)\big).
\end{align*}
By substituting the expression for $\|\eM_t - S_t \sM_t\|_2$ in~\eqref{eqn:m-err}, we have
\begin{align*}
    \| \eA_t - S_{t+1} \sA_t S_t^{\top} \|_2
    = \bigo\Big((1 + \nu^{-2}) m t^4 (d_y + d_u)^{2} d_x^{1/2} n^{-1/2} \log^{1/2}(T/p) \Big),
\end{align*}
which is also a bound on $\| \eB_t - S_{t+1} \sB_t \|_2$.

By Assumption~\ref{asp:cost-obs}, $(\sQ_t)_{t=0}^{\ell-1}$ and $\sQ_T$ are positive definite; they are identity matrices under the normalized parameterization. For $(\sQ_t)_{t=\ell}^{T-1}$, which may not be positive definite, we identify them in Algorithm~\ref{alg:sys-id} by~\eqref{eqn:q-id}. 
By Lemma~\ref{lem:qr}, 
\begin{align*}
    & \|\tQ_t - S_t \sQ_t S_t^{\top}\|_F \\
    =\;& \bigo(d_x^2 \log^2(n/p) \cdot m t^4 (d_y + d_u)^{5/2} d_x^{1/2} n^{-1/2} \log^{1/2}(T/p)
    + d_x^2 d_x^{1/2} n^{-1/2} \log^{1/2}(T/p)) \\
    =\;& \bigo( m t^4 (d_y + d_u)^{5/2} d_x^{5/2} n^{-1/2} \log^{5/2}(nT/p)),
\end{align*}
where we have used $\nu = \Omega(1)$.
By construction, $\eQ_t$ is the projection of $\tQ_t$ onto the positive semidefinite cone with respect to the Frobenius norm~\citep[Section 8.1.1]{boyd2004convex}. Since $S_t \sQ_t S_t^{\top} \succcurlyeq 0$,  we have 
\begin{align*}
    \|\eQ_t - S_t \sQ_t S_t^{\top}\|_F \le \|\tQ_t - S_t \sQ_t S_t^{\top}\|_F.
\end{align*}

\vspace{3pt}
\noindent\textbf{Certainty equivalent linear quadratic control.}
The last step of Algorithm~\ref{alg:corel} is to compute the optimal controller in the estimated system $((\eA_t, \eB_t, \sR_t)_{t=0}^{T-1}, (\eQ_t)_{t=0}^{T})$ by RDE~\eqref{eqn:p-riccati}.

Since $\sz_t = \E[x_t \given h_t]$, the residual $\sz_t - x_t$ is independent of $h_t$ and $\sz_t$. Hence, we have
\begin{align*}
    \E[x_t x_t^{\top} ] 
    = \E[\sz_t (\sz_t)^{\top}] + \E[(x_t - \sz_t)(x_t - \sz_t)^{\top}],
\end{align*}
where $\E[(x_t - \sz_t)(x_t - \sz_t)^{\top}]$ is a constant matrix regardless of actions $(u_\tau)_{\tau \le t}$.
Hence, 
\begin{align*}
    \E[x_t^{\top} \sQ_t x_t] - \E[(\sz_t)^{\top} \sQ_t \sz_t] = \ipc{\sQ_t}{\E[(x_t - \sz_t)(x_t - \sz_t)^{\top}]}_{F}
\end{align*}
is a constant, and it suffices to consider the latent state space for studying the policy suboptimality gap. 

In the latent state space, action $u_t = \eK_t \eM_t h_t = \eK_t (S_t \sz_t + \delta_t)$, where $\delta_t := (\eM_t - S_t \sM_t) h_t$ is a Gaussian noise vector correlated with $\sz_t$. 
Since $h_0 = y_0 = \sC_0 \sz_0 + \sC_0 (x_0 - \sz_0) + v_0$, we have $\|\Cov(y_{0})\|_2 = \bigo(\|\Cov(\sz_0)\|_2)$ and $\|\Cov(\delta_0)\|_2 = \bigo(\|\eM_0 - S_0 \sM_0\|_2^2 \|\Cov(\sz_0)\|_2)$.
Define $\iota_0 := \|\eM_0 - S_0 \sM_0\|_2$ and $\iota_t := t^{1/2} (1 + \max_{\tau \le t - 1} \|\eK_{\tau}\|_2) \|\eM_t - S_t \sM_t\|_2$ for $1\le t\le T-1$. 
Below we use induction to show that as long as $(\|\eM_t - S_t \sM_t\|_2)_{0\le t\le T-1}$ are small enough, $\|\Cov(\delta_t)\|_2 \le \iota_t^2 \max_{\tau \le t} \|\Cov(\sz_{\tau})\|_2$ for all $0\le t\le T-1$.

Suppose that $\|\Cov(\delta_{t})\|_2 \le \iota_t^2 \max_{\tau \le t} \|\Cov(\sz_{\tau})\|_2$ for all $t \le t_1$.
For any $t \ge 1$, by the definition of the innovation term $i_{t}$ and the dynamics of $\sz_t$ in Proposition~\ref{prp:state-est-lds},  we have 
\begin{align*}
    y_{t} = \sC_{t} (\sz_{t} - \sL_{t} i_{t}) + i_{t} = \sC_{t} \sz_{t} + (I - \sC_{t} \sL_{t}) i_{t}.
\end{align*}
Hence, for any $t \ge 0$, $\|\Cov(y_{t})\|_2 = \bigo(\|\Cov(\sz_{t})\|_2)$ as $\|\Cov(i_t)\|_2 = \bigo(1)$.
For any $t \ge 0$, we also have
\begin{align*}
    \|\Cov(u_t)\|_2 =\ \|\Cov(\eK_t (S_t \sz_t + \delta_t))\|_2
    = \bigo(\|\eK_t\|_2^2 (\|\Cov(\sz_t)\|_2 + \|\Cov(\delta_t)\|_2)).
\end{align*}

By Lemma~\ref{lem:cov-cat}, we further have
\begin{align*}
    \|\Cov(h_{t_1 + 1})\|_2
    \le\;& \nlsum_{t=0}^{t_1 + 1} \|\Cov(y_t)\|_2 + \nlsum_{t=0}^{t_1} \|\Cov(u_t)\|_2 \\
    =\;& \nlsum_{t=0}^{t_1 + 1} \bigo(\|\Cov(\sz_t)\|_2)
    + \nlsum_{t=0}^{t_1} \bigo(\|\eK_t\|_2^2 (\|\Cov(\sz_t)\|_2 + \|\Cov(\delta_t)\|_2)).
\end{align*}
By the induction hypothesis on $\|\Cov(\delta_t)\|_2$, we have
\begin{align*}
    \nlsum_{t=0}^{t_1} \|\eK_t\|_2^2  (\|\Cov(\sz_t)\|_2 + \|\Cov(\delta_t)\|_2))
    \le (1 + \max_{t\le t_1} \iota_t^2) (t_1 + 1) \max_{t \le t_1} \|\eK_{t}\|_2^2  \max_{t \le t_1} \|\Cov(\sz_{t})\|_2.
\end{align*}
Hence, as long as $\|\eM_t - S_t \sM_t\|_2 = \bigo(t^{-1/2} (1 + \max_{\tau \le t - 1} \|\eK_{\tau}\|_2)^{-1})$ for all $t \le t_1$ such that $\max_{t\le t_1} \iota_t^2 = \bigo(1)$, we complete the induction by
\begin{align*}
    &\|\Cov(\delta_{t_1 + 1})\|_2 \\
    =\;& \bigo(\|\eM_{t_1 + 1} - S_{t_1 + 1} \sM_{t_1 + 1}\|_2^2) \|\Cov(h_{t_1 + 1})\|_2 \\
    =\;& \bigo\Big(\|\eM_{t_1 + 1} - S_{t_1 + 1} \sM_{t_1 + 1}\|_2^2 \cdot \big(\nlsum_{t=0}^{t_1 + 1}\|\Cov(\sz_{t})\|_2
    + (t_1 + 1) (1 + \max_{t \le t_1} \|\eK_{t}\|_2^2) \max_{t \le t_1} \|\Cov(\sz_{t})\|_2 \big)\Big) \\
    =\;& \bigo( \iota_{t_1 + 1}^2 \max_{t\le t_1 + 1} \|\Cov(\sz_t)\|_2 ),
\end{align*}
where we use the definition of $\iota_{t_1 + 1}$.

Now let us bound the operator norms of $(K_t)_{t=0}^{T-1}$. For $0\le t \le \ell - 1$, from the backward recursion in~\eqref{eqn:p-riccati} and~\eqref{eqn:K-riccati}, there exists a dimension-free constant $\kappa > 1$ that depends polynomially on $\ell$ and other problem parameters, such that $\|\eK_t\|_2 = \bigo(\kappa^{\ell - t})$. For $\ell \le t \le T - 1$, let $\eps_2$ be the maximum of $\|\eA_t - S_{t+1}\sA_t S_t^{\top}\|_2, \|\eB_t - S_{t+1}\sB_t\|_2, \|\eQ_t - S_t\sQ_t S_t^{\top}\|_2$ for all $t \ge \ell$. Then,
\begin{equation}  \label{eqn:eps_2}
    \begin{aligned}
        \eps_2 = \bigo((1 + \nu^{-2}) m T^4
        (d_y + d_u)^{5/2} d_x^{5/2} n^{-1/2} \log^{5/2}(nT/p)).    
    \end{aligned}
\end{equation}
Applying Lemma~\ref{lem:lq}
~to the last $T - \ell$ steps, we have $\|\eK_{t} - \sK_{t} S_t^{\top}\|_2 = \bigo(\eps_2)$. Hence, we have \mbox{$\|\eK_t\|_2 = \bigo(\|\sK_t\|_2) = \bigo(1)$}.

We have shown in~\eqref{eqn:m-err-init} that for $0\le t\le \ell - 1$,
\begin{align*}
    \|\eM_t - S_t \sM_t\|_2 
    = \bigo(\ell^{3/2} (d_y + d_u) d_x^{3/4} n^{-1/4} \log^{1/4}(\ell/p)),
\end{align*}
and in~\eqref{eqn:m-err} that for $t\ge \ell$,
\begin{align*}
    \|\eM_t - S_t \sM_t\|_2
    = \bigo(\nu^{-1} m t^{7/2} (d_y + d_u)^2 d_x^{1/2} n^{-1/2} \log^{1/2}(T/p)).
\end{align*}
Let $n$ be large enough such that $\|\eM_t - S_t \sM_t\|_2 = \bigo(t^{-1/2} (1 + \max_{\tau \le t - 1} \|\eK_t\|_2)^{-1})$ is satisfied for all $t \ge 1$, and the conditions in Lemma~\ref{lem:e2e} are satisfied, with $\eps_2$ satisfying~\eqref{eqn:eps_2} and $\eps_1$, $\varphi_1$, $\varphi_2$ satisfying
\begin{align*}
    &\eps_1 =  \bigo((1 + \beta^{-1})\ell^{5/2}(d_y + d_u) d_x^{3/4} n^{-1/4} \log^{1/4}(\ell/p)), \\
    &\varphi_1 = \bigo(\ell^{2} \kappa^{\ell} (d_y + d_u) d_x^{3/4} n^{-1/4} \log^{1/4}(\ell/p)), \\
    &\varphi_2 = \bigo(\nu^{-1} \kappa^{\ell} m T^4 (d_y + d_u)^2 d_x^{1/2} n^{-1/2} \log^{1/2}(T/p)).
\end{align*}
Hence, by Lemma~\ref{lem:e2e}, there exists dimension-free constant $a_4 > 1$ depending polynomially on $\ell$ and other problem parameters, such that with probability at least $1 - p$, 
\begin{align*}
    J(\epi) - J(\spi)
    =\;& \bigo\Big((1 + \beta^{-1})\ell^{5/2} (d_y + d_u) d_x^{7/4} n^{-1/4} \log^{1/4}(\ell/p) (\kappa a_4)^{\ell} \\
    &\quad + (1 + \nu^{-4}) m^2 T^{10} (d_y + d_u)^{5} d_x^{6} n^{-1} \log^{5}(nT/p) \Big),
\end{align*}
where $J(\epi)$, $J(\spi)$ correspond to $J^{\delta}(\eK)$, $J(\sK)$ in Lemma~\ref{lem:e2e}, respectively, in the latent state space.
\vspace*{-3pt}
\endproof

\vspace*{-4pt}
\section{Concluding remarks}

We examined the cost-driven state representation learning methods in time-varying LQG control. With a finite-sample analysis, we showed that a direct, cost-driven state representation learning algorithm effectively solves LQG. In the analysis, we revealed the importance of using multi-step cumulative costs as the supervision signal, and the dependence on $\ell$, the controllability index, due to early-stage insufficient excitement of the system. 
For the same reason, our policy suboptimality gap has a significantly worsened 
dependence on $\ell$, as the latent model can only be partially identified, and the learned latent model may not be stable. Hence, an immediate question is how we can learn a stable latent model to improve the dependence on $\ell$.
A major limitation of our method is the use of history-based state representation functions; recovering the recursive Kalman filter would be ideal. 

In Part II of this work, we will explore how the cost-driven state representation learning approach performs in the infinite-horizon LTI setting, as well as discuss more opportunities that this work has opened up for future research.

\vspace*{-4pt}
\section*{Acknowledgment}
YT, SS acknowledge partial support from NSF CCF-2112665 (TILOS AI Research Institute). KZ acknowledges the support from the Simons-Berkeley Research Fellowship, the U.S. Army Research Office grant W911NF-24-1-0085, and the NSF CAREER Award 2443704. 
The authors also thank Xiang Fu, Horia Mania, and Alexandre Megretski for helpful discussions.

\bibliographystyle{plainnat}
\bibliography{latent}

\appendix
\section*{Appendix}

\begin{lemma}  \label{lem:subexp}
    Let $x\sim \gauss(0, \Sigma_x)$ and $y\sim \gauss(0, \Sigma_y)$ be $d$-dimensional Gaussian random vectors. Let $Q$ be a $d\times d$ positive semidefinite matrix. Then, there exists an absolute constant $c > 0$ such that 
    \vspace*{-2pt}
    \begin{align*}
        \|\ipc{x}{y}_{Q}\|_{\psi_1} \le c \sqrt{d} \|Q\|_2 \sqrt{\|\Sigma_x\|_2 \|\Sigma_y\|_2}.
    \end{align*}
\end{lemma}

\begin{proof} 
    Since $|\ipc{x}{y}_{Q}| = |x^{\top} Q y| \le \|x\| \|Q\|_2 \|y\|$, 
    \begin{align*}
        \|\ipc{x}{y}_{Q}\|_{\psi_1} =\;& \||\ipc{x}{y}_{Q}|\|_{\psi_1} \\
        \le\;& \|\|x\| \|Q\|_2 \|y\|\|_{\psi_1} = \|Q\|_2 \cdot \|\|x\| \|y\|\|_{\psi_1}.
    \end{align*}
    For $x\sim \gauss(0, \Sigma_x)$, we know that $\|x\|$ is sub-Gaussian. Actually, by writing $x = \Sigma_x^{1/2} g$ for $g \sim \gauss(0, I)$, we have 
    \begin{align*}
        \|\|x\|\|_{\psi_2} =\;& \|\|\Sigma_x^{1/2} g\|\|_{\psi_2} \\ 
        \le\;& \|\|\Sigma_x^{1/2}\|_2 \|g\|\|_{\psi_2} = \|\Sigma_x^{1/2}\|_2 \|\|g\|\|_{\psi_2}.
    \end{align*}
    The distribution of $\|g\|_2$ is known as $\chi$ distribution, and we know that $\|\|g\|\|_{\psi_2} = c' d^{1/4}$ for an absolute constant $c' > 0$ (see, e.g.,~\citep{wainwright2019high}).
    Hence, $\|\|x\|\|_{\psi_2} \le c' d^{1/4} \|\Sigma_x\|_2^{1/2}$. 

    Similarly, $\|\|y\|\|_{\psi_2} \le c' d^{1/4} \|\Sigma_y\|_2^{1/2}$. Since $\|\|x\| \|y\|\|_{\psi_1} \le \|\|x\|\|_{\psi_2} \|\|y\|\|_{\psi_2}$ (see, e.g.,~\citep[Lemma 2.7.7]{vershynin2018high}),  we have
    \begin{align*}
        \|\ipc{x}{y}_{Q}\|_{\psi_1} \le (c')^2 \sqrt{d} \|Q\|_2 \sqrt{|\Sigma_x\|_2 \|\Sigma_y\|_2}.
    \end{align*}
    Taking $c = (c')^2$ concludes the proof.
\end{proof}

\begin{lemma}  \label{lem:cov-cat}
    Let $x, y$ be random vectors of dimensions $d_x, d_y$, respectively, defined on the same probability space. Then, $\|\E[[x; y][x; y]^{\top}]\|_2 \le \|\E[xx^{\top}]\|_2 + \|\E[yy^{\top}\|_2$.
\end{lemma}

\vspace*{2pt}
\begin{proof}
    Let $\E[[x; y][x; y]^{\top}] = D D^{\top}$ be a factorization of the positive semidefinite matrix $\E[[x; y][x; y]^{\top}]$, where $D \in \R^{(d_x + d_y) \times (d_x + d_y)}$. Let $D_x$ and $D_y$ be the matrices consisting of the first $d_x$ rows and the last $d_y$ rows of $D$, respectively. Then, 
    \begin{align*}
        \E[[x; y][x; y]^{\top}] = D D^{\top} = [D_x; D_y] [D_x^{\top}, D_y^{\top}]
        = \begin{bmatrix}
            D_x D_x^{\top} & D_x D_y^{\top} \\ 
            D_y D_x^{\top} & D_y D_y^{\top}
        \end{bmatrix}.
    \end{align*}
    Hence, $\E[xx^{\top}] = D_x D_x^{\top}$ and $\E[yy^{\top}] = D_y D_y^{\top}$. The proof is completed by noticing that 
    \begin{align*}
        \|\E[[x; y][x; y]^{\top}]\|_2 = \|D^{\top} D\|_2 =\;& \|[D_x^{\top}, D_y^{\top}] [D_x; D_y]\|_2 \\
        =\;& \|D_x^{\top} D_x + D_y^{\top} D_y\|_2 \\
        \le\;& \|D_x^{\top} D_x\|_2 + \|D_y^{\top} D_y\|_2
        = \|\E[xx^{\top}]\|_2 + \|\E[yy^{\top}]\|_2.
    \end{align*}
\end{proof}

\begin{lemma}  \label{lem:cs-rvec}
    Let $x$ and $y$ be random vectors defined on the same probability space. Then, $\|\E[xy^{\top}]\|_2^2 \le \|\E[xx^{\top}]\|_2 \cdot \|\E[yy^{\top}]\|_2$.
\end{lemma}

\begin{proof}
    Let $d_x, d_y$ be the dimensions of the values of $x, y$, respectively.
    For any vectors $v\in \R^{d_x}$, $w \in \R^{d_y}$, by the Cauchy-Schwarz inequality,  
    \begin{align*}
        (v^{\top} \E[x y^{\top}] w)^2
        =\;& (\E[v^{\top} x y^{\top} w])^2 \\
        \le\;& \E[(v^{\top} x)^2] \cdot \E[(w^{\top} y)^2] \\
        =\;& \E[v^{\top} x x^{\top} v] \cdot \E[w^{\top} y y^{\top} w] \\
        =\;& (v^{\top} \E[x x^{\top}] v) \cdot (w^{\top} \E[y y^{\top}] w).
    \end{align*}

    Taking the maximum over $v$, $w$ on both sides subject to $\|v\|, \|w\| \le 1$ gives 
    \begin{align*}
        \|\E[x y^{\top}]\|_2^2 \le \| \E[xx^{\top}] \|_2 \cdot \|\E[yy^{\top}] \|_2,
    \end{align*}
    which completes the proof.
\end{proof}

\begin{lemma}  \label{lem:cov-diff}
    Let $\sx$ and $\delta$ be $d$-dimensional random vectors defined on the same probability space. Define $x := \sx + \delta$. Then, 
    \begin{align*}
        \|\E[x x^{\top}] - \E[\sx (\sx)^{\top}]\|_2
        \le \|\E[\delta \delta^{\top}]\|_2 + 2 \|\E[\sx (\sx)^{\top}]\|^{1/2}_2 \|\E[\delta \delta^{\top}]\|_2^{1/2}.
    \end{align*}
\end{lemma}
\vspace*{2pt}
\begin{proof}
    Since $x = [I_d, I_d] [\sx; \delta]$, we have 
    \begin{align*}
        \E[x x^{\top}] =\;& [I_d, I_d] \begin{bmatrix}
            \E[\sx (\sx)^{\top}] & \E[\sx \delta^{\top}] \\
            \E[\delta (\sx)^{\top}] & \E[\delta \delta^{\top}]
            \end{bmatrix} \begin{bmatrix}
                I_d \\
                I_d
            \end{bmatrix} 
            =\;& \E[\sx (\sx)^{\top}] + \E[\delta \delta^{\top}] + \E[\sx \delta^{\top}] + \E[\delta (\sx)^{\top}].
    \end{align*}
    Hence, by Lemma~\ref{lem:cs-rvec}, we have
    \begin{align*}
        \|\E[x x^{\top}] - \E[\sx (\sx)^{\top}]\|_2
        \le \|\E[\delta \delta^{\top}]\|_2 + 2 \|\E[\sx (\sx)^{\top}]\|^{1/2}_2 \|\E[\delta \delta^{\top}]\|^{1/2}_2,
    \end{align*}
    which completes the proof.
\end{proof}

\begin{lemma}  \label{lem:cov-sum}
    Let $(x_i)_{i=1}^{n}$ be $d$-dimensional random vectors defined on the same probability space. Then, 
    \begin{align*}
        \Big\|\E\Big[\Big(\nlsum_{i=1}^{n} x_i\Big) \Big(\nlsum_{i=1}^{n} x_i\Big)^{\top}\Big]\Big\|_2^{1/2} \le \nlsum_{i=1}^{n} \|\E[x_i x_i^{\top}]\|_2^{1/2}.
    \end{align*}
\end{lemma}

\vspace*{2pt}
\begin{proof}
    By the triangle inequality, 
    \begin{align*}
        \Big\|\E\Big[\Big(\nlsum_{i=1}^{n} x_i\Big) \Big(\nlsum_{i=1}^{n} x_i\Big)^{\top}\Big]\Big\|_2
        = \Big\|\E\Big[\nlsum_{1\le i, j\le n} x_i x_j^{\top}\Big]\Big\|_2 
        \le \nlsum_{1\le i, j\le n} \|\E[x_i x_j^{\top}]\|_2.
    \end{align*}
    By Lemma~\ref{lem:cs-rvec}, for any $1\le i, j \le n$, 
    \begin{align*}
        \|\E[x_i x_j^{\top}]\|_2 \le \|\E[x_i x_i^{\top}]\|_2^{1/2} \cdot \|\E[x_j x_j^{\top}]\|_2^{1/2}.
    \end{align*}
    Hence, we further have 
    \begin{align*}
        \Big\|\E\Big[\Big(\nlsum_{i=1}^{n} x_i\Big) \Big(\nlsum_{i=1}^{n} x_i\Big)^{\top}\Big]\Big\|_2
        \le\;& \nlsum_{1\le i, j \le n} \|\E[x_i x_i^{\top}]\|_2^{1/2} \cdot \|\E[x_j x_j^{\top}]\|_2^{1/2} \\
        =\;& \Big( \nlsum_{i=1}^{n} \|\E[x_i x_i^{\top}]\|_2^{1/2} \Big)^2,
    \end{align*}
    which completes the proof.
\end{proof}